\newif\ifarxiv 
\arxivfalse

\newif\ifreview 
\reviewfalse

\newif\iffinal 
\finaltrue

\ifarxiv
\documentclass{article}
\else
\ifreview
\documentclass[3p,review,authoryear]{elsarticle}
\else
\documentclass[3p,final,authoryear]{elsarticle}
\fi
\journal{Artificial Intelligence}
\fi

\usepackage[utf8]{inputenc}

\ifarxiv
\usepackage[a4paper,left=2cm,right=2cm,top=2cm,bottom=2cm]{geometry}
\usepackage[onehalfspacing]{setspace}
\usepackage[round,compress]{natbib}
\fi

\usepackage{amssymb,amsfonts,amsmath,amsthm}
\usepackage{dsfont}
\usepackage{bbm}

\usepackage{mathtools}
\usepackage{booktabs,multirow}
\usepackage{xspace}
\usepackage[dvipsnames]{xcolor}
\usepackage{graphicx}
\usepackage{enumitem}
\usepackage{subcaption}
\usepackage{natbib}

\ifarxiv
\usepackage[pdfusetitle,hidelinks]{hyperref} 
\else
\usepackage{hyperref}
\fi

\ifarxiv
\usepackage{authblk}
\fi

\usepackage{algorithm}
\usepackage[noend]{algorithmic}

\usepackage{tikz}
\usetikzlibrary{shapes,arrows}
\usetikzlibrary{decorations}
\usetikzlibrary{plotmarks}
\usetikzlibrary{calc}
\usetikzlibrary{mindmap}
\usetikzlibrary{shadows}
\usetikzlibrary{backgrounds}
\usetikzlibrary{patterns}
\usetikzlibrary{shapes.symbols}

\allowdisplaybreaks

\newtheorem{theorem}             {Theorem}
\newtheorem{lemma}      [theorem]{Lemma}

\newtheorem{definition} [theorem]{Definition}

\newtheorem{claim}      [theorem]{Claim}

\newcommand{\Real}{\mathbb{R}}

\newcommand{\Natural}{\mathbb{N}}

\DeclareMathOperator{\Unif}{Unif}                         

\newcommand*{\bigO}{\mathcal{O}}

\usepackage{xcolor}

\newcommand{\LZ}{\textsc{LZ}\xspace}                      
\newcommand{\TZ}{\textsc{TZ}\xspace}
\newcommand{\LO}{\textsc{LO}\xspace}                      
\newcommand{\TOs}{\textsc{TO}\xspace}                     
\newcommand{\ones}[1]{|#1|_1}                             

\newcommand{\OM}{\textsc{OneMax}\xspace}                           
\newcommand{\JUMP}{\textsc{Jump}\xspace}                           
\newcommand{\RRRfull}{\textsc{RealRoyalRoad}\xspace}  

\newcommand{\COCZ}{\textsc{COCZ}\xspace}                           
\newcommand{\COCZfull}{\textsc{CountingOnesCountingZeroes}\xspace} 

\newcommand{\OMM}{\textsc{OMM}\xspace}                             
\newcommand{\OMMfull}{\textsc{OneMinMax}\xspace}                   

\newcommand{\LOTZ}{\textsc{LOTZ}\xspace}                           
\newcommand{\LOTZfull}{\textsc{LeadingOnesTrailingZeroes}\xspace}  

\newcommand{\OJZJ}{\textsc{OJZJ}\xspace}                           

\newcommand{\expect}[1]{\mathrm{E}\left[#1\right]}        

\newcommand{\RRRMO}{\ensuremath{\textsc{RR}_{\mathrm{MO}}}\xspace}  
\newcommand{\RRRMOuni}{\ensuremath{\textsc{uRR}_{\mathrm{MO}}}\xspace}  

\newcommand{\rp}{\mathrm{rp}}
\newcommand{\refer}{\mathcal{R}_p}

\newcommand{\nsga}{NSGA\nobreakdash-II\xspace}
\newcommand{\nsgaIII}{NSGA\nobreakdash-III\xspace}

\newcommand{\uRRRMO}{\ensuremath{\mathrm{u}\textsc{RR}_{\mathrm{MO}}^*}\xspace} 

\newenvironment{proofofclaim}{\textsc{Proof of Claim.}}{\hfill\scriptsize\scalebox{0.75}{$\blacksquare$}}

\newcommand{\zeros}[1]{|#1|_0}  

\iffinal
\newcommand{\todo}[1]{}

\else
\newcommand{\todo}[1]{\textcolor{red}{[TODO: #1]}}

\fi

\ifreview
\usepackage[switch,mathlines]{lineno}

\usepackage{etoolbox} 
\newcommand*\linenomathpatch[1]{
  \cspreto{#1}{\linenomath}%
  \cspreto{#1*}{\linenomath}%
  \csappto{end#1}{\endlinenomath}%
  \csappto{end#1*}{\endlinenomath}%
}
\linenomathpatch{equation}
\linenomathpatch{gather}
\linenomathpatch{multline}
\linenomathpatch{align}
\linenomathpatch{alignat}
\linenomathpatch{flalign}
\fi

\ifarxiv
\title{Many-Objective Problems Where Crossover is Provably Essential}
\fi 

\ifarxiv
\author{Andre Opris\\
	University of Passau\\
	Passau, Germany
}
\fi 

\begin{document}

\ifarxiv
\maketitle
\else
\begin{frontmatter}

\title{Many-Objective Problems Where Crossover is Provably Essential}

\author[1]{Andre~Opris}\ead{andre.opris@uni-passau.de}
\cortext[cor1]{Corresponding author}
\affiliation[1]{organization={Chair of Algorithms for Intelligent Systems, University of Passau},
	city={Passau},
	country={Germany}}
\fi

\begin{abstract}
This article addresses theory in evolutionary many-objective optimization and focuses on the role of crossover operators. The advantages of using crossover are hardly understood and rigorous runtime analyses with crossover are lagging far behind its use in practice, specifically in the case of more than two objectives. We present two many-objective problems \RRRMO and \RRRMOuni together with a theoretical runtime analysis of the GSEMO and the widely used \nsgaIII algorithm to demonstrate that one point crossover on \RRRMO, as well as uniform crossover on \RRRMOuni, can yield an exponential speedup in the runtime. In particular, when the number of objectives is constant, this algorithms can find the Pareto set of both problems in expected polynomial time when using crossover while without crossover they require exponential time to even find a single Pareto-optimal point. For both problems, we also demonstrate a significant performance gap in certain superconstant parameter regimes for the number of objectives. To the best of our knowledge, this is the first rigorous runtime analysis in many-objective optimization, apart from \citep{Opris2025}, which demonstrates an exponential performance gap when using crossover for more than two objectives. Additionally, it is the first runtime analysis involving crossover in many-objective optimization where the number of objectives is not necessarily constant.
\end{abstract}

\ifarxiv
\textbf{Keywords: evolutionary computation, runtime analysis, recombination, multi-objective optimization}
\else
\begin{keyword}
evolutionary computation\sep
runtime analysis\sep
recombination\sep
multi-objective optimization
\end{keyword}

\end{frontmatter}

\ifreview
\linenumbers
\fi

\fi


\section{Introduction}

Evolutionary multi-objective algorithms (EMOAs) mimic principles from natural evolution as mutation, crossover (recombination) and selection to evolve a population of solutions dealing with multiple conflicting objectives to explore a Pareto optimal set. Those have been frequently applied to a variety of multi-objective optimization problems and also have several applications in practice~\citep{kdeb01,coello2013evolutionary} such as scheduling problems~\citep{704576}, vehicle design~\citep{MultiVehicles} or practical combinatorial optimization problems~\citep{10.1145/3638529.3654077}. They are also widely used in machine learning, artificial intelligence, and various fields of engineering~\citep{9515233,LUUKKONEN2023102537,MULTIENGINEERING}. Particularly, in real world scenarios, there exist many problems with four or more 
objectives~\citep{Chikumbo2012ApproximatingAM,doi:10.1142/5712}. Thus, it is not unexpected that the study of EMOAs became a very important area of research in the last decades, especially for many objectives. 

The mathematical runtime analysis of MOEAs began approximately 20 years ago~\citep{LaumannsTZWD02,Giel03,Thierens03}. In recent years, it has also gained significant momentum, particularly with analyses of classic algorithms such as NSGA-II, NSGA-III, SMS-EMOA, and SPEA2~\citep{ZhengLuiDoerrAAAI22,WiethegerD23,BianZLQ23,RenSPEA2,OprisNSGAIII,Opris2025} or by discussing how MOEAs can solve submodular problems~\citep{QianYTYZ19,Qian_Bian_Feng_2020}.

The \emph{global simple evolutionary multi-objective optimizer} (GSEMO) \citep{Giel03} is the predominant algorithm in the rigorous analysis of MOEAs due to its apparent simplicity. It was the first MOEA to undergo mathematical runtime analysis, and it is still a primary algorithm for discovering new phenomena, as seen in recend works such as~\citep{DinotDHW23}. At the same time, many other algorithms, for example in the area of submodular optimization, build on it. For example, algorithms such as POSS (Pareto Optimization for Subset Selection) \citep{QianYZ15nips} and POMC (Pareto Optimization for Monotone Constraints) are all variants of the GSEMO applied to a suitable bi-objective formulation of the submodular problem of interest.

However, when the number of objectives increases, the size of the Pareto front and the number of incomparable solutions can grow exponentially and therefore, covering a high dimensional front, is a difficult task. There are already strong differences between two and more objectives. \nsga~\citep{Deb2002}, the most used EMOA, optimizes bi-objective problems efficiently (see~\citep{10.1007/978-3-540-70928-2_55} for empirical results or~\citep{ZhengLuiDoerrAAAI22,DoerrQu22,DaOp2023,Dang2024} for rigorous runtime analyses) while it perform less when dealing with three or more objectives (see~\citep{CHAUDHARI20221509} for empirical results or \citep{Zheng2023Inefficiency} for rigorous negative results). This occurs because the so-called \emph{crowding distance}, which serves as the tie-breaker in \nsga, naturally induces a sorting only for two objectives. Specifically, when non-dominated individuals are sorted by the first objective, this automatically determines their order with respect to the second objective in reverse, making the crowding distance an effective measure of proximity in the objective space. However, for problems with three or more objectives, such a correlation between objectives does not necessarily exist, meaning that individuals can have a crowding distance of zero even if they are not actually close to others. To overcome this problem, ~\citet{DebJain2014} proposed \nsgaIII, a refinement of the very popular \nsga, designed to handle more than two objectives, and instead of the crowding distance, uses reference points (previously set by the user) to guarantee that the solution set is well-distributed across the objective space. In particular,~\citet{DebJain2014} empirically showed that \nsgaIII can solve problems between 3 and 15 objectives efficiently. Due to its versatileness, it gained significant traction ($\sim$5500 citations) and now has sereval applications~\citep{TANG2024121723,doi:10.1080/23311916.2016.1269383,GU2022117738}. However, theoretical breakthroughs on its success have only occurred recently. The first rigorous runtime analyses of the state of the art \nsgaIII were published at IJCAI~2023~\citep{WiethegerD23} and GECCO~2024~\citep{OprisNSGAIII}. Apart from further papers like~\citep{DoerrNearTight, Opris2025, opris2025multimodal}, we are not aware of any other such analyses. Hence, its theoretical understanding is still substantially behind its achievements in practice. For example there are several empirical results on the usefulness of crossover in many objectives, particularly for \nsgaIII (see for instance~\citep{YI2020470,8628707}), but apart from~\citep{Opris2025}, focusing on \emph{one-point crossover}, we are not aware of any such theoretical result addressing rigorous runtime analysis in more than two objectives. This is remarkable, because the different variants of crossover are very useful operators in evolutionary computation. Rigorous mathematical analyses of NSGA-III reveal both its strengths and limitations, while also offering practical guidance across various problem classes.

\textbf{Our contribution:} We build on the considerations of~\citep{Dang2024}, and investigate two examples of pseudo-Boolean functions which may involve more than two objectives $m$, called $m$-\RRRMO and $m$-\RRRMOuni. These problems serve as a "royal road" where the use of crossover significantly improves performance. Table~\ref{tab:overview-runtime-results} provides an overview of all runtime results in terms of \emph{fitness evaluations} in this paper, where the crossover probability is also parameterized by $p_c$. When crossover is disabled, both GSEMO and NSGA-III require expected $n^{\Omega(n/m)}$ fitness evaluations to even find a single Pareto-optimal point for both problems, assuming $m = o(n / \log n)$ and a population size of $\mu = 2^{o(n/m)}$ in case of NSGA-III. This runtime is exponential for $m = O(\log(n))$, but still superpolynomial for every possible $m \in o(n/\log(n))$. In sharp contrast, for constant $p_c$, NSGA-III using \emph{one-point crossover} can find the Pareto set of $m$-\RRRMO in expected $O(\mu n^3)$ fitness evaluations for $(4n/(5m)+1)^{m-1} \leq \mu \leq 2^{O(n^2)}$, whereas GSEMO requires $O((4n/(5m)+1)^{m-1} \cdot n^3)$ evaluations. These runtimes grow exponentially in the number of objectives $m$ which is typical for many-objective optimization to ensure adequate coverage of the Pareto front. The runtime of \nsgaIII in terms of generations becomes $O(n^3)$ and hence, does not asymptotically depend on the number $m$ of objectives, and even not on the population size $\mu$ in contrast to~\citep{Dang2024} for the bi-objective case. For population sizes $\mu=\Omega(n^4)$, this is an improvement by a factor of $\Omega(\mu/n^4)$ compared to~\citep{Dang2024}. An additional interesting observation is that as $p_c$ tends to zero, the upper bound of GSEMO becomes smaller than that of \nsgaIII. This is not necessarily because GSEMO outperforms NSGAIII, but rather because rigorous analyses of \nsgaIII typically require the population size to be in the same order as the size of a maximum set of mutually incomparable solutions observed across all stages of optimization. This can be overly pessimistic in stages where the actual number of such solutions is much smaller. GSEMO avoids this limitation by employing a dynamic population size that adapts throughout the process. This even allows that two different runtime bounds can be established: One general one for any $p_c \in (0,1)$, but also a much better one if additionally $1/p_c$ and $1/(1-p_c)$ are bounded by a polynomial in $n$. If in the latter $m$ is constant and also $p_c = O(1/n^{m+1})$, the upper bound of GSEMO is by a factor of $\Omega(n^{m-2})$ smaller than that of NSGA-III when $\mu=(4n/(5m)+1)^{m-1}$. 

For the class $m$-\RRRMOuni, we obtain similar results for a constant crossover probability $p_c$. NSGA-III using \emph{uniform crossover} finds its Pareto set in expected $O(\mu m n^2+\mu n^2/m)$ fitness evaluations whereas GSEMO needs $O(n^2(2n/m)^{m-1})$ fitness evaluations in expectation. The runtime of \nsgaIII does also not depend on the population size in terms of generations and hence, we also get an improvement compared to~\citep{Dang2024} already for $\mu= \Omega(n)$ by a factor of $\mu/n$. We also observe that the proven upper bound for GSEMO is smaller than that of NSGA-III for $\mu=(2n/m)^{m-1}$ by at least a factor of $\Omega(m)$ which becomes even larger if $p_c$ tends to zero. This is also due to the dynamic population size of GSEMO. But anyway, if the number of objectives is constant, uniform crossover also  guarantees an exponential speedup in the runtime.


For our purposes we also have to adapt the general arguments from~\citep{OprisNSGAIII} about the protection of good solutions of \nsgaIII to this situation. All these results transfer to similar algorithms which are able to handle many-objective problems like the SMS-EMOA~\citep{Zheng_Doerr_2024}, the SPEA-2~\citep{RenSPEA2} or variants of NSGA-II~\citep{Krejca2025b}.

\begin{table}[t]
\begin{center}
\caption{Overview of all runtime bounds for the GSEMO and NSGA-III on $m$-\RRRMO and $m$-\RRRMOuni for different crossover probabilities $p_c \in (0,1)$ in terms of fitness evaluations where $m$ denotes the number of objectives, $n$ the problem size and $\mu$ the population size in case of \nsgaIII. For \nsgaIII, we always use $\varepsilon_{\text{nad}} \geq f_{\max}$ and a set $\refer$ of reference points for $p \in \mathbb{N}$ as defined in~\citep{OprisNSGAIII} with $p \geq 2m^{3/2}f_{\max}$ where $f_{\max}$ denotes the maximum value in each objective of the underlying $m$-objective function $f$ (compare with the preliminaries section below for explanations regarding $\mathcal{R}_p$ and $\varepsilon_{\text{nad}}$). For the \nsgaIII we also require that $(4n/(5m)+1)^{m-1} \leq \mu \leq 2^{O(n^2)}$ in case of $m$-\RRRMO and $(2n/m)^{m-1} \leq \mu \leq 2^{O(n^2)}$ in case of $m$-\RRRMOuni. For $m$-\RRRMO, one-point crossover is used, whereas for $m$-\RRRMOuni uniform crossover is applied. In the line with one-point*, we additionally require that both $1/p_c$ and $1/(1 - p_c)$ are bounded by a polynomial in $n$. Further, $d$ denotes a sufficiently small constant and in every case standard bit mutation is used.}
\label{tab:overview-runtime-results}
\begin{tabular}{lllll}
    algorithm & problem & runtime bound & crossover & $m$\\
    \toprule
    NSGA-III & $m$-\RRRMO & $O\left(\frac{\mu n^3}{1-p_c}+\frac{\mu m^2}{p_c}\right)$ (Theorem~\ref{thm:Runtime-Analysis-NSGA-III-mRRMO-onepoint}) & one-point & $\leq n$\\ 
    NSGA-III & $m$-\RRRMO & $n^{\Omega(n/m)}$ (Theorem~\ref{thm:NSGA-III-pc-zero}) &  none & $o (\frac{n}{\log(n)})$ \\
    GSEMO & $m$-\RRRMO & $O\left(\frac{n^3(4n/(5m))^{m-1}}{1-p_c} + \frac{n(4n/(5m))^{m/2}}{p_c}\right)$ (Theorem~\ref{thm:Runtime-Analysis-GSEMO-mRRMO}) &  one-point & $\leq d \sqrt{\frac{n}{\log(n)}}$\\
    GSEMO & $m$-\RRRMO & $O\left(\frac{n^3(4n/(5m))^{m-1}}{1-p_c} + \frac{mn}{p_c}\right)$ (Theorem~\ref{thm:Runtime-Analysis-GSEMO-mRRMO}) & one-point* & $\leq d \sqrt{\frac{n}{\log(n)}}$\\
    GSEMO & $m$-\RRRMO & $n^{\Omega(n/m)}$ (Theorem~\ref{thm:GSEMO-pc-zero}) & none & $o(\frac{n}{\log(n)})$\\
    \midrule
    NSGA-III & $m$-\RRRMOuni & $O\left(\frac{\mu m n^2}{1-p_c}+ \frac{\mu n^2}{p_c m}\right)$ (Theorem~\ref{thm:Runtime-Analysis-NSGA-III-mRRMO-uniform}) &  uniform & $\leq n$\\
    NSGA-III & $m$-\RRRMOuni & $n^{\Omega(n/m)}$ (Theorem~\ref{thm:NSGA-III-pc-zero-uniform}) & none & $o (\frac{n}{\log(n)})$\\
    GSEMO & $m$-\RRRMOuni & $O\left(\frac{n^2(2n/m)^{m-1}}{1-p_c}+ \frac{n^2 (2n/m)^{m/2}}{p_c m}\right)$ (Theorem~\ref{thm:Runtime-Analysis-GSEMO-mRRMO-uniform}) &  uniform & $\leq d \sqrt{\frac{n}{\log(n)}}$\\
    GSEMO & $m$-\RRRMOuni & $n^{\Omega(n/m)}$ (Theorem~\ref{thm:GSEMO-pc-zero-uniform}) &  none & $o(\frac{n}{\log(n)})$\\
    \bottomrule
\end{tabular}
\end{center}
\end{table}

This article is an extended version of the paper~\citep{Opris2025} appeared at AAAI~2025, where only the function $m$-\RRRMO was introduced. That work focused on the analysis of \nsgaIII on $m$-\RRRMO where the number of objectives is constant. They showed that the expected runtime of \nsgaIII on $m$-\RRRMO with one-point crossover is polynomial, but exponential if crossover is turned off. In this extended manuscript we added runtime results of GSEMO on $m$-\RRRMO, and further introduced the function $m$-\RRRMOuni to also showcase that the use of uniform crossover may significantly improve performance. As in~\citep{Dang2024} $m$-\RRRMOuni requires a completely different construction as $m$-\RRRMO. Unlike~\cite{Opris2025}, all our results are formulated for a number $m$ of objectives which is not necessarily constant.

\textbf{Related work:} As already mentioned in~\citep{Dang2024}, there are several rigorous results on the usefulness of crossover on pseudo-Boolean functions in single objective optimization. An exponential performance gap in the runtime was proven the first time by~\citet{Jansen2005c}. These functions were named "real royal road" functions because previous attempts to define "royal road" functions had been unsuccessful~\citep{Mitchell1992,Forrest1993}. To optimize the \RRRfull function, EAs without crossover need exponential time with overwhelming probability while an easy designed EA with one-point crossover optimizes \RRRfull in time $O(n^4)$. The reason is that \RRRfull yields EAs to evolve strings with all one-bits cumulated in a single block of length $2n/3$, and then with one-point crossover the optimal string can easily be generated. \Citet{Jansen2005c} also introduced a class of royal road functions for uniform crossover, which a $(\mu+1)$~GA can optimize in expected time $O(n^3)$ whereas evolutionary algorithms without crossover need exponential time with overwhelming probability. 

For $\JUMP_k$, where a fitness valley of size~$k$ has to be traversed, it has been shown rigorously that \emph{uniform} crossover gives a polynomial or superpolynomial speedup, depending on the parameter $k$~\citep{Jansen2002}. This result has been further improved in~\citep{Koetzing2011, Dang2017}. In~\citep{Opris24Jump} a variant of the standard $(\mu+1)$ GA has been studied where multiple offspring are generated during a crossover step which give massively improved runtime bounds compared to the standard $(\mu+1)$ EA without crossover. Advantages through crossover were also proven for the easy $\OM(x)$ problem~\citep{Sudholt2016,Corus2018a,Doerr2015,Doerr2018} where just the number of ones in the bit string $x$ is counted, and on 
combinatorial optimization problems like the closest string problem~\citep{Sutton21}, shortest paths~\citep{Doerr2012,DOERR201312} or graph coloring problems~\citep{Fischer2005,Sudholt2005}. Also special NP hard graph problems like the $k$-vertex cover or $k$-vertex cluster problem can be optimized efficiently with variants of crossover~\citep{SuttonGraph24}.\\
Early rigorous analyses of EMO focused on simple algorithms, such as SEMO (which mutates by \emph{one-bit mutation}, i.e. flipping a single bit chosen uniformly at random) and its variant GSEMO (which employs \emph{standard bit mutation} as a global search operator, flipping each bit independently with probability $1/n$), both without crossover~\citep{Laumanns2004,Giel2010}. Even today, these algorithms remain a hot topic in evolutionary computation research.
Very recently, GSEMO on unbounded integer spaces has been analyzed~\citep{Doerr_Krejca_Rudolph_2025} and progress in the understanding of GSEMO's population dynamics has been achieved by proving a lower bound of $\Omega(n^2 \log(n))$ for the runtime of GSEMO on the \COCZ, \OJZJ (a variant of the $\JUMP_k$-benchmark for two objectives) and \OMM-benchmarks~\citep{doerr2025tightruntimeguaranteesunderstanding}. 



NSGA-II \citep{Deb2002} is a very prominent algorithm in multi-objective optimization ($\sim$50000 citations), however its theoretical analysis only succeeded recently.
%
%
\citet{ZhengLuiDoerrAAAI22} conducted the first runtime analysis of NSGA-II without
crossover and proved that in expected $\bigO(\mu n^2)$ and
$\bigO(\mu n \log n)$ fitness evaluations the whole Pareto set of \LOTZ and \OMM respectively is found, if the population size $\mu$ is at least four times the size of the Pareto set. 
However, \nsga fails badly in optimizing many-objective problems where the number of objectives is at least three~\citep{Zheng2023Inefficiency}. To address this, \citet{Krejca2025b} proposed a variant of NSGA-II that incorporates a simple tie-breaking rule, enabling efficient handling of problems with more than two objectives.

The theoretical analysis of \nsgaIII only succeeded recently: \citet{WiethegerD23} conducted the first runtime analysis of \nsgaIII on the $3$-\OMMfull problem and showed that for $p \geq 21n$ divisions along each objective for defining the set of reference points, \nsgaIII finds the complete Pareto front of $3$-\OMM in expected $O(\mu n \log(n))$ evaluations where the population size $\mu$ coincides with the size of the Pareto front of $3$-\OMM. \citet{OprisNSGAIII} generalized this result on problems with more than three objectives and gave also a runtime analysis for the classical $m$-\COCZfull and $m$-\LOTZfull benchmarks~\citep{Laumanns2004} for any constant number $m$
of objectives: \nsgaIII with uniform parent selection and standard bit mutation optimizes $m$-\LOTZ in expected $O(n^{m+1})$ evaluations with a population size of $\mu = O(n^{m-1})$ and $m$-\OMM, $m$-\COCZ in expected $O(n^{m/2+1} \log(n))$ fitness evaluations where $\mu = O(n^{m/2})$ (coinciding with the size of the Pareto front of $m$-\OMM and $m$-\COCZ, respectively). They could also reduce the number of required divisions by more than a factor of $2$. 
There are also similar runtime analyses of the SPEA-II~\citep{RenSPEA2} or the SMS-EMOA~\citep{Zheng_Doerr_2024}, which are also two very popular algorithms in many-objective optimization. The efficiency of \nsga, \nsgaIII and SMS-EMOA compared to the classical GSEMO under a mild diversity-perserving mechanism is studied in~\citep{Dang2024Illustrating} where it is shown that on a bi-objective version of \textsc{Trap}, called $\textsc{OneTrapZeroTrap}$, those three algorithms need expected $O(\mu n \log(n))$ fitness evaluations for population sizes $\mu \geq 2$ to find both Pareto-optimal points $0^n$ and $1^n$. They significantly outperform GSEMO which can optimize $\textsc{OneTrapZeroTrap}$ only in expected $n^n$ time. A general dichotometry for conducting runtime analyses in the many-objective setting in terms of \emph{levels} can be found in~\citep{LevelBasedMOEAs}. However, this method is only limited to find the first point on the Pareto front.

All these results above do not take crossover operators in the many-objective setting into account and there are only a 
few papers about EMOAs with crossover which gave a rigorous runtime analysis about this topic. A few variants of GSEMO with crossover have been studied~\citep{Qian2013,Qian_Bian_Feng_2020,Doerr2022} and the first improvement of \nsga with crossover on the runtime on the classical \LOTZ, \OMM and \COCZ problems to $O(n^2)$ was provided by~\citet{Bian2022PPSN}. However, they used stochastic tournament selection and could not outperform~\citet{Covantes2020} which used SEMO with diversity-based parent selection schemes, but without crossover. Later, in parallel independent work,~\citet{DoerrQ23b} and~\citet{Dang2023} showed the first improvements on the runtime of \nsga with crossover for classical parent selection mechanisms in the bi-objective setting. The former studied \OJZJ and showed that crossover speeds up the expected runtime by a factor of $n$. The latter constructed a $\RRRfull$-function, similar to~\citet{Jansen2005c}, to show that crossover can even give an exponential speedup on the runtime. Those results have been extended to the many-objective setting by~\citet{Opris2025} for the \nsgaIII where the number of objectives is constant. They constructed an $m$-$\RRRfull$-function which can be optimized by \nsgaIII with one-point crossover in $O(n^3)$ generations (the runtime not depending on $\mu$ in contrast to~\citep{Dang2024}), but only in time $n^{\Omega(n)}$ without crossover for a polynomial large population size. 

\section{Preliminaries}

Let $\ln$ be the logarithm to base $e$ and $\log$ the logarithm to base $2$. let $[m]:=\{1, \ldots , m\}$ for $m \in \mathbb{N}$. For a finite set $A$ we denote by $\vert{A}\vert$ its cardinality. For two random variables $X$ and $Y$ on $\mathbb{N}_0$ we say that $Y$ stochastically dominates $X$ if $\Pr(Y \leq c) \leq \Pr(X \leq c)$ for every $c \geq 0$ (which also implies $\expect{Y} \leq \expect{X}$). We say that $f = O(\text{poly}(n))$ if there is a polynomial $g$ in $n$ with $f = O(g)$. For two bit strings $x,y$ denote by $H(x,y)$ its Hamming distance, i.e. $H(x,y) = \sum_{i=1}^n |x_i-y_i|$. The number of ones in a bit string $x$ is denoted by $\ones{x}$ and the number of zeros by $\zeros{x}$ respectively. The number of leading zeros in $x$, denoted by $\LZ(x)$, is the length of the longest prefix of $x$ which contains only zeros, and the number of trailing zeros in $x$, denoted by $\TZ(x)$, the length of the longest suffix of $x$ containing only zeros respectively. For example, if $x=00110110110000$, then $\LZ(x)=2$ and $\TZ(x)=4$. The number of leading ones in $x$, denoted by $\LO(x)$ and trailing ones in $x$, denoted by $\TOs(x)$, is the length of the longest prefix and suffix of $x$ which contains only ones respectively. Throughout this paper we use \emph{standard bit mutation} as mutation operator where each bit is flipped independently with probability $1/n$. For crossover, two operators are used: \emph{One-point crossover} on two parents $y_1,y_2 \in \{0,1\}^n$ chooses a \emph{cutting point} $k$ from $\{0, \ldots , n\}$ uniformly at random and creates an offspring $z$ with $z_i = (y_1)_i$ for all $1 \leq i \leq k$ and $z_i = (y_2)_i$ for $k < i \leq n$. \emph{Uniform crossover} creates a $z$ where $z_i = (y_1)_i$ with probability $1/2$ and $z_i = (y_2)_i$ with probability $1/2$ independently for each $i \in [n]$.

This article is about \emph{many-objective optimization}, particularly the maximization of a discrete $m$-objective pseudo-Boolean function $f(x):=(f_1(x), \ldots , f_m(x))$ where $f_i:\{0,1\}^n \to \mathbb{N}_0$ for each $i \in [m]$. When $m=2$, the function is also called \emph{bi-objective}. Let $f_{\max}$ be the maximum possible value of $f$ in one objective, i.e. $f_{\max}:=\max\{f_j(x) \mid x \in \{0,1\}^n, j \in [m]\}$. For $N \subseteq \{0,1\}^n$ let $f(N):=\{f(x) \mid x \in N\}$.

\begin{definition}
	Consider an $m$-objective function $f$.
	\begin{enumerate}
		\item[1)] Given two search points $x, y \in \{0, 1\}^n$,
		$x$ \emph{weakly dominates} $y$, denoted by $x \succeq y$,
		if $f_{i}(x) \geq f_{i}(y)$ for all $i \in [m]$ and $x$ \emph{(strictly) dominates} $y$, denoted by $x \succ y$, if one inequality is strict; if neither $x \succeq y$ nor $y \succeq x$ then $x$ and $y$ are \emph{incomparable}.
		\item[2)] A set $S \subseteq \{0,1\}^n$ is a \emph{set of mutually incomparable solutions} with respect to $f$ if all search points in $S$ are incomparable.
		\item[3)] Each solution not dominated by any other in $\{0, 1\}^n$ is called \emph{Pareto-optimal}. A mutually incomparable set of these solutions that covers all possible non-dominated fitness values is called a \emph{Pareto(-optimal) set} of $f$.
	\end{enumerate}
\end{definition}

To perform runtime analyses of \nsgaIII, it is essential to estimate the size of a maximum set of mutually incomparable solutions with respect to $f$. The following result is particularly useful for this purpose.

\begin{lemma}
\label{lem:dominance-preparation}
Let $f:\{0,1\}^n \to \mathbb{N}_0^m$ be any $m$-objective fitness function and let $S$ be a set of mutually incomparable solutions with respect to $f$. Then the following holds.
\begin{itemize}
\item[(i)] Suppose that $f$ attains at most $k$ different values in each objective. Then $|S| \leq k^{m-1}$.
\item[(ii)] Suppose that $m=2$ and let $a \in \mathbb{N}$ with $|f_1(x)-f_2(x)| \leq a$ for every $x \in \{0,1\}^n$. Then $|S| \leq a+1$.
\end{itemize}
\end{lemma}

\begin{proof}
(i): Let $V:=\{(v_1, \ldots , v_{m-1}) \mid \text{ there is } x \in \{0,1\}^n \text{ with } f(x)=(v_1, \ldots , v_m)\}$. Then $|V| = k^{m-1}$. We show that $|S| \leq |V|$. Let $x,y \in S$ with $x \neq y$. Then also $(f_1(x), \ldots , f_{m-1}(x)) \neq (f_1(y), \ldots , f_{m-1}(y))$ (if $(f_1(x), \ldots , f_{m-1}(x)) = (f_1(y), \ldots , f_{m-1}(y))$ then $x,y$ would be comparable since they differ only in the last component). Therefore $|S| = |\{(f_1(x), \ldots , f_{m-1}(x)) \mid x \in S\}| \leq |V|$.

(ii): We consider $W:=f(S) = \{f(x) \mid x \in S\}$. Then $|W| = |S|$ since two incomparable solutions have two different fitness vectors. For $v\in W$ denote by $d_v:=v_1-v_2$ the difference in both components. For $v,w \in W$ we observe that $d_w \notin \{d_v-1, d_v\}$ or, in other words, $w_2 \notin \{w_1 + d_v - 1, w_1 + d_v\}$. Otherwise $x,y$ with $f(x)=v,f(y)=w$ are comparable (since $v_1 < w_1$ implies $v_2 = v_1+d_v < w_1 + d_v$ and therefore $v_2 = v_1+d_v \leq w_2$ due to $w_2 \in \{w_1 + d_v-1,w_1 + d_v\}$ and $v_1 > w_1$ implies $v_2 = v_1+d_v > w_1 + d_v$ and therefore $v_2 = v_1+d_v > w_2$). Therefore $|d_v-d_w| \geq 2$ for every $v,w \in W$ implying that $|S| = |\{f_2(x)-f_1(x) \mid x \in S\}| \leq \lceil{(2a+1)/2}\rceil = a+1$ since $|f_1(x)-f_2(x)| \leq a$ concluding the proof.
\end{proof} 


\begin{algorithm}[t]
	\begin{algorithmic}[1]
	\STATE{Initialize $P_0 \sim \text{Unif}( (\{0,1\}^n)^{\mu})$;}
	\FOR{$t:= 0$ to $\infty$}
		\STATE{Initialize $Q_t:=\emptyset$;}
		\FOR{$i=1$ to $\mu/2$}
			\STATE{Sample $p_1,p_2$ from $P_t$ uniformly at random;}
			\STATE{Sample $r \sim \text{Unif}([0,1])$;}
			\IF{$r \leq p_c$}
				\STATE{Create $c_1$ by $1$-point crossover on $p_1,p_2$;}
				\STATE{Create $c_2$ by $1$-point crossover on $p_1,p_2$;}
			\ELSE
				\STATE{Create $c_1,c_2$ as copies from $p_1,p_2$;}
			\ENDIF
			\STATE{Create $s_1,s_2$ by standard bit mutation on $c_1,c_2$ with mutation probability $1/n$;}
			\STATE{Update $Q_t:=Q_t \cup \{s_1,s_2\}$;}
		\ENDFOR
		\STATE{Set $R_t := P_t \cup Q_t$;}
		\STATE{Partition $R_t$ into layers $F^1_t,F^2_t,\ldots ,F^k_t$ of non-dominated fitness vectors;}
		\STATE{Find $i^* \geq 1$ such that $\sum_{i=1}^{i^*-1} \lvert{F_t^i}\rvert < \mu$ and $\sum_{i=1}^{i^*} \lvert{F_t^i}\rvert \geq \mu$;}
		\STATE{Compute $Y_t = \bigcup_{i=1}^{i^*-1} F_t^i$;}
		\STATE{Choose $\tilde{F}_t^{i^*} \subset F_t^{i^*}$ such that $\lvert{Y_t \cup \tilde{F}_t^{i^*}}\rvert = \mu$ with Algorithm~\ref{alg:Survival-Selection};}
		\STATE{Create the next population $P_{t+1} := Y_t \cup \tilde{F}^t_{i^*}$;}
		\ENDFOR
	\end{algorithmic}
	\caption{\nsgaIII on an $m$-objective function $f$ with population size $\mu$ and crossover probability $p_c$~\citep{DebJain2014}}
	\label{alg:nsga-iii}
\end{algorithm}

The \nsgaIII algorithm~\citep{DebJain2014} is shown in Algorithm~\ref{alg:nsga-iii}. At first, a multiset of size $\mu$ is generated, called \emph{population}, by initializing $\mu$ individuals uniformly at random. Then in each generation, a population $Q_t$ of $\mu$ new offspring is created by conducting the following operations $\mu/2$ times, where a single sequence of this operations is often called a \emph{trial}. At first two parents $p_1$ and $p_2$ are chosen uniformly at random. Then $1$-point crossover will be applied on $(p_1,p_2)$ with some probability $p_c \in [0,1]$ to produce two solutions $c_1,c_2$. If $1$-point crossover is not executed (with probability $1-p_c$), $c_1,c_2$ are exact copies of $p_1,p_2$. Finally, two offspring $s_1$ and $s_2$ are created with \emph{standard bit mutation} on $c_1$ and $c_2$, i.e. by flipping each bit independently with probability $1/n$.\\ 
During the survival selection, the parent and offspring populations $P_t$ and $Q_t$ are merged into $R_t$, and then partitioned into layers $F^1_{t+1},F^2_{t+1},\dots$ using the \emph{non-dominated sorting algorithm}~\citep{Deb2002}. The layer $F^1_{t+1}$ consists of all non-dominated points, and $F^i_{t+1}$ for $i>1$ consists of points that are only dominated by
those from $F^1_{t+1},\dots,F^{i-1}_{t+1}$. Then the critical and unique index $i^*$ with $\sum_{i=1}^{i^*-1} \lvert{F_i^t}\rvert < \mu$ and $\sum_{i=1}^{i^*} \lvert{F_i^t}\rvert \geq \mu$ is determined (i.e. there are fewer than $\mu$ search points in $R_t$ with a lower rank than $i^*$, but at least $\mu$ search points with rank at most $i^*$). 
All individuals with a smaller rank than $i^*$ are taken into $P_{t+1}$ and the remaining points are chosen from $F_{i^*}^t$ with Algorithm~\ref{alg:Survival-Selection}. 

\begin{algorithm}[t]
	\begin{algorithmic}[1]
	\STATE{Compute the normalization $f^n$ of $f$;} 
	\STATE{Associate each $x\in Y_t \cup F_t^{i^*}$ with its reference point $\rp(x)$ such that the distance between $f^n(x)$ and the line through the origin and $\rp(x)$ is minimized;} 
	\STATE{For each $r \in \refer$, set $\rho_r:=|\{x\in Y_t \mid \mathrm{rp}(x)=r\}|$;}
	\STATE{Initialize $\tilde{F}_t^{i^*}=\emptyset$ and $R':=\refer$;}
	\WHILE{true}
		\STATE{Determine $r_{\min} \in R'$ such that $\rho_{r_{\min}}$ is minimal (where ties are broken randomly);}
		\STATE{Determine $x_{r_{\min}} \in F_t^{i^*} \setminus \tilde{F}_t^{i^*}$ which is associated with $r_{\min}$ and minimizes the distance between the vectors $f^n(x_{r_{\min}})$ and $r_{\min}$ (where ties are broken randomly)\label{line:association};}
		\IF{$x_{r_{\min}}$ exists}
			\STATE{$\tilde{F}_t^{i^*} = \tilde{F}_t^{i^*} \cup \{x_{r_{\min}}\}$;}
			\STATE{$\rho_{r_{\min}} = \rho_{r_{\min}} + 1$;}
			\IF{$\lvert{Y_t}\rvert + \lvert{\tilde{F}_t^{i^*}}\rvert = \mu$}
			\RETURN{$\tilde{F}_t^{i^*}$}
			\ENDIF
		\ELSE
			\STATE{$R'=R' \setminus \{r_{\min}\}$;}
		\ENDIF
	\ENDWHILE
	\end{algorithmic}
	\caption{Selection procedure utilizing a set $\refer$ of reference points for maximizing a function}
	\label{alg:Survival-Selection}
\end{algorithm}

At first in Algorithm~\ref{alg:Survival-Selection}, a normalized objective function $f^n$ is computed and then each individual with rank at most $i^*$ is associated with reference points. We use the same set of structured reference points $\refer$ as
proposed in the original paper~\citep{DebJain2014}, originated in~\citep{Das1998}. The points are defined on the simplex of the unit vectors 
$(1,0,\dots,0)^{\intercal},(0,1,\dots,0)^{\intercal},\dots,(0,0,\dots,1)^{\intercal}$ as:
\[
\left\{\left(\frac{a_1}{p}, \ldots ,\frac{a_m}{p} \right) 
\text{ }\Big|\text{ }
(a_1,\dots,a_m) \in \mathbb{N}_0^m, 
\sum_{i=1}^m a_i = p
\right\}
\]
where $p \in \mathbb{N}$ is a parameter one can choose according 
to the fitness function $f$. 

Now each individual $x$ is associated with the reference point $\text{rp}(x)$ such that the distance between $f^n(x)$ and the line through the origin and $\text{rp}(x)$ is minimal. 
Then, one iterates through all the reference points where the reference point with the fewest associated individuals that are already selected for the next generation $P_{t+1}$ is chosen. Ties are broken uniformly at random. A reference point is omitted if it only has associated individuals that are already selected for $P_{t+1}$. Then, among the not yet selected individuals of that reference point, the one nearest to the chosen reference point is taken for the next generation where ties are again broken uniformly at random. If the required number of individuals is reached (i.e. if $\lvert{Y_t}\rvert+\lvert{\tilde{F}_t^{i^*}}\rvert = \mu$) the selection ends. In Line~\ref{line:association} of Algorithm~\ref{alg:Survival-Selection} one could use any other diversity-perserving mechanism if $\rho_{r_{\min}} \geq 1$. Note that \nsga follows the same scheme as Algorithm~\ref{alg:nsga-iii} with the difference that $\tilde{F}_t^{i^*} \subset F_t^{i^*}$ is chosen based on sorting according to the crowding distance, rather than using Algorithm~\ref{alg:Survival-Selection}~\citep{Deb2002}.

Further, we use the normalization procedure from~\citep{WiethegerD23} which can be also used for maximization problems as shown in~\citep{OprisNSGAIII}. We omit detailed explanations as they are not needed for our purposes. For an $m$-objective function $f\colon \{0,1\}^n\rightarrow \mathbb{N}_0^m$, the normalized fitness vector $f^{n}(x):=(f_1^n(x),\dots,f_m^n(x))$ of a search point $x$ is given by
\begin{align}
	f_j^n(x)
	=\frac{f_j(x)-y_j^{\min}}{y_j^{\text{nad}}-y_j^{\min}}\label{eq:nsgaiii-normalize}
\end{align}
for each $j \in [m]$. The points $y^{\text{nad}}:=(y_1^{\text{nad}}, \ldots, y_m^{\text{nad}})$ and 
$y^{\min}:=(y_1^{\min}, \dots, y_m^{\min})$ from the objective space are denoted by
\emph{nadir} and \emph{ideal} points, respectively. In particular, $y_j^{\text{min}}$ is set to the minimum value in objective $j$ from all search points seen so far (i.e. from $R_0,\ldots , R_t$). Computing the nadir point is non-trivial, but the procedure ensures for each $j \in [m]$ that $y_j^{\text{nad}} \geq \varepsilon_{\text{nad}}$, and $y_j^{\text{min}} \leq y_j^{\text{nad}} \leq y_j^{\text{max}}$ 
where $y_j^{\max}$ is the maximum value in objective $j$ from all search points seen so far 
and $\varepsilon_{\text{nad}}$ is a positive threshold. The following crucial result from~\citep{OprisNSGAIII} shows that sufficiently many reference points protect good solutions. In other words, if a population covers a fitness vector $v$ with a first-ranked individual $x$, i.e. there is $x \in F_t^1$ with $f(x)=v$, then it is covered for all future generations as long as $x \in F_t^1$. (Compare also with~\citep{WiethegerD23} for a similar result, but limited to the 3-objective $m$-\OMM problem for a higher number $p$ of divisions.)

\begin{lemma}[\citet{OprisNSGAIII}, Lemma~3.4]
\label{lem:Reference-Points}
Consider \nsgaIII optimizing an $m$-objective function $f$ with $\varepsilon_{\text{nad}} \geq f_{\max}$ and a set $\refer$ of reference points for $p \in \mathbb{N}$ with $p \geq 2m^{3/2}f_{\max}$. Let $P_t$ be its current population and $F_t^1$ be the multiset describing the first layer of the merged population of parent and offspring. Assume the population size $\mu$ fulfills the condition $\mu \geq \lvert{S}\rvert$ where $S$ is a maximum set of mutually incomparable solutions. 
Then for every $x \in F_t^1$ there is a $x' \in P_{t+1}$ with $f(x')=f(x)$.
\end{lemma}

\begin{algorithm}[t]
	\begin{algorithmic}[1]
		\STATE Initialize $P_0:=\{s\}$ where $s \sim \Unif(\{0,1\}^n)$;
		\FOR{$t:= 0$ \TO $\infty$}
		\STATE{Sample $p_1 \sim \Unif(P_t)$;}
		\STATE{Sample $u \sim \Unif([0,1])$;}
		\IF{($u<p_c$)}
		\STATE{Sample $p_2 \sim \Unif(P_t)$;}
		\STATE{Create $s$ by crossover between $p_1$ and $p_2$;}
		\ELSE
		\STATE{Create $s$ as a copy of $p_1$;}
		\ENDIF
		\STATE{Create $s'$ by mutation on $s$;}\label{alg:gsemo:line-mutation} 
	\IF{($s'$ is \NOT dominated by any individual in $P_t$)}
	\STATE{Create the next population $P_{t+1} := P_t \cup \{s'\}$;}
	\STATE{Remove all $x \in P_{t+1}$ 
		weakly dominated by $s'$;} 
\ENDIF
\ENDFOR
\end{algorithmic}
\caption{GSEMO Algorithm}
\label{alg:gsemo}
\end{algorithm}

The GSEMO algorithm is presented in Algorithm~\ref{alg:gsemo}. It starts from a randomly generated solution, and in each generation, a new search point
$s'$ is created as follows. With probability $p_c \in [0,1]$,
crossover is performed on two parents selected uniformly at random and without replacement, followed by mutation. Otherwise, mutation is applied directly to a randomly chosen parent. If the offspring $s'$ is not dominated by any solution in the current population $P_t$, it is added to the population, and all solutions weakly dominated by $s'$ are removed. Consequently, the population size $|P_t|$ may vary over time.

\subsection{Mathematical Tools}

At first we state some inequalities and monotonicity properties of functions which are very useful in our analyses below.

\begin{lemma}
\label{lem:inequality}
For every $k \geq 4$ and $x\geq 5$ the inequality $(\sqrt{2x+4})^k \leq (x+1)^{k-1}$ holds.
\end{lemma}

\begin{proof}
Suppose that $k=4$. We have $(\sqrt{2x+4})^4 = (2x+4)^2 \leq (x+1)^3$ for $x=5$ and hence, also for $x \geq 5$ since the derivative of the left term with respect to $x$ is $8x+16$, the derivative of the left term with respect to $x$ is $3x^2+6x+3$ and for $x \geq 5$ we have $8x + 16 \leq 3x^2+6x+3$. The result follows also for $k > 4$ since $\sqrt{2x+4} \leq x+1$ for $x \geq 2$.
\end{proof}

\begin{lemma}
\label{lem:function-monotone}
The following properties are satisfied.
\begin{itemize}
    \item[(i)] For every $k>0$ the function $f: \; ]0,\infty[ \; \to \mathbb{R}, x \mapsto (k/x+1)^{x-1},$ is strictly monotone increasing.
    \item[(ii)] For every $k>0$ the function $g: \; ]0,\infty[ \; \to \mathbb{R}, x \mapsto x \cdot \ln(k/x+1),$ is strictly monotone increasing.
    \item[(iii)] For every $k>1$ the function $[1,k/e] \to \mathbb{R}, x \mapsto (k/x)^{x-1}$ is strictly monotone increasing.
    \item[(iv)] For every $k>0$ the function $]0,k/e] \to \mathbb{R}, x \mapsto (k/x)^{x/2},$ is strictly monotone increasing.
\end{itemize}
\end{lemma}

\begin{proof}
    (i): By taking the logarithm, it is enough to show that the function $h: \; ]0,\infty[ \; \to \mathbb{R}, x \mapsto (x-1)\ln(k/x+1),$ is strictly monotone increasing. The derivative of $h$ with respect to $x$ is $$\ln\left(\frac{k}{x}+1 \right) - \frac{x-1}{k/x+1} \cdot \frac{k}{x^2} = \ln \left(\frac{k}{x}+1\right) - \frac{x-1}{x} \cdot \frac{k/x}{k/x+1} > 0$$
    since $y/(1+y) < \ln(1+y)$ for every $y > -1$ (see for example~\citep{Dragomir2020}), and therefore
    $$\ln\left(\frac{k}{x}+1\right) > \frac{k/x}{k/x+1} > \frac{x-1}{x}\frac{k/x}{k/x+1}.$$
    
    (ii): The derivative of $g$ with respect to $x$ is 
    $$\ln\left(\frac{k}{x}+1\right) - x \cdot \frac{1}{k/x+1} \cdot \frac{k}{x^2} > \frac{k/x}{k/x+1} - \frac{k/x}{k/x+1} = 0.$$

    (iii): Note that $\ln((k/x)^{x-1}) = (x-1) \ln (k/x)$ and the derivative with respect to $x$ of the right term is $\ln(k/x)-(x-1)/x > 0$ for $1 < x \leq k/e$.

    (iv): It is enough to show that the function $h: \; ]0,k/e] \to \mathbb{R}, x \mapsto x\ln(k/x)/2,$ is strictly monotone increasing. Its derivative with respect to $x$ is $\ln(k/x)/2 - 1/2 > 0$ for all $x \in \text{ } ]0,k/e[$.
\end{proof}

The following bounds are useful for estimating the probabilities when stochastic amplifications occur.

\begin{lemma}[Lemma~10 in \citet{Badkobeh2015}]\label{lem:Badkobeh}
	For every $p\in[0,1]$ and every $\lambda\in\Natural$,
	\begin{align*}
		1-(1-p)^\lambda \in \left[\frac{p\lambda}{1+p\lambda},\frac{2p\lambda}{1+p\lambda}\right].
	\end{align*}
\end{lemma}

We also require a tail bound for the sum of independent geometric random variables, which is a valuable tool for deriving runtime bounds that hold not only in expectation, but also with high probability.

\begin{theorem}[Theorem 15 in~\citet{DOERR2019115}]\label{thm:Doerr-dominance}
	Let $X_1, \ldots , X_n$ be independent geometric random variables with success probabilities $p_1, \ldots , p_n$. Let $X=\sum_{i=1}^n X_i$, $s=\sum_{i=1}^n (1/p_i)^2$ and $p_{\min}:=\min\{p_i \mid i \in [n]\}$. Then for all $\lambda \geq 0$
	\begin{enumerate}
		\item[(1)]
		$\Pr(X \geq \expect{X} + \lambda) \leq \exp \left(-\frac{1}{4} \min \left\{\frac{\lambda^2}{s}, \lambda p_{\min}\right\}\right)$,
		\item[(2)]
		$\Pr(X \leq \expect{X} - \lambda) \leq \exp \left(-\frac{\lambda^2}{2s}\right)$.
	\end{enumerate} 
\end{theorem}

Finally, we employ drift analysis (originated in~\citep{Jun2004}), which became a standard tool in runtime analysis over the last years (see for example~\citep{Koetzing2016} for a detailed overview).

\begin{theorem}[Additive drift theorem from \citet{Jun2004}]
\label{thm:additive-drift}
Let $(X_t)_{t\geq 0}$ be a sequence of non-negative random variables with a
finite state space $S\subseteq \Real_{\geq 0}$ with $0\in S$. Define
$T := \inf\{t \geq 0 \mid X_{t} = 0\}$.
\begin{enumerate}
	\item[(i)] Suppose that there is $\delta>0$ such that $\expect{X_{t}-X_{t+1}\mid X_{t}=s}\geq\delta$ for all $s\in S\setminus\{0\}$ and all $t\geq 0$. Then $\expect{T \mid X_0}\leq X_0/\delta$.
	\item[(ii)] Suppose that there is $\delta>0$ such that $\expect{X_{t}-X_{t+1}\mid X_{t}=s}\leq\delta$ for all $s\in S\setminus\{0\}$ and all $t\geq 0$. Then $\expect{T \mid X_0}\geq X_0/\delta$.
\end{enumerate}
\end{theorem}

In our analysis of evolutionary algorithms on the proposed royal road functions, we encounter situations where a part of the bit string provides no fitness signal. As a result, standard bit mutations tend to drive this region toward an approximately equal number of zeros and ones, regardless of the initialization. Since this behavior plays a key role in several of our proofs, we provide a lemma that bounds the expected time to reach such a balanced state. As a side note, our result generalizes Lemma~5 from~\citep{Dang2024}, as the size of the considered region can be also $o(n)$.

\begin{lemma}
\label{lem:additive-drift-main}
Consider a sequence of search points $x(0), x(1), \ldots $ from $\{0,1\}^n$ such that $x(t+1)$ results from $x(t)$ by applying a standard bit mutation. Let $L \subset [n]$ be a set of fixed locations in $x(t)$ with $|L| = cn$ where $c \in (0,1]$. Further, let $X_t$ be the number of ones in $x(t)$ from these locations $L$ and let $0<\lambda<1$ be a constant such that $\lambda c n \geq 2$. Let $T:=\inf\{t \mid cn(1-\lambda)/2 \leq X_t \leq cn(1+\lambda)/2\}$. Then $\expect{T}=O(n)$. This means also that $\expect{T^*}=O(n)$ for $T^*:=\inf\{t \mid X_t \leq cn(1+\lambda)/2\}$.
\end{lemma}

\begin{proof}
Suppose that $X_0 >cn(1+\lambda)/2$. Let $T_{\text{down}}^*:=\inf\{t \mid X_t \leq cn(1+\lambda)/2\}$. We have that
$$\expect{X_t - X_{t+1} \text{ } \Big| \text{ } X_t >\frac{cn(1+\lambda)}{2}} > \frac{c(1+\lambda)}{2}-\frac{c(1-\lambda)}{2} = c \lambda$$
since $X_t >cn(1+\lambda)/2$ means that there are at more than $cn(1+\lambda)/2$ ones at locations from $L$ and hence, in expectation, at least $c(1+\lambda)/2$ ones, but at most $c-c(1+\lambda)/2=c(1-\lambda)/2$ zeros are flipped. Due to $X_0 \leq cn$ we obtain with the additive drift theorem (Theorem~\ref{thm:additive-drift}) 
$$\expect{T_{\text{down}}^* \mid X_0 > \frac{cn(1+\lambda)}{2}} \leq \frac{cn}{c \lambda} = \frac{n}{\lambda}.$$
Similarly, if $X_0 <cn(1-\lambda)/2$ define $Y_t := cn-X_t$ counting the number of zeros at locations from $L$ and obtain $Y_0 > cn(1+\lambda)/2$ as well as
$$\expect{Y_t - Y_{t+1} \text{ } \Big| \text{ } Y_t >\frac{cn(1+\lambda)}{2}} > c \lambda$$
since $Y_t >cn(1+\lambda)/2$ means that there are more than $cn(1+\lambda)/2$ zeros, but fewer than $cn(1-\lambda)/2$ ones in $L$. Hence, for $T_{\text{up}}^*:=\inf\{t \geq 0 \mid X_t \geq cn(1-\lambda)/2\} = \inf\{t \geq 0 \mid Y_t \leq cn(1+\lambda)/2\}$ we obtain again with Theorem~\ref{thm:additive-drift}
$$\expect{T_{\text{up}}^* \mid X_0 < \frac{cn(1-\lambda)}{2}} = \expect{T_{\text{up}}^* \mid Y_0 > \frac{cn(1+\lambda)}{2}} \leq \frac{cn}{c \lambda} = \frac{n}{\lambda}.$$

Let $A$ be the event that $X_{T_\text{down}^*} < (cn(1-\lambda))/2$ when $X_0 > (cn(1+\lambda))/2$, and $X_{T_\text{up}^*} > (cn(1+\lambda))/2$ when $X_0 < (cn(1-\lambda))/2$. Denote the occurence of the event $A$ as a \emph{failure}. Let $T^*:=T_\text{down}^*$ if $X_0 > cn(1+\lambda)/2$ and $T^*:=T_\text{up}^*$ if $X_0 < cn(1-\lambda)/2$. In every case we have that $\expect{T^*} \leq n/\lambda$. If no failure occurs then $T=T^*$. We show that $\Pr(A) = o(1)$. For the event $A$ to occur, at least $\lfloor{c\lambda n}\rfloor$ distinct bits at positions in $L$ must be flipped during some iteration $t < T^*$, since the process $(X_t)_{t \geq 0}$ must skip the interval $[cn(1-\lambda)/2,\, cn(1+\lambda)/2]$ which has length $c\lambda n$. In any fixed iteration, the probability of  flipping at least $\lfloor{c\lambda n}\rfloor$ bits is at most 
$$\frac{\binom{cn}{\lfloor{c\lambda n}\rfloor}}{n^{\lfloor{c\lambda n}\rfloor}} \leq \frac{2^{cn}}{n^{\lfloor{c\lambda n}\rfloor}} = O\left(\frac{1}{n^2}\right)$$
where the last equality holds due to $c \lambda n \geq 2$. By a union bound, the probability that at least $c\lambda n$ bits are flipped in an iteration $t \leq T^*$ is at most $O(1/n^2) \expect{T^*} = O(1/n)=o(1)$ and hence, with probability $o(1)$ a failure occurs. If a failure occurs we repeat the above arguments (either for $Y_t$ if $X_0>cn(1+\lambda)/2$ or for $X_t$ if $X_0<cn(1-\lambda)/2$). The expected number of periods is $1+o(1)$ which proves the lemma.
\end{proof}

Throughout this paper, we also use a variant of the negative drift theorem that applies even when the drift bound $\varepsilon$ is smaller than a constant. In particular, this bound can also depend on the size $\ell$ of the drift interval. For more common versions, where the drift bound is at least a constant, see \citep{LenglerS18} and \citep{Oliveto2011}.

\begin{theorem}[Negative drift theorem, see Theorem~3 in \citep{WittDrift}]
\label{thm:negative-drift}
Let $(Y_t)_{t \in \mathbb{N}}$ be a stochastic process over $\mathbb{R}$, adapted to a filtration $\mathcal{F}_t$. Suppose there exists an interval $[a,b] \subset \mathbb{R}$ and, possibly depending on $\ell := b - a$, a drift bound $\varepsilon := \varepsilon(\ell)>0$, a typical forward jump factor $\kappa := \kappa(\ell)>0$, a scaling factor $r:= r(\ell)>0$ as well as a sequence of functions $\Delta_t:=\Delta_t(Y_{t+1}-Y_t)$ satisfying $\Delta_t \preceq Y_{t+1}-Y_t$ such that for all $t \geq 0$ the following three conditions hold:
\begin{itemize}
\item[(1)] $\expect{\Delta_t \cdot \mathbbm{1}\{\Delta_t \leq \kappa \varepsilon\} \mid \mathcal{F}_t; a < Y_t < b} \geq \varepsilon$,
\item[(ii)] $\Pr(\Delta_t \leq -ir \mid \mathcal{F}_t ; a < Y_t) \leq e^{-i}$ for all $i \in \mathbb{N}$,
\item[(iii)] $\chi \ell \geq 2 \ln(4/(\chi \varepsilon))$, where $\chi:=\min\{1/(2r),\varepsilon/(17r^2),1/(\kappa \varepsilon)\}$.
\end{itemize}
Then for $T^*:=\min\{t \geq 0 \mid Y_t \leq a\}$ it holds that $\Pr(T^* \leq e^{\chi \ell/4} \mid \mathcal{F}_0 ; Y_0 \geq b) = O(e^{-\chi \ell/4})$.
\end{theorem}

As soon as the part of the bit string not providing a fitness signal contains approximately the same number of ones and zeros, the negative drift theorem shows that, with high probability, this ratio remains stable over a long time for a large parameter regime of $c$, as described in the following lemma.  

\begin{lemma}
\label{lem:negative-drift-main}
Consider a sequence of search points $x(0), x(1), \ldots $ from $\{0,1\}^n$ such that $x(t+1)$ results from $x(t)$ by applying a standard bit mutation. Let $L \subset [n]$ be a set of fixed locations in $x(t)$ with $|L| = cn$ where $c \geq \sqrt{\ln(n)}/(d \sqrt{n})$ for a sufficiently small constant $d>0$. Let $X_t$ be the number of ones in $x(t)$ from these locations $L$ and let $0<\lambda<1$ be a constant. Suppose that $cn(1-\lambda/2)/2 \leq X_0 \leq cn(1+\lambda/2)/2$ and let $T:=\inf\{t \geq 0 \mid X_t \notin [cn(1-\lambda)/2,cn(1+\lambda)/2]\}$. Then $T \leq n^3$ with probability at most $O(1/n^3)$. 
\end{lemma}

\begin{proof}
Define $a:=-cn(1+\lambda)/2$, $b:=-cn(1+\lambda/2)/2$ and $Y_t:= -X_t$. Note $b>a$ and that $b-a=cn \lambda/4$. Denote by $A_i$ the event that $|Y_{t+1}-Y_t| = i$. Then $A_i$ can only happen if at least $i$ bits among the $L$ locations are flipped and $A_1$ coincides with the event to flip exactly one bit there. The latter happens with probability $c \cdot (1-1/n)^{cn-1}$. Hence, we obtain for $\mathcal{F}_t := Y_0, \ldots , Y_t$
\begin{align*}
& \Pr(A_1) \cdot \expect{Y_{t+1} - Y_t \text{ }\Big|\text{ } \mathcal{F}_t; A_1; a < Y_t < b} = c\left(1-\frac{1}{n}\right)^{cn-1} \cdot \expect{X_t - X_{t+1} \text{ }\Big|\text{  } \mathcal{F}_t; A_1; -b < X_t < -a}. \\
\intertext{Since under the condition $-b<X_t<-a$ the number of ones at the $L$ locations is at least $cn(1+\lambda/2)/2$, and the number of zeros is at most $cn-cn(1+\lambda/2)/2 = cn(1-\lambda/2)/2$, we obtain}
&\geq c\left(1-\frac{1}{n}\right)^{cn-1} \cdot \left(\frac{cn(1+\lambda/2)/2}{cn} - \frac{cn(1-\lambda/2)/2}{cn} \right) \geq c\left(1-\frac{1}{n}\right)^{cn-1} \cdot \frac{\lambda}{2} \geq \frac{c}{e} \cdot \frac{\lambda}{2} = \frac{c \lambda}{2e}.
\end{align*}
Further, we obtain for all $i \geq 1$ due to $\binom{cn}{i} \leq (cn)^i/i! \leq cn^i/i!$ and $|Y_{t+1} - Y_t| = |X_{t+1} - X_t| \leq i$ if $A_i$ occurs
\begin{align*}
\Pr(A_i) &\cdot \expect{|Y_{t+1} - Y_t| \text{ }\Big|\text{ } \mathcal{F}_t; A_i; a < Y_t < b} \leq i \cdot \binom{cn}{i} \frac{1}{n^i} \leq \frac{ci}{i!} = \frac{c}{(i-1)!} \leq \frac{2c}{2^{i-1}} = \frac{c}{2^{i-2}}
\end{align*}
and therefore, for $i \geq \lceil{\log(4e/\lambda)}\rceil + 3=:\beta \in \Theta(1)$ 
\begin{align*}
\sum_{i=\beta}^{cn}\Pr(A_i) &\cdot \expect{|Y_{t+1} - Y_t| \text{ }\Big|\text{ } \mathcal{F}_t; A_i; a < Y_t < b} \leq \sum_{i=\beta}^{cn}\frac{c}{2^{i-2}} \leq \frac{1}{2^{\beta-2}}\sum_{i=0}^{\infty}\frac{c}{2^i} =  \frac{c}{2^{\beta-3}} \leq \frac{c \lambda}{4e}.
\end{align*}
Furthermore, if $a < Y_t < b$, there are more ones than zeros at locations $L$ and hence, for every $0 \leq i \leq cn$, the probability of flipping $i$ zeros more than ones is at least the probability of flipping $i$ ones more than zeros. Therefore, for every $0 \leq i \leq cn$,
\begin{align*}
& \expect{(Y_{t+1} - Y_t) \mathbbm{1}\{|Y_{t+1} - Y_t| = i\} \text{ }\Big|\text{  } \mathcal{F}_t, a < Y_t < b} \\
&= \expect{(X_t - X_{t+1}) \mathbbm{1}\{|X_t - X_{t+1}| = i\} \text{ }\Big|\text{  } \mathcal{F}_t, -b < X_t < -a} \geq 0.
\end{align*}
Hence, we obtain with the law of total probability by applying the latter inequality for $2 \leq i \leq \beta-1$ and the first for $i=1$
\begin{align*}
& \expect{(Y_{t+1} - Y_t) \mathbbm{1}\{Y_{t+1} - Y_t \leq \beta-1\} \text{ }\Big|\text{  } \mathcal{F}_t; a < Y_t < b} \\
&= \sum_{i=1}^\infty \Pr(A_i) \cdot \expect{(Y_{t+1} - Y_t) \mathbbm{1}\{Y_{t+1} - Y_t \leq \beta-1\} \text{ }\Big|\text{  } A_i ; \mathcal{F}_t; a < Y_t < b}\\
&\geq \sum_{i=1}^{\beta-1} \Pr(A_i) \cdot \expect{(Y_{t+1} - Y_t) \text{ }\Big|\text{  } A_i ; \mathcal{F}_t; a < Y_t < b} -  \sum_{i=\beta}^\infty \Pr(A_i) \cdot \expect{|Y_{t+1} - Y_t| \text{ }\Big|\text{  } A_i ; \mathcal{F}_t; a < Y_t < b}\\
&\geq \frac{c \lambda}{2e} - \sum_{i=\beta}^{\infty} \Pr(A_i) \cdot \expect{|Y_{t+1} - Y_t| \text{ }\Big|\text{  } A_i ; \mathcal{F}_t; a < Y_t < b} \geq \frac{c \lambda}{2e} - \frac{c \lambda}{4e} = \frac{c \lambda}{4e} =: \varepsilon(\ell)=:\varepsilon
\end{align*}
and hence, property~(i) of Theorem~\ref{thm:negative-drift} is satisfied for $\kappa:=(\beta-1)/\varepsilon$. Note also $|X_t-X_{t+1}| \geq i$ requires to flip $i$ bits among locations in $L$ and hence
\begin{align*}
\Pr(Y_{t+1}-Y_t &\leq -3i \mid \mathcal{F}_t; a < Y_t) = \Pr(X_t-X_{t+1} \leq -3i \mid \mathcal{F}_t; a < Y_t) \\
&\leq \binom{cn}{3i} \frac{1}{n^{3i}} \leq \frac{1}{(3i)!} \leq \frac{1}{2^{3i-1}} \leq e^{-i}
\end{align*}
which implies property (ii) in Theorem~\ref{thm:negative-drift} for $r=3$. Further, 
\[
\chi := \min\left\{\frac{1}{2r},\frac{\varepsilon}{17r^2},\frac{1}{\kappa \varepsilon}\right\} = \min\left\{\frac{1}{6},\frac{c \lambda}{612e},\frac{1}{\beta-1}\right\} \geq \frac{c \lambda}{612e}
\]
and therefore, for $\ell=b-a=cn\lambda/4$ we obtain 
$$\chi \cdot \ell \geq \frac{c^2n \lambda^2}{2448e} \geq \frac{\ln(n) \lambda^2}{2448ed^2}$$
and $1/(\chi \varepsilon) = 4e/(c \lambda \chi) \leq 2448e^2/(c^2 \lambda^2)$ which implies 
$$2\ln(4/(\chi \varepsilon)) \leq 2\ln(9792e^2/(c^2 \lambda^2)) \leq 2\ln(9792e^2nd^2/(\ln(n) \lambda^2)) $$ and hence, we can assume that $\chi \ell \geq 2\ln(4/(\chi \varepsilon))$ since $d$ is sufficiently small. This implies also property~(iii) in Theorem~\ref{thm:negative-drift}. Hence, Theorem~\ref{thm:negative-drift} is applicable and we obtain for $T^*:= \min\{t \geq 0 \mid X_t \geq -a\} = \min\{t \geq 0 \mid Y_t \leq a\}$, $d$ sufficiently small and $n$ sufficiently large
\begin{align*}
&\Pr(T^* \leq n^3 \mid \mathcal{F}_0; Y_0 \geq b) \leq \Pr(T^* \leq e^{\chi \ell/4} \mid \mathcal{F}_0; Y_0 \geq b) = O(e^{-\chi \ell/4}) = O(1/n^3).
\end{align*}
By swapping the roles of ones and zeros we obtain by a completely symmetric argument on $\tilde{a}=cn(1-\lambda)/2$, $\tilde{b}=cn(1-\lambda/2)/2$ and the process $(X_t)_{t \geq 0}$ that 
$$\expect{(X_{t+1} - X_t) \mathbbm{1}\{X_{t+1} - X_t \leq \beta-1\} \text{ }\Big|\text{  } \mathcal{F}_t; \tilde{a} < X_t < \tilde{b}} \geq c \lambda/(4 e)$$
and the remaining properties (ii) to (iii) of Theorem~\ref{thm:negative-drift} are fulfilled for $r$, $\kappa$ and $\chi$ defined as above. This implies for $T^+:=\min\{t \geq 0 \mid X_t \leq \tilde{a}\}$
\begin{align*}
&\Pr(T^+ \leq n^3 \mid \mathcal{F}_0; X_0 \geq \tilde{b}) \leq \Pr(T^+ \leq e^{\chi \ell/4} \mid \mathcal{F}_0; X_0 \geq \tilde{b}) = O(e^{-\chi \ell/4}) = O(1/n^3).
\end{align*}
With a union bound on both intervals defined by $a,b$ and $\tilde{a},\tilde{b}$ respectively, the lemma is proven.
\end{proof} 

\section{A Many-Objective Royal-Road Function For One-Point Crossover}
In this section, we define the many-objective $\RRRfull$ function for one-point crossover, denoted as $m$-\RRRMO and give some elementary properties. To this end, let $m \in \mathbb{N}$ be an even number and $n$ be divisible by $5m/2$. For a bit string $x$, let $x:=(x^1, \ldots , x^{m/2})$, where each substring $x^j$ has length $2n/m$. 

\begin{definition}
\label{def:RRRMO-one-point}
Regarding the substring $x^j$ for $j \in [m/2]$, let 
\begin{itemize}
\item $B:=\{y \in \{0,1\}^{2n/m} \mid \ones{y} = 6n/(5m), \LZ(y)+\TZ(y)=4n/(5m)\}$, and 
\item $A:=\{y \in \{0,1\}^{2n/m} \mid \ones{y} = 8n/(5m), \LZ(y)+\TZ(y)=2n/(5m)\}$.
\end{itemize}
Regarding the whole bit string $x$ let
\begin{itemize}
\item $L:=\{x \in \{0,1\}^n \mid 0 \leq \ones{x^j} \leq 6n/(5m) \text{ for all } j \in [m/2], \ones{x^i} < 6n/(5m) \text{ for an } i \in [m/2]\},$
\item $M:=\{x \in \{0,1\}^n \mid \ones{x^j} = 6n/(5m) \text{ for all } j \in [m/2] \text{ and } x^i \notin B \text{ for an } i \in [m/2]\},$
\item $N:=\{x \in \{0,1\}^n  \mid x^j \in A \cup B \text{ for all } j \in [m/2]\}$.
\end{itemize}

Then the function class $m$-$\text{\RRRMO}: \{0,1\}^n \to \mathbb{N}_0^m,$ is defined as 
\[
m\text{-\RRRMO}(x) = (f_1(x), f_2(x), \ldots ,f_m(x))
\]
with 
\[
f_k(x) = g_k(x):= \begin{cases}
	\ones{x^{1+(k-1)/2}} \text{ if $k$ is odd,} \\
	\ones{x^{1+(k-2)/2}} \text{ if $k$ is even,}
\end{cases}
\]
if $x \in L$,
\[
f_k(x) = h_k(x) := g_k(x) + \begin{cases}
	\LZ(x^{1+(k-1)/2}) \text{ if $k$ is odd, } \\
	\TZ(x^{1+(k-2)/2}) \text{ if $k$ is even, }
\end{cases}
\]
if $x \in M$, 
\[
f_k(x) = 4n\vert{K(x)}\vert/(5m)+ h_k(x)
\]
if $x \in N$ where $K(x):=\{j \in [m/2] \mid x^j \in A\}$, and $f_k(x)=0$ otherwise. 
\end{definition}

In the $m$-objective $\RRRfull$ function, the bit string is divided into $m/2$ blocks, each of length $2n/m$ and the maximum in each objective is $2n/5+2n/m$. We briefly describe the typical behaviour of EMOAs on this function for $m=o(\sqrt{n})$ as it is also shown in Figure~\ref{fig:RRRMO-Prozess}.

\begin{figure}[t]
    \label{fig:RRRMO}
    \centering
    \includegraphics[scale=0.79]{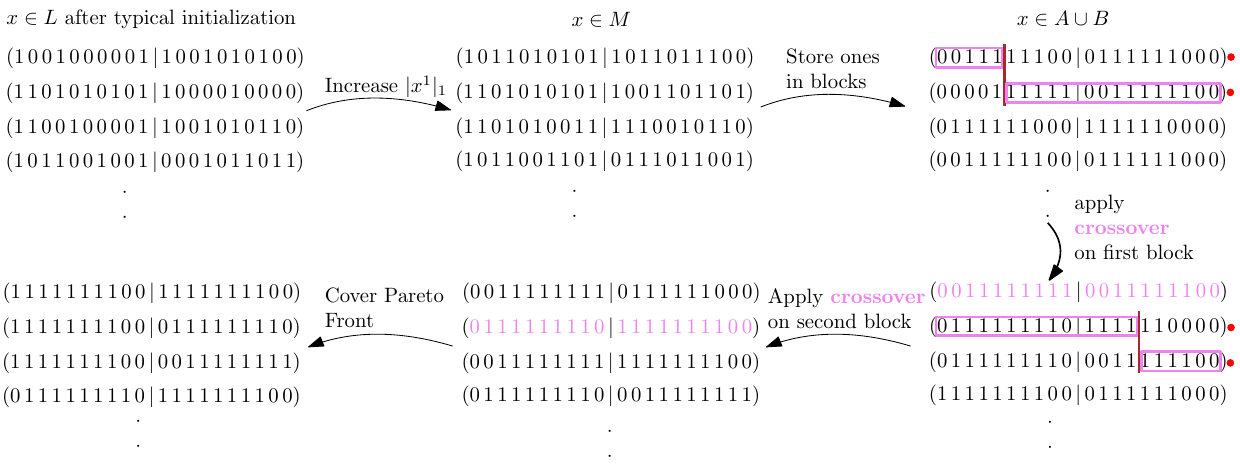}
    \caption{Once all ones are accumulated in cumulative blocks, one-point crossover can extend this block size by $2n/(5m)$. The two red dots represent the selected parents, while the brown line marks a potential cutting point. The segments of the bit strings inherited by the offspring are highlighted in purple. Firstly, crossover is applied on the first block and then, crossover is applied to the second one, using two parents that are already optimized in the first block. Finally, the entire Pareto front is completely covered.}
    \label{fig:RRRMO-Prozess}
\end{figure}

Algorithms that initialize their population uniformly at random typically start with search points $x$ where $0 < \ones{x^j} \leq 3/5 \cdot (2n/m) = 6n/(5m)$ for all $j \in [m/2]$, meaning that $x \in L \cup M \cup N$. For these points, we give a fitness signal that encourages increasing the number of ones in each $x^j$ to $6n/(5m)$, leading to search points $x \in M \cup N$. Next, we aim to collect these ones within a cumulative block by increasing the sum of leading and trailing zeros in each block $j$, thereby achieving $x^j \in B$ for all $j \in [m/2]$, or equivalently, $x \in N$. Finally, when $x \in N$, a strong fitness signal is assigned \emph{equally} to each objective based on $\vert{K(x)}\vert$, which is the number of blocks $j \in [m/2]$ where $x^j \in A$. We will see that the cardinality of $K(x)$ can be increased easily using recombination, in particular $1$-point crossover, by recombining two suitable individuals with $x^j \in B$ in order to create an individual $y$ with $y^j \in A$. The Pareto set is reached when $K(x)=[m/2]$ and can then easily be covered through mutation. EMOAs show a very similar behavior on $2$-\RRRMO as on the bi-objective version from~\citep{Dang2024}, but there are some differences in its structure: Instead of scaling the number of ones by a factor of $n$, we consider the leading ones and trailing zeros in each block only if the number of ones is $6n/(5m)$ or $8n/(5m)$ in each block. The reason for our approach is that the maximum possible value in each objective is $\Theta(n)$ instead of $\Theta(n^2)$ as it is the case in~\citep{Dang2024}. This implies, according to Lemma~\ref{lem:Reference-Points}, that for our function \nsgaIII requires a significantly lower number of reference points to protect good solutions.

The next two lemmas summarize important properties of $m$-\RRRMO. In the first we discuss dominance relations of search points and in the second we estimate the maximum number of mutually incomparable solutions.

\begin{lemma}
\label{lem:RRMO-properties}
The following properties are satisfied.
\begin{enumerate}
	\item[(1)] Let $x,y$ with $x \in L$ and $y \in M \cup N$. Then $y$ dominates $x$.
	\item[(2)] Let $\mathcal{Q}:=\{x \in \{0,1\}^n \mid x^j \in B \text{ for all } j \in [m/2]\}$. Then for all $x \in M$ there is $y \in \mathcal{Q}$ dominating $x$.
	\item[(3)] Let $x,y \in N$ with $\vert{K(x)}\vert<\vert{K(y)}\vert$. Then $y$ dominates $x$. 
	\item[(4)] Let $x,y \in N$ with $x \neq y$ and  $\vert{K(x)}\vert=\vert{K(y)}\vert$. Then $x$ and $y$ are incomparable.
	\item[(5)] The Pareto set $\mathcal{P}$ of $m$-\RRRMO is 
	\[
	\mathcal{P}:=\{x \in \{0,1\}^n \mid x^j \in A \text{ for all } j \in [m/2]\}.
	\]
\end{enumerate}
\end{lemma}

\begin{proof}
(1): Note that $f_k(y) \geq 6n/(5m)$ for every $k \in [m]$ since each block $j$ contains at least $6n/(5m)$ ones. On the other hand, $f_k(x) \leq 6n/(5m)$ and there is a block $i \in [m/2]$ with $\ones{x^i}< 6n/(5m)$, i.e. $f_{2i}(x)<6n/(5m)$. Hence, $y$ dominates $x$.

(2): Let $i \in [m/2]$ such that $\LZ(x^i)+\TZ(x^i)$ is not maximum (i.e. $\LZ(x^i)+\TZ(x^i) < 4n/(5m)$). Then there is a zero in $x^i$ not contributing to $\LZ(x^i)+\TZ(x^i)$ (i.e. between the leftmost and rightmost one in $x^i$). Hence, exchanging that zero with the leftmost one creates a search point $w$ with $w^j=x^j$ for $j \in [m/2] \setminus \{i\}$, $\ones{w^i} = \ones{x^i}$, $\LZ(w^i) = \LZ(x^i)+1$ and $\TZ(w^i) = \TZ(x^i)$. Hence, $w$ dominates $x$. Repeating this operation in $w^i$ until there is no such zero left gives the desired search point $y$ by the transitivity of dominance.

(3): Note that $h_k(x) \leq 2n/m$ for all $k \in [m]$ since in every block the sum of the number of ones and leading (trailing) zeros does not exceed $2n/m$. Since each block contains at least $6n/(5m)$ ones, $h_k(x) \geq 6n/(5m)$. Putting this together gives 
\begin{align*}
f_k(x) &= \frac{4n\vert{K(x)}\vert}{5m}+h_k(x) \leq \frac{4n\vert{K(x)}\vert}{5m} + \frac{2n}{m} = \frac{4n(\vert{K(x)}\vert+1)}{5m} -\frac{4n}{5m} + \frac{2n}{m} = \frac{4n(\vert{K(x)}\vert+1)}{5m} + \frac{6n}{5m}\\
&\leq \frac{4n\vert{K(y)}\vert}{5m} + \frac{6n}{5m} \leq \frac{4n\vert{K(y)}\vert}{5m}+h_k(y) = f_k(y).
\end{align*}
Since either $h_1(x) < 2n/m$ or $h_2(x)<2n/m$ (the leading and trailing zeros in block~1 are not $4n/(5m)$ at the same time since $\ones{x^1}=6n/(5m)$), the first inequality is strict for $k=1$ or $k=2$.

(4): At first suppose that $K(x) \neq K(y)$. Let $i \in K(x) \setminus K(y)$ and $j \in K(y) \setminus K(x)$. Then $\ones{x^i} = 8n/(5m)$ and $\LO(x^i) + \TZ(x^i) = 2n/(5m)$ whereas $\ones{y^i} = 6n/(5m)$ and $\LO(x^i) + \TZ(x^i) = 4n/(5m)$. Therefore, 
$$h_{2i-1}(x)+h_{2i}(x) = \frac{16n}{5m}+\frac{2n}{5m} = \frac{18n}{5m} > \frac{16n}{5m} = \frac{12n}{5m} + \frac{4n}{5m} = h_{2i-1}(y)+h_{2i}(y)$$
and analoguously $h_{2j-1}(x)+h_{2j}(x) =  16n/(5m) < 18n/(5m) = h_{2j-1}(y)+h_{2j}(y)$. Since $f_k(x) = 4n|K(x)|/(5m) + h_k(x)$ we see that $f_{2i-1}(x)+f_{2i}(x)>f_{2i-1}(y)+f_{2i}(y)$ and $f_{2j-1}(y)+f_{2j}(y)>f_{2j-1}(x)+f_{2j}(x)$. This implies incomparability of $x$ and $y$ (since if $x$ and $y$ were comparable then either $f_\ell(x) \leq f_\ell(y)$ for all $\ell \in [m]$ or $f_\ell(y) \leq f_\ell(x)$ for all $\ell \in [m]$ implying $f_k(x)+f_{k+1}(x) \leq f_k(y)+f_{k+1}(y)$ for every odd $k \in [m]$ in the first case and $f_k(x)+f_{k+1}(x) \geq f_k(y)+f_{k+1}(y)$ for every odd $k \in [m]$ in the second).

Suppose that $K(x) = K(y)$. Since $x^i$ and $y^i$ both consist either of a cumulative block of $6n/(5m)$ ones or $8n/(5m)$ ones, we find a block $i \in [m/2]$ where $\LZ(x^i) > \LZ(y^i)$ (and hence $\TZ(x^i)< \TZ(y^i)$) or $\LZ(x^i) < \LZ(y^i)$ (and hence $\TZ(x^i)> \LZ(y^i)$). This implies that in the first case, $f_{2i-1}(x)>f_{2i-1}(y)$ and $f_{2i}(x)<f_{2i}(y)$, while in the second case, $f_{2i-1}(x)<f_{2i-1}(y)$ and $f_{2i}(x)>f_{2i}(y)$. As a result, $x$ and $y$ are incomparable.

(5): Note that $\vert{K(z)}\vert=m/2$ for $z \in \mathcal{P}$ and $\vert{K(w)}\vert<m/2$ for every $w \notin \mathcal{P}$. Hence, by (2) every point $z \in \mathcal{P}$ dominates every point $w \notin \mathcal{P}$. By Property~(4) applied to $K(x)=[m/2]$, two distinct $x,y \in \mathcal{P}$ are incomparable.  
\end{proof}

\begin{lemma}\label{lem:RRMO-size-non-dom-set}
Let $S$ be a set of mutually incomparable solutions of $m$-\RRRMO. Let $E:=\{0,1\}^n \setminus (M \cup N)$ and for $\ell \in \{0, \ldots , m/2\}$ let $N_\ell:=\{x \in N \mid |K(x)|=\ell\}$. The following properties are satisfied.
\begin{enumerate}
\item[(1)]
Either $S \subset E$ or $S \subset M \cup N_0$ or there is $\ell \in [m/2]$ with $S\subset N_\ell$.
\item[(2)]
Suppose $S \subset E$. Then $|S| \leq (6n/(5m)+1)^{m/2} \leq (4n/(5m)+1)^{m-1}$.  
\item[(3)]
Suppose $S \subset M \cup N_0$. Then $|S| \leq (4n/(5m)+1)^{m-1}$.
\item[(4)] 
Suppose $S \subset N_\ell$ for $\ell>0$. Then 
$$|S| \leq \binom{m/2}{\ell}\left(\frac{2n}{5m}+1\right)^\ell \left(\frac{4n}{5m}+1\right)^{m/2-\ell} \leq \left(\frac{4n}{5m}+1\right)^{m-1}.$$ 
\end{enumerate}
Hence, $|S| \leq (4n/(5m)+1)^{m-1}$ in every case.
\end{lemma}

\begin{proof}
(1): This follows directly from Lemma~\ref{lem:RRMO-properties}(1), which states that every $x \in M \cup N_0$ dominates every $y \in E$, and from Lemma~\ref{lem:RRMO-properties}(3), which states that every $x \in N_{\ell}$ dominates every $y \in N_j$ for $j < \ell$.

(2): Every $v \in f(S)$ fulfills $v_k \in \{0, \ldots , 5n/(6m)\}$, and $v_k = v_{k+1}$ for every odd $k \in [m]$. Note also that $(u_1, \ldots , u_{m-2}) \neq (v_1, \ldots ,v_{m-2})$ for two different $u,v \in f(S)$. Otherwise, $x,y$ with $f(x)=u$ and $f(y)=v$ are comparable (due to $f(S) \subset f(E)$ and $v_{m-1}=v_m$ for every $v\in f(E)$). Therefore, 
$$\vert{S}\vert = \vert{f(S)}\vert \leq \left(\frac{6n}{5m}+1\right)^{m/2-1} \leq \left(\sqrt{\frac{6n}{5m}+1}\right)^{m-1}\leq \left(\frac{4n}{5m}+1\right)^{m-1}$$
where the last inequality holds because $\sqrt{x+1} \leq 2x/3+1$ for every $x \geq 0$.

(3): We have that $\ones{y^j}=6n/(5m)$ for every $j \in [m/2]$ and therefore $\LZ(y^j),\TZ(y^j) \in \{0, \ldots ,4n/(5m)\}$. This implies that each objective has at most $(4n/(5m)+1)$ different values. With Lemma~\ref{lem:dominance-preparation}(i) we see that $|S| \leq (4n/(5m)+1)^{m-1}$.

(4): Two distinct search points $x,y \in N_\ell$ are incomparable by Lemma~\ref{lem:RRMO-properties}(4). The number of distinct search points $y \in N_\ell$ is
\begin{align*}
\binom{m/2}{\ell} \left(\frac{2n}{5m}+1\right)^\ell \left(\frac{4n}{5m}+1\right)^{m/2-\ell} =:\alpha(\ell,m,n)
\end{align*}
since there are $\binom{m/2}{\ell}$ possibilities for $K(x)$ with $|K(x)|=\ell$ and $(2n/(5m)+1)^\ell \cdot (4n/(5m)+1)^{m/2-\ell}$ search points with a common $K(x)$: If $i \in K(x)$ there are $2n/(5m)+1$ different possibilities for $x^i$ and if $i \notin K(x)$ there are $4n/(5m)+1$ different possibilities. 

We have that $\alpha(1,2,n) = \alpha(0,2,n) \leq 4n/10+1 = (4n/(5m)+1)^{m-1}$ which gives the result for $m=2$. So we may assume that $m \geq 4$. We obtain 
\begin{align*}
\alpha(\ell,m,n) &\leq 2^{m/2-\ell} \cdot \binom{m/2}{\ell}\left(\frac{2n}{5m}+1\right)^{m/2} \leq 4^{m/2-1}\left(\frac{2n}{5m}+1\right)^{m/2} =:\beta(m,n)
\end{align*}
where the last inequality is obvious for $\ell=0$ and holds due to $\binom{m/2}{\ell} \leq 2^{m/2}$ for $\ell \geq 1$. We further estimate $\beta(m,n)$ by considering two cases.

\textbf{Case 1:} Suppose $4n/(5m) \leq 5$ (or in other words $2n/(5m) \in \{1,2\}$). Since $2n/(5m)+1 \leq 3$ we see $\sqrt{2n/(5m)+1} \leq 2$ and therefore
\begin{align*}
    \beta(m,n) &\leq 2 \cdot 4^{m/2-1}\left(\frac{2n}{5m}+1\right)^{(m-1)/2} = 4^{(m-1)/2}\left(\frac{2n}{5m}+1\right)^{(m-1)/2}\\
    &= \left(\sqrt{\frac{8n}{5m}+4}\right)^{m-1} \leq \left(\frac{4n}{5m}+1\right)^{m-1}
\end{align*} 
where the last inequality holds due to $\sqrt{2x+4} \leq x+1$ for $x \geq 2$.

\textbf{Case 2:} Suppose $4n/(5m) > 5$ (or in other words $2n/(5m)>2$). Then 
\begin{align*}
\beta(m,n) \leq \left(\sqrt{\frac{8n}{5m}+4}\right)^m \leq \left(\frac{4n}{5m}+1\right)^{m-1}
\end{align*}
by Lemma~\ref{lem:inequality} on $k=m$ by pluggin in $x=4n/(5m) > 5$. This concludes the proof.
\end{proof}

Lemma~\ref{lem:RRMO-size-non-dom-set} refines the result from the conference version~\citep{Opris2025}. In that version it was shown that for $m$-\RRRMO the maximum number of mutually incomparable solutions is bounded by $c \cdot (4n/(5m)+1)^{m-1}$ for a suitable constant $c \geq 1$ where the number $m$ of objectives is constant.

\section{Runtime Analyses of GSEMO and NSGA-III with One-Point Crossover on $m$-\RRRMO}

We now show that both NSGA-III and GSEMO exhibit a significant performance gap depending on whether one-point crossover is used, across a wide range of the number $m$ of objectives. In particular, when $m$ is constant, both algorithms with one-point crossover are able to find the entire Pareto front in expected polynomial time, whereas without crossover, they require exponential time even to find a single Pareto-optimal point.

\subsection{Analysis of NSGA-III}
We begin with NSGA-III, as our analysis builds on the conference version~\citep{Opris2025}, where NSGA-III has already been analyzed on $m$-\RRRMO for a constant number $m$ of objectives.  

\begin{theorem}
\label{thm:Runtime-Analysis-NSGA-III-mRRMO-onepoint}
Let $m \in \mathbb{N}$ be divisible by $2$ and $n$ be divisible by $5m/2$. Then the algorithm \nsgaIII (Algorithm~\ref{alg:nsga-iii}) with $p_c \in (0,1)$, $\varepsilon_{\text{nad}} \geq 2n/5+2n/m$, a set $\refer$ of reference points as defined above for $p \in \mathbb{N}$ with $p \geq 2m^{3/2}(2n/5+2n/m)$, crossover probability $p_c \in (0,1)$ and a population size $\mu \geq (4n/(5m)+1)^{m-1}$, $\mu = 2^{O(n^2)}$, finds the Pareto set of $f:=m\text{-\RRRMO}$ in expected $O(n^3/(1-p_c)+m^2/p_c)$ generations and $O(\mu n^3/(1-p_c)+\mu m^2/p_c)$ fitness evaluations.  
\end{theorem}

\begin{proof}
Note that $f_{\max}=2n/5+2n/m$ by noticing that $\vert{K(x)}\vert \leq m/2$ and $\ones{x^j}+\LZ(x^j),\ones{x^j}+\TZ(x^j) \leq 2n/m$ for every $x \in \{0,1\}^n$. So during the whole optimization procedure we may apply Lemma~\ref{lem:Reference-Points} to obtain that non-dominated solutions are never lost which means that for every $x \in P_t$ there is $x' \in P_{t+1}$ weakly dominating $x$. Further, we use the method of typical runs~\citep[Section~11]{Wegener2002} and divide the run into several phases. For every phase we compute the expected waiting time to reach one of the next phases. A phase can be skipped if the goal of a later phase is achieved.

\textbf{Phase 1:} Create $x$ with $f(x) \neq 0$ for every $j \in [m]$.\\
Let $x$ be initialized uniformly at random. By a classical Chernoff bound the probability that $0 < \ones{x^j} \leq 6n/(5m)$ for a fixed $j \in [m/2]$ is $p^{\text{init}} = 1-e^{-\Omega(n/m)}$ since the expected number of ones in block $j$ of a search point is $1/2 \cdot 2n/m = n/m$ after initialization. Note that $p^{\text{init}} \geq 1/2$ regardless of the value of $m$. 
We consider two cases to estimate the expected number of generations required to complete this phase by $1+o(1)$.

\textbf{Case 1:} Suppose $m=2$. Then $p^{\text{init}} = 1- e^{-\Omega(n)}$ and hence, the probability that every individual $x$ satisfies $\ones{x}=0$ or $\ones{x} > 3n/5$ is $e^{-\Omega(\mu n)}$, because all individuals are initialized independently. Since the probability is at least $n^{-n}$ to create any individual with mutation (no matter if crossover is executed) and hence, one with fitness distinct from $0$, the expected number of generations to finish this phase in this case is at most $1 - e^{-\Omega(\mu n)} + n^n \cdot e^{-\Omega(\mu n)} = 1+o(1)$.

\textbf{Case 2:} Suppose that $m>2$ or, in other words, $m \geq 4$ since $m$ is even. After initializing $x$, the probability that there is $j \in [m/2]$ with $\ones{x^j}> 6n/(5m)$ or $\ones{x^j}=0$ is $1-(p^{\text{init}})^{m/2}$. Hence, the probability that all individuals have such a block $j$ after initialization is at most 
\begin{align*}
(1-(p^{\text{init}})^{m/2})^\mu &\leq e^{-\mu \cdot (p^{\text{init}})^{m/2}} \leq e^{-(4n/(5m)+1)^{m-1} /2^{m/2}} = e^{-1/\sqrt{2} \cdot (4n/(5m)+1)^{m-1} /\sqrt{2}^{m-1}} \\
&= e^{-1/\sqrt{2} \cdot (4n/(5 \sqrt{2}m)+1/\sqrt{2})^{m-1}}
\end{align*}
by using $1-x \leq e^x$ for every $x \in \mathbb{R}$, $\mu \geq ((4n/(5m)+1)^{m-1}$ and $p^{\text{init}} \geq 1/2$. We also have $4n/(5 \sqrt{2} m)-n/(5m) = n/(5m)(4/\sqrt{2}-1)\geq 1/2 \cdot (4/\sqrt{2}-1) = \sqrt{2}-1/2 \geq 1-1/\sqrt{2}$ and therefore $4n/(5 \sqrt{2}m)+1/\sqrt{2} \geq n/(5m)+1$. Hence, we further estimate 
\begin{align*}
(1-(p^{\text{init}})^{m/2})^\mu &\leq e^{-1/\sqrt{2} \cdot (n/(5m)+1)^{m-1}} \leq e^{-1/\sqrt{2} \cdot (n/20+1)^3} =e^{-\Omega(n^3)}
\end{align*}
where we use the monotonicity of the function $]0,\infty[ \to \mathbb{R}, x \mapsto (n/(5x)+1)^{x-1},$ to derive the last inequality (see Lemma~\ref{lem:function-monotone}(i)). 
So the expected number of generations to finish this phase is at most 
$(1-e^{-\Omega(n^3)})+ e^{-\Omega(n^3)} \cdot n^n = 1+o(1)$.

\textbf{Phase 2:} Create $x$ with $\ones{x^j} = 6n/(5m)$ for all $j \in [m/2]$.\\ 
Let $O_t:=\max\{\ones{x} \mid x \in P_t, f(x) \neq 0\}$, i.e. the maximum number of ones occuring in an indvidual in the current population. Note that $0 \leq \ones{x}^i \leq 6n/(5m)$ for every $i \in [m/2]$ and there is $j \in [m/2]$ with $\ones{x}^j > 0$ (since $f(x) \neq 0$). Hence, $1 \leq O_t < m/2 \cdot 6n/(5m) = 3n/5$. Further, $O_t$ cannot decrease by Lemma~\ref{lem:Reference-Points} since a solution $x$ with $\ones{x}=O_t$ is non-dominated. 
To increase $O_t$ in one trial it suffices to choose a parent $z \in P_t$ with $\ones{z}=O_t$ (prob. at least $1-(1-1/\mu)^2 \geq 1/\mu$), omit crossover (prob. $(1-p_c)$) and flip one of $2n/m-\ones{x^j} \geq 4n/(5m)$ zero bits to one in a block $j$ with $\ones{x^j} < 6n/(5m)$ (prob. at least $\binom{4n/(5m)}{1} \cdot 1/n \cdot (1-1/n)^{n-1} = 4/(5m) \cdot (1-1/n)^{n-1} \geq 4/(5me)$). Since in each generation $\mu/2$ pairs of two individuals are generated independently of each other, the probability to increase $O_t$ is at least $1-(1-r_t)^{\mu/2} \geq \frac{\mu r_t}{2}/(1+\tfrac{\mu r_t}{2})$ for $r_t:=4(1-p_c)/(5\mu me)$ (for this inequality see Lemma~\ref{lem:Badkobeh}). Hence, the expected number of generations to complete this phase is at most $$\left(\frac{3n}{5}-1\right)\left(1+\frac{2}{\mu r_t}\right) = \left(\frac{3n}{5}-1\right)\left(1+\frac{10me}{4(1-p_c)}\right) = O\left(\frac{n m}{1-p_c}\right).$$
\textbf{Phase 3:} Create $x$ with $x^j \in B$ (which means that $\LZ(x^j)+\TZ(x^j) = 4n/(5m)$) for each $j \in [m/2]$.

Let $W_t:=\{x \in P_t \mid \ones{x^j}=6n/(5m) \text{ for all } j \in [m/2]\}$ and for $k \in [m]$, $x\in \{0,1\}^n$ let $T_k(x) = \LZ(x^{1+(k-1)/2})$ if $k$ is odd and $T_k(x) = \TZ(x^{1+(k-2)/2})$ otherwise. Note that for $x \in W_t$ we have $f_k(x) = 6n/(5m) + T_k(x)$ and $0 \leq \LZ(x^j)+\TZ(x^j) \leq 4n/(5m)$ for each $j \in [m/2]$. Set $\alpha_t=\max\{\sum_{i=1}^{m/2} (T_{2i-1}(x)+T_{2i}(x)) \mid x \in W_t\}$. Then $\alpha_t \in \{0, \ldots , 2n/5-1\}$ and this phase is finished if $\alpha_t$ becomes $2n/5$. According to Lemma~\ref{lem:Reference-Points}, $\alpha_t$ cannot decrease since a corresponding solution $w$ with value $\alpha_t$ has the largest sum in all objectives and is therefore non-dominated. In $w$ the total number of zeros not contributing to any $\LZ(w^j)+\TZ(w^j)$ for a $j \in [m/2]$ (i.e. which are between two ones from $w^j$) is $2n/5-\alpha_t=:\sigma_t$. To increase $\alpha_t$ in one trial, is suffices to choose such a solution $w$ from $P_t$, omit crossover and execute mutation as follows: Flip one of the $\sigma_t$ zeros to one and the leftmost one bit in the same block $i$ as the flipped zero to zero to increase $\LZ(w^i)+\TZ(w^i)$ while keeping the remaining bits unchanged (prob. $\sigma_t/(n^2) \cdot (1-1/n)^{n-2} \geq \sigma_t/(en^2)$). Then $\LZ(w^j)+\TZ(w^j)$ remains unchanged for every $j \in [m/2] \setminus \{i\}$ and hence, $\alpha_t$ increases. 
Let $r_t:=(1-p_c)\sigma_t/(en^2\mu)$. Then the probability of increasing $\alpha_t$ in one generation is at least $1-(1-r_t)^{\mu/2} \geq \tfrac{r_t\mu}{2}/(1+\tfrac{r_t \mu}{2})$. Therefore, the expected number of generations required to increase $\alpha_t$ is at most $1+2/(r_t \mu) = 1+2en^2/((1-p_c)\sigma_t)$. 
Since $\sigma_t \in [2n/5]$, the expected number of generations to obtain $\alpha_t=2n/5$ in total is at most 
$$\sum_{j=1}^{2n/5}\left(1+\frac{2en^2}{(1-p_c)j}\right) \leq \frac{2n}{5}+\frac{2en^2(\ln(2n/5)+1)}{1-p_c} = O\left(\frac{n^2\log(n)}{1-p_c}\right)$$
where the first inequality holds due to $\sum_{k=1}^\alpha 1/k \leq \ln(\alpha)+1$ for $\alpha \in \mathbb{N}$.

\textbf{Phase 4:} Find all individuals in $\mathcal{Q}:=\{x \in \{0,1\}^n \mid x^j \in B \text{ for all } j \in [m/2]\}$.

Note that $\mathcal{Q} = \{x \in N \mid K(x) = \emptyset\}$. For a specific search point $w \in \mathcal{Q}$ not already found we first upper bound the probability by $e^{-\Omega(n)}$ that a solution $x$ with $x=w$ has not been created after $8en^3/(1-p_c)$ generations. Let $D_t:=\{x \in P_t \mid x^j \in B \text{ for every } j \in [m/2]\}$. We consider $d_t:= \min_{x \in D_t} \sum_{i \in [m/2]} H(x^i,w^i)/2$. For $x \in D_t$ we have that $H(x^i,w^i)$ is even and $H(x^i,w^i) \leq 8n/(5m)$ for all $i \in [m/2]$ (since $\ones{x^i}=\ones{y^i}=6n/(5m)$ and every $x^i$ has length $2n/m$). This implies $0 < d_t \leq m/2 \cdot 4n/(5m) = 2n/5=:s(n)$. Since a solution $x \in D_t$ is non-dominated (by Lemma~\ref{lem:RRMO-properties}(2),~(4)), $d_t$ cannot increase (by Lemma~\ref{lem:Reference-Points}). Note that we created $w$ if $d_t=0$. For $1 \leq \beta \leq s(n)$, define the random variable $X_\beta$ as the number of generations $t$ with $d_t=\beta$. Then the total number of generations required to find a solution $x$ with $x=w$ is at most $X=\sum_{\beta=1}^{s(n)} X_\beta$. Fix $y \in D_t$ with $\sum_{i \in [m/2]} H(w^i,y^i)/2=d_t \neq 0$. 
To decrease $d_t$, it suffices to choose $y$ as a parent, omit crossover and flip two specific bits during mutation in order to shift a block of ones in $y$ in that direction of the corresponding block of $w$   
(prob. $1/n^2 \cdot (1-1/n)^{n-2} \geq 1/(en^2)$). Hence, for $a_t:=(1-p_c)/(\mu en^2)$, the probability to decrease $d_t$ in one generation is at least $1-(1-a_t)^{\mu/2} \geq \frac{a_t \mu/2}{1+a_t \mu/2} \geq a_t \mu/4 = (1-p_c)/(4en^2)$. Thus, the random variable $X$ is stochastically dominated by an independent sum $Z := \sum_{\beta=1}^{s(n)} Z_\beta$ of geometrically distributed random variables $Z_\beta$ with success probability $q:=q_\beta:=(1-p_c)/(4en^2)$. Note that $\expect{Z}=4en^2s(n)/(1-p_c)$. Now we use Theorem~\ref{thm:Doerr-dominance}(1): For $d:=\sum_{\beta=1}^{s(n)} 1/q_\beta^2 = 16e^2n^4s(n)/(1-p_c)^2$ and $\lambda \geq 0$ we obtain
\[
\Pr(Z \geq \expect{Z} + \lambda) \leq \exp\left(-\frac{1}{4} \min\left\{\frac{\lambda^2}{d}, \lambda q \right\}\right).
\]
For $\lambda = 4en^3/(1-p_c)$ we obtain 
$$\Pr(X \geq 8en^3/(1-p_c)) \leq \Pr(Z \geq 8en^3/(1-p_c)) \leq \exp\left(-\frac{1}{4} \min\left\{\frac{n^2}{s(n)}, n\right\}\right) = \exp\left(-\frac{n}{4}\right).$$
By a union bound over all possible~$w$, the probability that $\mathcal{Q}$ is completely found after $8en^3/(1-p_c)$ generations is at most
$$(4n/(5m)+1)^{m/2} \cdot e^{-n/4} = e^{-n/4+m/2 \cdot \ln(4n/(5m)+1)} \leq e^{-n/4+n \cdot \ln(3)/5} = e^{-\Omega(n)}$$
where the last inequality holds due to $m \leq 2n/5$ and Lemma~\ref{lem:function-monotone}(ii): The function $]0,\infty[ \to \mathbb{R}, x \mapsto x \cdot \ln(4n/(10x)+1)$ is strictly monotone increasing. 

If this does not happen, we can repeat the argument. Since the expected number of periods is $1+o(1)$, the expected number of generations to finish this phase is at most $(1+o(1))(8en^3/(1-p_c)) = O(n^3/(1-p_c))$.

For defining the next phases define $\gamma_t:=\max\{\vert{K(x)}\vert \mid x \in P_t \cap N\}$. 
Suppose that $\gamma_t \in \{0, \ldots , m/2-1\}$ which means that there is an individual $z \in P_t$ with $z \in \{x \in N \mid |K(x)| = \gamma_t\}$, but no corresponding individual $w \in P_t$ with $w \in N$ and $|K(w)| > \gamma_t$. By Lemma~\ref{lem:RRMO-properties}(3) and Lemma~\ref{lem:Reference-Points}, $\gamma_t$ cannot decrease.

\textbf{Phase $\gamma_t$+5:} Create an individual $x$ with $\vert{K(x)}\vert = \gamma_t+1$.

Note that Phase~5 starts when $\gamma_t=0$. If Phase $m/2-1+5=m/2+4$ is finished, $K(x) = [m/2]$ and hence, by Lemma~\ref{lem:RRMO-properties}(5), a Pareto optimal search point is found. Let $S_t=\{x \in P_t \mid n \in N \text{ and } |K(x)| = \gamma_t\}$. Note that $S_t \neq \emptyset$ and an individual $x$ is non-dominated if and only if $x \in S_t$: This follows from Lemma~\ref{lem:RRMO-properties}(3) in the case where $\gamma_t>0$. If $\gamma_t=0$, the set $\mathcal{Q}$ is completely found (as we have progressed through Phase~4), implying that $\mathcal{Q} \subset S_t$ and hence, for all $w \notin S_t$ (i.e. $w \in L \cup M$) there exists some $z \in S_t$ that dominates $w$ by Lemma~\ref{lem:RRMO-properties}(2). Hence, $|S_t|$ cannot decrease.

\textbf{Subphase A:} Every $x \in P_t$ satisfies $x \in S_t$.

In expected $O(n/(1-p_c))$ generations we have that $|S_t| \geq n$ since the probability is at least $(1-p_c)(1-1/n)^n/\mu \geq (1-p_c)/(4 \mu)=:\beta$ to increase $|S_t|$ in one trial (choose one $x \in S_t$ as a parent, omit crossover and do not flip any bit during mutation) and hence, at least 
$$1-(1-\beta)^{\mu/2} \geq \frac{\beta \mu/2}{1+\beta \mu/2} = \frac{(1-p_c)/8}{1+(1-p_c)/8} \geq \frac{1-p_c}{16}$$
to increase $|S_t|$ in one generation. So suppose that $|S_t| \geq n$. Denote by $X_t$ the number of newly created individuals $x$ with $x \in S_t$ in the next $\ell:=\lceil{1/(1-p_c)}\rceil$ generations. Then, we have $\expect{X_t} \geq |S_t|/4$. This follows from the fact that, in $\lceil{1/(1-p_c)}\rceil$ generations, the expected number of mutation only steps is at least $\mu$ and therefore, the expected number of individuals chosen as parents from $S_t$ in mutation only steps is at least $\mu \cdot |S_t|/\mu = |S_t|$ where the parent is then cloned with probability $(1-1/n)^n \geq 1/4$ to produce an offspring in $S_t$. 
Hence, by a classical Chernoff bound, $\Pr(X_t \leq 0.5\expect{X_t}) \leq e^{-\Omega(|S_t|)} = e^{-\Omega(n)}$ and with probability $1-e^{-\Omega(n)}$ we have that $|S_{t+\ell}| \geq \min\{|S_t|+|S_t|/8,\mu\} = \min\{9S_t/8,\mu\}$ and by a union bound over $\lceil{\log_{9/8}(\mu/|S_t|)}\rceil=O(n^2)$ (since $\mu \in 2^{O(n^2)}$) periods of $\lceil{1/(1-p_c)}\rceil$ generations, we obtain with probability at least $1-e^{-\Omega(n)}$ 
that Subphase~A is finished in at most $\lceil{\log_{9/8}(\mu/|S_t|)}\rceil \cdot \lceil{1/(1-p_c)}\rceil = O(\log(\mu/|S_t|)/(1-p_c))=O(n^2/(1-p_c))$ generations. If this does not happen, we repeat the above argument and obtain an expected number of $(1+o(1))O(n^2/(1-p_c)) = O(n^2/(1-p_c))$ generations.

\textbf{Subphase B:} Create an individual $y$ with $\vert{K(y)}\vert \geq \gamma_t+1$.

Let $t_0$ be the first generation when $|S_{t_0}| = \mu$. Regarding this generation $t_0$, we define
$$\alpha:=\min\{j \in [m/2] \mid \text{there is $x \in P_{t_0}$ with }x^j \in B\}.$$
Note that $\alpha =1$ if $\ell_t = 0$. For $t \geq t_0$ we say that an individual $x$ is \emph{conform} if $|K(x)|=\gamma_t$, $x^j \in A$ for $j<\alpha$, and $x^j \in B$ for $j = \alpha$. By definition of $\alpha$ there is at least one conform individual $z$ in generation $t_0$. Note also that conform individuals cannot be lost between generations since they are non-dominated. The goal is to bound the expected waiting time for increasing $\gamma_t$ by recombining two suitable conform individuals. However, those may need to be created in a preceding step, as discussed in the following claim.

\begin{claim}
\label{claim:conformity-NSGAIII}
After $O(n^3/(m(1-p_c)))$ generations in expectation the following holds.
\begin{itemize}
    \item[(i)] For every bit string $z \in B \subset \{0,1\}^{2n/m}$ there is a conform $y \in P_t$ with $y^\alpha = z$.
    \item[(ii)] Let $\alpha>1$. Then for every bit string $w \in A \subset \{0,1\}^{2n/m}$ there is a conform $x \in P_t$ with $x^{\alpha-1} = w$.
\end{itemize}
\end{claim}

\begin{proofofclaim}
We show that we obtain (i) and (ii) after $O(n^3/(m(1-p_c)))$ generations in expectation, respectively.

(i): To obtain the statement, we may perform the following operations at most $4n/(5m)$ times. Select a conform $x \in P_t$, such that $H(x^\alpha,z) = 2$ for a $z \in B$, where no conform $w \in P_t$ satisfies $w^\alpha=z$, as parent. Then omit crossover, and flip two specific bits during mutation in block $\alpha$ to create such a $z$. All this happens with probability at least $(1-p_c)/(4en^2)$ in one generation (compare also with Subphase~4 above).

(ii): Note that we may execute the following operations at most $2n/(5m)$ times. Select a conform $x \in P_t$, such that $H(x^{\alpha-1},z) = 2$, where no conform $w \in P_t$ satisfies $w^{\alpha-1}=z$, as parent, omit crossover and flip two specific bits during mutation in block $\alpha-1$ to create such a $z$ which can be done with probability at least $(1-p_c)/(4en^2)$. This proves the claim.
\end{proofofclaim}

Now, suppose that properties (i) and (ii) from Claim~\ref{claim:conformity-NSGAIII} hold after passing through the preceding $O(n^3/(m(1-p_c)))$ generations in expectation. Note that it may happen that
$$\min\{j \in [m/2] \mid \text{there is $x \in P_t$ with }x^j \in B\} < \alpha.$$ Now we move on to the recombination step. We say that a pair $(x_1,x_2)$ of individuals is \emph{good} if, with probability $\Omega(1/m)$, applying one-point crossover to $(x_1,x_2)$ followed by mutation creates an individual $y \in N$ with $|K(y)| \geq \gamma_t+1$. In the next claim we estimate the number of such good pairs.
\begin{claim}
\label{claim:good-pairs-NSGAIII}
Every $x \in P_t$ is part of a good pair. Therefore, there are at least $\mu/2$ good pairs of individuals. 
\end{claim}

\begin{proofofclaim}
Let $x \in P_t$. We consider three cases where the case "$\alpha=1$" is covered by Case~1 and Case~2. The possibilities of recombinations are depicted in Figure~\ref{fig:Conform}.

\begin{figure}[t]
    \centering
    \includegraphics[scale=0.9]{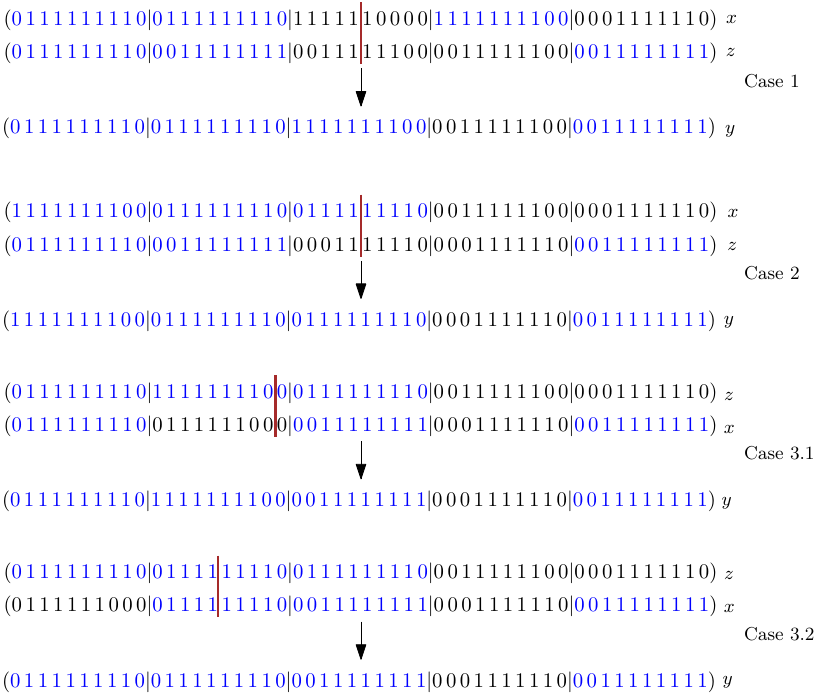}
    \caption{Illustration of the different recombination possibilities after generating suitable conform individuals for parameters $n = 50$, $m = 10$, $\gamma_t = 3$, and $\alpha = 3$. For any given $x \in P_t$, a conform individual $z \in P_t$ can always be found such that the pair $(x, z)$ is good. With probability $\Omega(1/m)$, one-point crossover produces an offspring $y \in N$ such that $|K(y)| > \gamma_t$. Substrings $x^j,z^j \in A$ are highlighted in blue, while the remaining blocks are shown in black. A potential crossover point is marked in brown.}
    \label{fig:Conform}
\end{figure}

\textbf{Case 1:} Suppose that $x^\alpha \in B$ and $x^\ell \in A$ for all $1 \leq \ell<\alpha$ (i.e. $x$ is conform). We consider two more subcases.

\textbf{Case 1.1:} There is $a \in \{0, \ldots , 2n/(5m)\}$ such that $x^\alpha=0^a1^{6n/(5m)}0^{4n/(5m)-a}$.

By Claim~\ref{claim:conformity-NSGAIII}(i) there is a conform $z \in P_t$ with $z^\alpha = 0^{2n/(5m)+a}1^{6n/(5m)}0^{2n/(5m)-a}$. Performing one-point crossover on $(x,z)$ with cutting point in $\{(\alpha-1)(2n)/m+2n/(5m) + a, \ldots , (\alpha-1)(2n)/m+6n/(5m)+a\}$ (i.e. in block $\alpha$ at position $b \in \{2n/(5m) + a, \ldots , 6n/(5m)+a\}$) to create $y^\alpha=0^a1^{8n/(5m)}0^{2n/(5m)-a}$ and omitting mutation afterwards, which happens with probability $(4n/(5m)+1)/(n+1) \cdot (1-1/n)^n \geq (4n/(5m)+1)/(4(n+1)) = \Omega(1/m)$, creates an individual $y$ with $|K(y)| \geq \gamma_t+1$: We have $y^j=x^j$ for $j<\alpha$, $y^\alpha \in A$, and $y^j=z^j$ for $j>\alpha$, and there are $\gamma_t-(\alpha-1)=\gamma_t-\alpha+1$ blocks $j$ in $z$ to the right of block $\alpha$ with $z^j \in A$ since $z$ is conform. Hence, $(x,z)$ is good. 

\textbf{Case 1.2:} There is $a \in \{1, \ldots , 2n/(5m)\}$ such that $x^\alpha=0^{2n/(5m)+a}1^{6n/(5m)}0^{2n/(5m)-a}$.

By Claim~\ref{claim:conformity-NSGAIII}(i) there is a conform $z \in P_t$ with $z^\alpha = 0^a1^{6n/(5m)}0^{4n/(5m)-a}$. The same argument as in Case~1.1 shows that $(z,x)$ is good.

\textbf{Case 2:} Suppose that $x^\ell \in A$ for every $1 \leq \ell \leq \alpha$. 

Let $a \in \{0,\ldots , 2n/(5m)\}$ such that $x^\alpha=0^a1^{8n/(5m)}0^{2n/(5m)-a}$. By Claim~\ref{claim:conformity-NSGAIII}(i) there exists a conform $z \in P_t$ with $z^\alpha=0^{2n/(5m) + a}1^{6n/(5m)}0^{2n/(5m)-a}$. Then performing one-point crossover on $(x,z)$ with cutting point in $\{(\alpha-1)(2n)/m+2n/(5m) + a, \ldots , 2n\alpha/m\}$, and omitting mutation afterwards (happening with probability at least $(6n/(5m)+1)/(4(n+1)) = \Omega(1/m)$) creates a $y$ with $y^\alpha = 0^a1^{8n/(5m)}0^{2n/(5m)-a}$ and $|K(y)| \geq \gamma_t+1$: We have $y^\ell \in A$ for every $\ell \in [\alpha]$ (since $x^\ell \in A$ for every $1 \leq \ell \leq \alpha$), $y^\ell = z^\ell$ for every $\ell>\alpha$ and since $z$ is conform, there are $\gamma_t-\alpha+1$ blocks $j$ in $z$ right to block $\alpha$ with $z^j \in A$. Hence, $(x,z)$ is good.

\textbf{Case 3:} Suppose that there is $k < \alpha$ with $x^k \in B$. We consider two subcases. 

\textbf{Case 3.1:} Suppose that $x^{\alpha-1} \in B$. Let $a \in \{0, \ldots , 4n/(5m)\}$ with $x^{\alpha-1}=0^a1^{6n/(5m)}0^{4n/(5m)-a}$. With Claim~\ref{claim:conformity-NSGAIII}(ii) we may choose a conform $z \in P_t$ with $z^{\alpha-1} = 1^{8n/(5m)}0^{2n/(5m)}$ if $a \leq 2n/(5m)$ and $z^{\alpha-1} = 0^{a-2n/(5m)}1^{8n/(5m)}0^{4n/(5m)-a}$ otherwise. Performing one-point crossover on $(z,x)$ with cutting point in block $\alpha-1$ at $\{2n(\alpha-2)/m+8n/(5m), \ldots , 2n(\alpha-1)/m\}$ if $a \leq 2n/(5m)$ and $\{2n(\alpha-2)/m+a, \ldots ,2n(\alpha-2)/m + 6n/(5m)+a\}$ otherwise, and omitting mutation afterwards (happening with probability $\Omega(1/m)$) generates an individual $y$ with $y^j=z^j$ for $j < \alpha-1$, $y^{\alpha-1}=z^{\alpha-1} \in A$, and $y^j=x^j$ for $j>\alpha-1$. We have that $|K(y)| \geq \gamma_t+1$, since the number of blocks $j < \alpha$ with $y^j = z^j \in A$ is at least by one greater than the number of blocks $j < \alpha$ with $x^j \in A$ (due to the conformity of $z$). Hence, $(z,x)$ is good.

\textbf{Case 3.2:} Suppose that $x^{\alpha-1} \in A$. Let $a \in \{0, \ldots , 2n/(5m)\}$ with $x^{\alpha-1}=0^a1^{8n/(5m)}0^{2n/(5m)-a}$. Consider a conform $z \in P_t$ with $z^{\alpha-1}=x^{\alpha-1}$. Performing one-point crossover on $(z,x)$ with cutting point in $\{2n(\alpha-2)/m, \ldots , 2n(\alpha-1)/m\}$ and omitting mutation afterwards (happening with probability $\Omega(1/m)$) generates an individual $y$ with $y^j=z^j$ for $j \leq \alpha-1$, $y^\alpha=z^\alpha$ and $y^j=x^j$ for $j>\alpha-1$. The number of blocks $j < \alpha$ with $y^j \in A$ is again at least by one greater than the number of blocks $j < \alpha$ with $x^j \in A$ which implies that $(z,x)$ is good. 

All these cases together prove the claim.
\end{proofofclaim}

Now, to create an individual $y$ with $|K(y)|=\gamma_t+1$ in one trial, one can choose a pair of good individuals as parents (prob. at least $(\mu/2)/\mu^2 = 1/(2\mu)$), and create a desired $y$ via one-point crossover and omitting mutation (prob. at least $s=\Omega(p_c/m)$). Hence, the probability to create such a $y$ in one generation is at least $1-(1-s/(2\mu))^{\mu/2} \geq \frac{s/4}{1+s/4} \geq s/8 = \Omega(p_c/m)$. Therefore, the expected number of generations required to complete this phase is $O(n^3/(m(1-p_c)) + m/p_c)$, considering Claim~\ref{claim:conformity-NSGAIII} as a prerequisite for generating suitable conform individuals that serve as parents for crossover.

\textbf{Phase m/2+5:} Cover the whole Pareto front.

The treatment of this phase is similar to Phase~4. Let $V:=\{x \in P_t \mid x^i \in A \text{ for every } i \in [m/2]\}$. We have $V \neq \emptyset$ and every $x \in V$ is Pareto-optimal (by Lemma~\ref{lem:RRMO-size-non-dom-set}(4)). 
Fix a search point $w \in V$ with $w \notin P_t$ and let $d_t:=\min_{x \in P_t} H(x,w)/2$. Note that $0<d_t \leq (2n/5m) \cdot m/2 = n/5$ (since $H(x,w)$ is even and $H(x,w)\leq 4n/(5m)$) and $d_t$ cannot increase. Define for $j \in [n/5]$ the random variable $X_j$ as the number of generations $t$ with $d_t=j$. Then this phase is finished in $\sum_{j=1}^{n/5} X_j$ generations. As in Subphase~A, with probability at least $(1-p_c)/(4en^2)$, the value $d_t$ decreases in one generation. Hence, $X:= \sum_{j=1}^{n/5} X_j$ is stochastically dominated by a sum $Z:= \sum_{j=1}^{n/5} Z_j$ of independently geometrically distributed random variables $Z_j$ with success probability $q:=q_j:=(1-p_c)/(4en^2)$. Hence, we can apply the remaining arguments from Subphase~A adapted to this situation to complete this phase in expected $O(n^3/(1-p_c))$ generations.

Now we upper bound the total runtime. Since Subphases~A,B and~C are passed at most $m/2$ times, the complete Pareto front is covered in expected 
$$O\left(\frac{nm}{1-p_c} + \frac{n^2 \log(n)}{1-p_c} + \frac{n^3}{1-p_c}+\frac{m}{2} \cdot \left(\frac{n^2}{1-p_c} + \frac{n^3}{m(1-p_c)} + \frac{m}{p_c}\right) + \frac{n^3}{1-p_c}\right) = O\left(\frac{n^3}{1-p_c}+ \frac{m^2}{p_c}\right)$$
generations and $\mu$ evaluations per generation provide $O(\mu n^3/(1-p_c) + \mu m^2/p_c)$ fitness evaluations, proving the theorem. 
\end{proof} 
If $p_c \in (0,1)$ is constant we see that in Theorem~\ref{thm:Runtime-Analysis-NSGA-III-mRRMO-onepoint}, the expected number of generations to optimize $m$-\RRRMO is $O(n^3)$ and hence, does not asymptotically depend on $m$ and even not on the population size $\mu$. If additionally $m$ is constant, this runtime is polynomial in terms of fitness evaluations if $\mu$ is also polynomial. \citet{Dang2024} showed for their bi-objective version of $m$-\RRRMO a runtime bound of $O(n^3/(1-p_c)+\mu/(p_cn))$ for \nsga in terms of generations which is by a factor of $\mu/n^4$ worse for population sizes $\mu \in \Omega(n^4)$ and constant $p_c \in (0,1)$. 

\subsection{Analysis of GSEMO}
In the following, we show that a result analogous to Theorem~\ref{thm:Runtime-Analysis-NSGA-III-mRRMO-onepoint} can also be established for the classical GSEMO with one-point crossover (see Algorithm~\ref{alg:gsemo}), where the population consists solely of mutually incomparable solutions, in terms of fitness evaluations. An important difference to \nsgaIII is that the population of GSEMO consists only of mutually incomparable solutions and hence, its size may vary over time. As a consequence, we are able to provide two slightly different runtime results for GSEMO on $m$-\RRRMO, depending on the crossover probability~$p_c$. Another difference is that GSEMO initializes only a single search point chosen uniformly at random, making it more likely that the initial solution has fitness zero. In such a case, we guarantee that a solution~$x$ with nonzero fitness, specifically $0 < \ones{x^j} \leq 6n/(5m)$ for every block $j \in [m/2]$, can be found by first applying the additive drift theorem to establish this condition in one block~$j$. Subsequently, the negative drift theorem ensures that the condition $0 < \ones{x^j} \leq 6n/(5m)$ is maintained over a sufficiently long period, providing enough time for it to be satisfied across all other blocks~$i$. For the latter condition to be satisfied we need $m \leq d\sqrt{n/\log n}$ where $d>0$ is a sufficiently small constant. Our result would also hold for any number~$m$ of objectives, provided that GSEMO is initialized with a search point with non-zero fitness. 

\begin{theorem}
\label{thm:Runtime-Analysis-GSEMO-mRRMO}
Let $m \in \mathbb{N}$ be divisible by $2$ and $n$ be divisible by $5m/2$. Suppose that $m\leq d\sqrt{n/\log(n)}$ for a sufficiently small constant $d$. Then the algorithm GSEMO (Algorithm~\ref{alg:gsemo}) with $p_c \in (0,1)$, finds the Pareto set of $f:=m\text{-\RRRMO}$ in expected
\[
O\left(\frac{n^3(4n/(5m))^{m-1}}{1-p_c}+\frac{mn}{p_c}\right)
\]
fitness evaluations, if $1/p_c = O(\text{poly}(n))$ and $1/(1-p_c) = O(\text{poly}(n))$, and in expected 
\[
O\left(\frac{n^3(4n/(5m))^{m-1}}{1-p_c}+\frac{n(4n/(5m))^{m/2}}{p_c}\right)
\]
fitness evaluations otherwise. 
\end{theorem}
\begin{proof}
Note that the runtime in terms of fitness evaluations corresponds to the runtime in terms of iterations. As in the previous proof, we use the method of typical runs and divide the optimization procedure into several phases. Note that GSEMO prevents phenotypic duplicates: For every fitness vector $v$, the population $P_t$ contains at most one individual $x$ with $f(x)=v$. Additionally, it preserves non-dominated solutions: For every non-dominated $x \in P_t$ there exists $x' \in P_{t+1}$ that weakly dominates $x$.

\textbf{Phase 1:} Create $x$ with $f(x) \neq 0$.

During this phase, the population only contains a single search point, and hence, crossover, if executed, has no effect (as it will recombine two identical bit strings). By a classical Chernoff bound the probability that the initial search point $x$ fulfills $\ones{x^j} > 6n/(5m)$ in block $j$ is $e^{-\Omega(n/m)}$ since the expected number of ones is $n/m$ in block $j$ after initialization. By a union bound over all blocks, we have that the probability that there is a block $j$ with $\ones{x^j}>6n/(5m)$ is at most $m/2 \cdot e^{-\Omega(n/m)} = e^{-\Omega(n/m)}$ due to $m = O(\sqrt{n/\log(n)})$. Suppose that this happens. Denote by $x$ the current solution. For a block $j \in [m/2]$, denote by $X_t^j$ the number of ones in $x^j$ in iteration $t$ and let $c:=2/m \geq 2\sqrt{\log(n)}/(\sqrt{n}d)$. Then with Lemma~\ref{lem:additive-drift-main} applied on $\lambda = 1/10$ we obtain that $X^j_t \leq 11n/(10m) = cn \cdot (1+1/10)/2$ in expected $O(n)$ iterations. Suppose that $X^j_t \leq 11n/(10m)$. With Lemma~\ref{lem:negative-drift-main} applied on $\lambda = 1/5$ we see that, as long as the fitness is zero, the probability is $O(1/n^3)$ to obtain $X_t^j > 6n/(5m) = cn(1+1/5)/2$ within the next $n^3$ iterations. 
Now we estimate the expected time until $X_t^j \leq 6n/(5m)$ for every $j \in [m/2]$. In iteration $t$ we call a block $j$ \emph{marked} if there is an iteration $t' < t$ such that $\ones{X_{t'+1}^j} \leq 11n/(5m)$. Then all blocks are marked in expected $O(m n)$ iterations. Denote the event that, after $O(m n)$ iterations, there is a block $j \in [m/2]$ which is not marked, or is marked and satisfies $X_t^j > 6n/(5m)$, as a \emph{failure}. If no failure occurs, then $X_t^j \leq 6n/(5m)$ for every $j \in [m/2]$ after $O(m n)$ iterations. A failure occurs with probability $1-\Omega(1)$: Firstly, by Markov's inequality, all blocks are marked with probability $1-\Omega(1)$ after $O(m n)$ iterations and secondly, by taking a union bound over all blocks, we see that the probability is $O(m/n^3)=o(1)$ that a marked block $j$ satisfies $\ones{x^j} > 6n/(5m)$ after $O(m n)$ iterations. 

If such a failure occurs, we repeat the arguments including both drift theorems above. Since the expected number of failures is $O(1)$, we see that, after expected $(1+O(1))O(mn) = O(mn)$ iterations, all blocks $j \in [m/2]$ satisfy $X_t^j \leq 6n/(5m)$. Then if $x \neq 0^n$ we see that $f(x) \neq 0$ and this phase is finished. So we pessimistically assume $x=0^n$ the first time when $\ones{x^j} \leq 6n/(5m)$ for each $j \in [m/2]$. But then to create $y$ with $f(y) \neq 0$ it suffices to flip exactly one bit which happens with probability $(1-1/n)^{n-1} \geq 1/e$. Hence, in expected $O(m n)$ iterations, one individual $x$ with $f(x) \neq 0$ is created if the initialized $z$ satisfies $\ones{z^j}>6n/(5m)$ for a block $j \in [m/2]$. This also includes a possible repetition of all the above arguments if for the mutant $y$ there is $j \in [m/2]$ with $\ones{y^j} > 6n/(5m)$ happening with probability $1-\Omega(1)$. Hence, in expected $(1-e^{-\Omega(n/m)})+e^{-\Omega(n/m)} \cdot O(m n) = 1+o(1)$ iterations, Phase~1 is finished.

\textbf{Phase 2:} Create $x$ with $\ones{x^j}=6n/(5m)$ for all $j \in [m/2]$.

As in Theorem~\ref{thm:Runtime-Analysis-NSGA-III-mRRMO-onepoint}, let $O_t:=\max\{\ones{x} \mid x \in P_t\}$. Then $1 \leq O_t < 3n/5$ and the maximum number of mutually incomparable solutions is at most $S_2:=(6n/(5m)+1)^{m/2}$ during this phase (by Lemma~\ref{lem:RRMO-size-non-dom-set}(2)). To increase $O_t$ it suffices to choose a parent $z \in P_t$ with $\ones{z}=O_t$ (prob. at least $1/|P_t| \geq 1/S_2$), omit crossover and flip one of $2n/m-\ones{x^j} \geq 4n/(5m)$ zero bits in a block $j$ with $\ones{x^j} < 6n/(5m)$ and not changing the remaining bits (prob. at least $4/(5me)$). Hence, in one iteration, the probability is at least $4(1-p_c)/(5meS_2)$ to increase $O_t$ and hence, in total, the number of required itertions to finish this phase is at least
\[
\frac{(3n/5-1) \cdot 5me \cdot S_2}{4 (1-p_c)} = O\left(\frac{nmS_2}{1-p_c}\right).
\]

\textbf{Phase~3:} Create $x$ with $x^j \in B$ for each $j \in [m/2]$.\\
Let $\alpha_t \in \{0, \ldots ,2n/5-1\}$ as in the proof of Theorem~\ref{thm:Runtime-Analysis-NSGA-III-mRRMO-onepoint}. The maximum number of mutually incomparable solutions is at most $S_3:=(4n/(5m)+1)^{m-1}$ (by Lemma~\ref{lem:RRMO-size-non-dom-set}) and a solution with corresponding value $\alpha_t$ is non-dominated. This implies that $\alpha_t$ cannot decrease. Let $\sigma_t=2n/5-\alpha_t$. Similar as in Phase~3 in the proof of Theorem~\ref{thm:Runtime-Analysis-NSGA-III-mRRMO-onepoint}, $\sigma_t$ can be decreased with probability at least
$$\frac{(1-p_c)\sigma_t}{en^2|P_t|} \geq \frac{(1-p_c)\sigma_t}{e n^2 S_3}$$
since an individual with corresponding value $\sigma_t$ is chosen with probability at least $1/|P_t| \geq 1/S_3$ in a mutation only iteration. Hence, the expected number of iterations to obtain $\sigma_t=0$ is at most  
\[
\sum_{j=1}^{2n/5} \frac{en^2S_3}{(1-p_c) j} = O\left(\frac{n^2 \log(n) S_3}{1-p_c}\right).
\]

\textbf{Phase $\gamma_t$+4:} Create an individual $x$ with $|K(x)|=\gamma_t+1$.

Note that Phase~4 starts when $\gamma_t=0$. If Phase $m/2-1+4=m/2+3$ is finished, a Pareto optimal search point is found since $\gamma_t+1=m/2$. Let $z$ be the first point generated in an iteration $t$ with $|K(z)|=\gamma_t$. If $\gamma_t=0$ then iteration $t$ coincides with the time when Phase~3 is finished. In this case, $|P_t| \leq (4n/(5m)+1)^{m-1}$ by Lemma~\ref{lem:RRMO-properties}. If $\gamma_t>0$ then $P_t$ consists only of this individual $z$ at time $t$ since by Lemma~\ref{lem:RRMO-properties}(3) this new created $z$ dominates all other individuals previously in $P_t$. 
Denote the event that 
an individual $y \in N$ with $K(y) \neq K(z)$ is created as a \emph{failure}. Note that a failure may only occur if $0 < |K(z)| < m/2$. As long as no failure occurs, we have that $P_t \subset \{x \in N \mid K(x)=K(z)\}$ and therefore $|P_t| = (2n/(5m)+1)^{\gamma_t}(4n/(5m)+1)^{m/2-\gamma_t}$ (since $\gamma_t$ distinct blocks $j$ satisfy $x^j \in A$, the remaining blocks $i$ satisfy $x^i \in B$, and we see that $|A| = 2n/(5m)+1$ and $|B|=4n/(5m)+1$). In the following claim we estimate the probability $p_{\text{fail}}$ that such a failure occurs. 
\begin{claim}
    The probability $p_{\text{fail}}$ that such a failure occurs is at most $n^{-\Omega(\sqrt{n})}$.
\end{claim}

\begin{proofofclaim}
Mutating $x^j \in A$ into $y^j \in B$ happens with probability at most 
$$(4n/(5m)+1) \cdot n^{-2n/(5m)} \leq n \cdot n^{-2n/(5m)} = n^{-2n/(5m)+1}$$
since $H(x^j,y^j) \geq 2n/(5m)$ and $|B|=4n/(5m)+1$. Mutating $x^j \in B$ into $y^j \in A$ happens with probability at most $(2n/(5m)+1) \cdot n^{-2n/(5m)} \leq n^{-2n/(5m)+1}$ since $|A| = 2n/(5m)+1$ and $H(x^j,y^j) \geq 2n/(5m)$. For a failure to happen, at least one of the two events above must occur, regardless of whether crossover is executed. This is because applying crossover on $z_1,z_2$ with $K(z_1)=K(z_2)$ creates an outcome $y$ with $K(y) \subset K(z_1)$ or $K(z_1) \subset K(y)$: If the cutting point is in block $j$, then $K(z_1) \setminus \{j\} = K(y) \setminus \{j\}$ since all blocks to the left of~$j$ in $z_1$ and all blocks to the right of~$j$ in $z_2$ are passed to $y$ and we have $K(z_1) = K(z_2)$. If $j \in K(z_1)$ then $K(y) \subset K(z_1)$ and if $j \notin K(z_1)$ then $K(z_1) \subset K(y)$. By a union bound over $m/2$ blocks, the probability that a failure occurs is at most $m/2 \cdot n^{-2n/(5m)+1} \leq n^{-2n/(5m)+2} = n^{-\Omega(\sqrt{n})}$ since $m=O(\sqrt{n/\log(n)})$, proving the claim.
\end{proofofclaim}

If no failure occurs we have for $S_4:=(4n/(5m)+1)^{m-1}$ (refering to Phase~4) and $S_{j+4}=(2n/(5m)+1)^j(4n/(5m)+1)^{m/2-j}$ for $\gamma_t=j>0$ (refering to Phase~$\gamma_t+4$ for $\gamma_t>0$) that $|P_t| \leq S_{\gamma_t+4}$.

\textbf{Subphase A}: Let $I:=K(z)$. Find $S_I:=\{x \mid x^j \in B \text{ for all } j \in [m/2]\setminus I \text{ and } x^j \in A \text{ for all } j \in I\}$ or a failure occurs.

We may assume that no failure occurs, as the occurrence of a failure only accelerates the completion of the phase. For a specific search point $w \in S_I$ not already found we first upper bound the probability by $e^{-\Omega(n)}$ that $w$ has not been created after $7en^3/(5S_{\gamma_t+4}(1-p_c))$ iterations. Let $D_t:=P_t \cap S_I$ and $d_t:= \min_{x \in D_t} H(x,w)/2$. Note that $D_t \neq \emptyset$ 
and $0 < d_t \leq 2n\gamma_t/(5m) + 4n(m/2-\gamma_t)/(5m) \leq 2n/5$ and since a solution $x \in D_t$ is non-dominated, $d_t$ cannot increase (by Lemma~\ref{lem:Reference-Points}). We created $w$ if $d_t=0$. For $1 \leq \beta \leq 2n/5$, define the random variable $X_\beta$ as the number of iterations $t$ with $d_t=\beta$. Then the total number of iterations required to find $w$ is at most $X=\sum_{\beta=1}^{2n/5} X_\beta$. Fix $y \in D_t$ with $H(w,y)/2=d_t \neq 0$. 
To decrease $d_t$, it suffices to choose $y$ as a parent (prob. at least $1/S_{4+\gamma_t}$), omit crossover (prob. $1-p_c$) and flip two specific bits during mutation (prob. at least $1/(en^2)$) in order to shift a block of ones in $y$ in that direction of the corresponding block of $w$. 
Hence, the probability to decrease $d_t$ in one iteration is at least $(1-p_c)/(en^2S_{4+\gamma_t})$. Thus, the random variable $X$ is stochastically dominated by an independent sum $Z := \sum_{\beta=1}^{2n/5} Z_\beta$ of geometrically distributed random variables $Z_\beta$ with success probability $q:=q_\beta:=(1-p_c)/(en^2S_{\gamma_t+4})$. Note that $\expect{Z}=2en^3S_{\gamma_t+4}/(5(1-p_c))$. Now we use Theorem~\ref{thm:Doerr-dominance}(1): For $d:=\sum_{\beta=1}^{2n/5} 1/q_\beta^2 = 2e^2n^5S_{\gamma_t+4}^2/(5(1-p_c)^2)$ and $\lambda \geq 0$ we obtain
\[
\Pr(Z \geq \expect{Z} + \lambda) \leq \exp\left(-\frac{1}{4} \min\left\{\frac{\lambda^2}{d}, \lambda q \right\}\right).
\]
For $\lambda = en^3S_{\gamma_t+4}/(1-p_c)$ we obtain
$$\Pr\left(X \geq \frac{7en^3S_{\gamma_t+4}}{5(1-p_c)}\right) \leq \Pr\left(Z \geq \frac{7en^3S_{\gamma_t+4}}{5(1-p_c)}\right) \leq \exp\left(-\frac{1}{4} \min\left\{\frac{5n}{2}, n\right\}\right) = \exp\left(-\frac{n}{4}\right).$$
As in Phase~4 of the proof of Theorem~\ref{thm:Runtime-Analysis-NSGA-III-mRRMO-onepoint}, we see by a union bound over all possible~$w$, that the probability that $S_I$ is completely found after $7en^3S_{\gamma_t+4}/(5(1-p_c))$ iterations is at most $e^{-\Omega(n)}=o(1)$. If this does not happen, we can repeat the argument. Thus the expected number of iterations to finish this phase is at most $(1+o(1))(7en^3S_{\gamma_t+4}/(5(1-p_c))) = O(n^3S_{\gamma_t+4}/(1-p_c))$.

\textbf{Subphase B}: Create an individual $x$ with $|K(x)|=\gamma_t+1$ or a failure occurs.

As in the subphase before, we may assume that no failure occurs. Let $i \in [m/2] \setminus K(z)$. To create such an individual $x$ in one iteration one has to choose two $y_1,y_2 \in P_t$ as parents such that $y_1^i=0^a1^{6n/(5m)}0^{4n/(5m)-a}$ and $y_2^i=0^{2n/(5m)+a}1^{6n/(5m)}0^{2n/(5m)-a}$ for an $a \in \{0, \ldots , 2n/(5m)\}$, perform one-point crossover with cutting point in $\{(i-1)(2n)/m+2n/(5m) + a, \ldots , (i-1)(2n)/m+6n/(5m)+a\}$ (i.e. in block $i$ at position $b \in \{2n/(5m) + a, \ldots , 6n/(5m)+a\}$ to create $y^i=0^a1^{8n/(5m)}0^{2n/(5m)-a}$, and then omitting mutation. Note that $K(y_1) \cup \{i\} \subset K(y)$ (since either $y^j=y_1^j$ or $y^j=y_2^j$ for every $j \in K(y_1)$) and hence, $\vert{K(y)}\vert \geq \gamma_t+1$. We estimate the probability that this sequence of events occurs. Since no failure occurs, all individuals $x \in P_t$ have the same value $K(x)$. Since $S_I$ is completely found, we even have $P_t=S_I$, and hence, for at least half of the individuals $y_1 \in P_t$ we find $a \in \{0, \ldots , 2n/(5m)\}$ such that $y_1^i=0^a1^{6n/(5m)}0^{4n/(5m)-a}$. Further, for a fixed $a$, there are $(4n/(5m)+1)^{m/2-\gamma_t-1}(2n/(5m)+1)^{\gamma_t}$ individuals $y_2$ in $P_t$ with $y_2^i = 0^{2n/(5m)+a}1^{6n/(5m)}0^{2n/(5m)-a}$. Hence, the probability is at least
$$\frac{(4n/(5m)+1)^{m/2-\gamma_t-1}(2n/(5m)+1)^{\gamma_t} }{2 \cdot (4n/(5m)+1)^{m/2-\gamma_t} (2n/(5m)+1)^{\gamma_t}} = \frac{1}{2 \cdot (4n/(5m)+1)}$$    
to choose suitable parents $y_1,y_2$. Further, the probability is at least 
$$\frac{4n/(5m)+1}{n+1} \cdot \left(1-\frac{1}{n}\right)^n  \geq \frac{4n/(5m)+1}{4(n+1)}$$
that a correct cutting point for crossover is found and after crossover, no bit is flipped during mutation. Hence, the probability in total to create an individual $y$ with $|K(y)|=\gamma_t+1$ is at least 
$$\frac{p_c}{2 \cdot (4n/(5m)+1)} \cdot \frac{4n/(5m)+1}{4(n+1)} = \frac{p_c}{8 \cdot (n+1)}$$
and consequently, this subphase is finished in expected $O(n/p_c)$ iterations. 

In total, Subphase A and B are completed in expected 
$$O\left(\frac{n^3S_{\gamma_t+4}}{1-p_c}+\frac{n}{p_c}\right)$$
iterations and the probability that a failure occurs during this time is $O((n^3S_{\gamma_t+4}/(1-p_c)+n/p_c)n^{-\Omega(\sqrt{n})})$ and a failure may only occur if $\gamma_t>0$, but if this happens, we may choose individuals $(y_1,y_2)$ with $K(y_1) \neq K(y_2)$ where $i:=\min(K(y_1) \setminus K(y_2)) < \min(K(y_2) \setminus K(y_1))$, perform one-point crossover with cutting point $2ni/m$ (prob. $1/(n+1)$) and do not flip any bit during mutation (prob. at least $1/4$). Then the outcome $y$ satisfies $K(y)>K(y_1)$ since all blocks $j$  with $j \leq i$ from $y_1$ are passed to $y$ (which all satisfy $y_1^j \in A$) and the remaining ones are taken from $y_2$. Note that there are $\gamma_t-(i-1) = \gamma_t-i+1$ many blocks $j$ in $y_2$ right to $i$ with $y_2^j \in A$ (compare also with Figure~\ref{fig:Cross2}). 

\begin{figure}[t]
    \centering
    \includegraphics[scale=1]{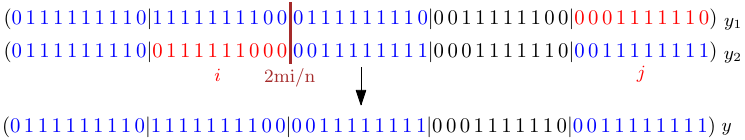}
    \caption{How crossover can increase $\gamma_t$ in case of a failure for $n=50$ and $m=10$. The crossover point is marked in brown, with blocks $i$ and $j$ highlighted in red. Strings $y_1^k ,y_2^k, y^k \in A$ are shown in blue.}
    \label{fig:Cross2} 
\end{figure}

In the following we estimate the probability to select such a pair. Since, for a given $x \in P_t$, there are at least $\max\{1,|P_t|-S_{\gamma_t+4}\}$ many $y$ with $K(y) \neq K(x)$ (since the maximum number of different individuals with the same $K(x)$ value is $S_{\gamma_t+4}$), and either $\min(K(x) \setminus K(y)) < \min(K(y) \setminus K(x))$ or $\min(K(y) \setminus K(x)) < \min(K(x) \setminus K(y))$, either $(x,y)$ or $(y,x)$ can be chosen for the above crossover. Hence, such a pair is selected with probability at least
$$\frac{|P_t| \cdot \max\{1,|P_t|-S_{\gamma_t+4}\}}{2|P_t|^2}  = \frac{\max\{1,|P_t|-S_{\gamma_t+4}\}}{2|P_t|} \geq \frac{1}{2S_{\gamma_t+4}+2}.$$
The last inequality obviously holds if $|P_t| \leq S_{\gamma_t+4}+1$. If $|P_t| > S_{\gamma_t+4}+1$ then $(|P_t| - S_{\gamma_t+4})/|P_t| \geq 1/(S_{\gamma_t+4}+1)$. 
Hence, in case of a failure, a desired $y$ is created with probability at least
\begin{align*}
\label{eq:probability}
\frac{p_c}{4(n+1)(2S_{\gamma_t+4}+2)}.
\end{align*}

Now we compute the expected number of iterations in total to finish Phase~$\gamma_t$+4. To this end, we consider two cases based on the crossover probability $p_c$. 

\textbf{Case 1:} $1/p_c \neq O(\text{poly}(n))$ or $1/(1-p_c) \neq O(\text{poly}(n))$. We can estimate the expected number of iterations as the sum of the expected iterations required to pass through Subphase~A and~B, and expected iterations required to create $y$ by recombining two suitable parents $y_1,y_2$ with $K(y_1) \neq K(y_2)$. These are at most
\[
O\left(\frac{n^3S_{\gamma_t+4}}{1-p_c}+\frac{n}{p_c} + \frac{4(n+1)(S_{\gamma_t+4}+1)}{p_c}\right) = O\left(\frac{n^3S_{\gamma_t+4}}{1-p_c}+\frac{nS_{\gamma_t+4}}{p_c}\right)
\]
iterations in expectation if $\gamma_t>0$. If $\gamma_t=0$ no failure can occur and we obtain $O(\frac{n^3S_{\gamma_t+4}}{1-p_c}+\frac{n}{p_c})$ iterations in expectation by the calculation above.

\textbf{Case 2:} $1/p_c = O(\text{poly}(n))$ and $1/(1-p_c) = O(\text{poly}(n))$. We have that 
$$O\left(\frac{n^3S_{\gamma_t+4}}{1-p_c}+\frac{n}{p_c}\right) = O(S_{\gamma_t+4} \cdot \text{poly}(n))$$
and hence, by a union bound, the probability that no failure occurs within $O(n^3S_{\gamma_t+4}/(1-p_c)+n/p_c)$ iterations is still at most 
$$O(S_{\gamma_t+4} \cdot \text{poly}(n)) \cdot n^{-\Omega(\sqrt{n})} = n^{-\Omega(\sqrt{n})}$$
since $\log(S_{\gamma_t+4}) \leq (m/2+1) \log(4n/(5m)+1) \leq d \sqrt{n \log(n)}$ where $d$ is a constant (since $m \leq d\sqrt{n/\log(n)}$). Hence, this phase is finished in
\begin{align*}
& O\left(\frac{n^3S_{\gamma_t+4}}{1-p_c}+\frac{n}{p_c} + \frac{4(n+1)(S_{\gamma_t+4}+1)}{p_c} \cdot n^{-\Omega(n)}\right)= O\left(\frac{n^3S_{\gamma_t+4}}{1-p_c}+\frac{n}{p_c}\right)
\end{align*}
iterations in expectation.

\textbf{Phase m/2+4:} Cover the whole Pareto front.\\
The treatment of this phase is similar to Phase~4. Let $V:=\{x \in P_t \mid x^j \in A \text{ for every } j \in [m/2]\}$. We have $V \neq \emptyset$ and every $x \in V$ is Pareto-optimal. Further, a set $S$ of mutually incomparable solutions satisfies $|S| \leq |V| = (2n/(5m)+1)^{m/2}=:S_{m/2+4}$. Fix a search point $w \in V$ with $w \notin P_t$ and let $d_t:=\min_{x \in P_t} H(x,w)/2$. Note that $0<d_t \leq n/5$ and $d_t$ cannot increase. Define for $j \in [n/5]$ the random variable $X_j$ as the number of iterations $t$ with $d_t=j$. Then this phase is finished in $\sum_{j=1}^{n/5} X_j$ iterations. As in Subphase A of Phase $\gamma_t+4$, with probability at least $(1-p_c)/(en^2S_{m/2+4})$, the value $d_t$ decreases in one iteration. Hence, $X:= \sum_{j=1}^{n/5} X_j$ is stochastically dominated by a sum $Z:= \sum_{j=1}^{n/5} Z_j$ of independently geometrically distributed random variables $Z_j$ with success probability $q:=q_j:=(1-p_c)/(en^2S_{m/2+4})$. Hence, we can apply the remaining arguments from Subphase A in Phase~$\gamma_t$+4 adapted to this situation to complete this phase in expected $O(n^3S_{m/2+4}/(1-p_c))$ iterations.

Finally, we upper bound the total runtime. Recall that $S_4=(4n/(5m)+1)^{m-1}$ and $S_{4+\ell} = (2n/(5m)+1)^\ell(4n/(5m)+1)^{m/2-\ell}$ for $\ell>0$. We obtain
\begin{align*}
\sum_{\ell=0}^{m/2} S_{4+\ell} &= S_4 + \sum_{\ell=1}^{m/2} S_{4+\ell} = \left(\frac{4n}{5m}+1\right)^{m-1}+ \left(\frac{4n}{5m}+1\right)^{m/2-1} \left(\frac{2n}{5m}+1\right)\sum_{\ell=1}^{m/2} \left(\frac{\frac{2n}{5m}+1}{\frac{4n}{5m}+1}\right)^{\ell-1} \\
&\leq \left(\frac{4n}{5m}+1\right)^{m-1}+ \left(\frac{4n}{5m}+1\right)^{m/2-1} \left(\frac{2n}{5m}+1\right)\sum_{\ell=0}^{\infty} \left(\frac{2}{3}\right)^\ell \leq S_4+\frac{S_5}{1-2/3} = S_4 + 3 S_5 \leq 4S_4,
\end{align*}
since the function $]0,\infty[ \to \mathbb{R}, x \mapsto (x+1)/(2x+1) = 1-x/(2x+1) = 1-1/(2+1/x),$ is strictly monotone decreasing and $2n/(5m) \geq 1$. Further we used the geometric series $\sum_{\ell=0}^{\infty} p^\ell = 1/(1-p)$ for $0<p<1$. From the above we also see that $\sum_{\ell=0}^{m/2-1} S_{4+\ell} \leq S_4 + 3S_5$ and that $\sum_{\ell=1}^{m/2-1} S_{4+\ell} \leq 3S_5$. If $1/p_c \neq O(\text{poly}(n))$ or $1/(1-p_c) \neq O(\text{poly}(n))$, the complete Pareto front is covered in expected 
\begin{align*}
 O&\left(\frac{nmS_2}{1-p_c} + \frac{n^2 \log(n) S_3}{1-p_c} + \sum_{\ell=0}^{m/2-1} \left(\frac{n^3S_{\ell+4}}{1-p_c}\right) + \frac{n}{p_c} + \sum_{\ell=1}^{m/2-1} \left(\frac{nS_{\ell+4}}{p_c}\right) + \frac{n^3S_{m/2+4}}{1-p_c}\right)\\
& = O\left(\frac{nmS_2}{1-p_c} + \frac{n^2 \log(n) S_3}{1-p_c} + \sum_{\ell=0}^{m/2} \left(\frac{n^3S_{\ell+4}}{1-p_c}\right) + \sum_{\ell=1}^{m/2-1} \frac{nS_{\ell+4}}{p_c}\right)= O\left(\frac{n^3S_4}{1-p_c}+ \frac{nS_5}{p_c}\right)\\
&= O\left(\frac{n^3(4n/(5m)+1)^{m-1}}{1-p_c} + \frac{n(4n/(5m)+1)^{m/2})}{p_c}\right) = O\left(\frac{n^3(4n/(5m))^{m-1}}{1-p_c} + \frac{n(4n/(5m))^{m/2})}{p_c}\right)
\end{align*}
iterations. If $1/p_c = O(\text{poly}(n))$ and $1/(1-p_c) = O(\text{poly}(n))$, the complete Pareto front is covered in expected
\begin{align*}
O&\left(\frac{nmS_2}{1-p_c} + \frac{n^2 \log(n) S_3}{1-p_c} + \sum_{\ell=0}^{m/2-1} \left(\frac{n^3S_{\ell+4}}{1-p_c} + \frac{n}{p_c}\right) + \frac{n^3S_{m/2+4}}{1-p_c}\right) = O\left(\frac{n^3S_4}{1-p_c}+ \frac{nm}{p_c}\right) \\
& = O\left(\frac{n^3(4n/(5m)+1)^{m-1}}{1-p_c} + \frac{nm}{p_c}\right) = O\left(\frac{n^3(4n/(5m))^{m-1}}{1-p_c} + \frac{mn}{p_c}\right) 
\end{align*}
iterations. This finishes the proof.
\end{proof}

Note that GSEMO also requires a polynomial number of fitness evaluations in expectation to compute the full Pareto front when both $p_c$ and $m$ are constant. Notice that the derived upper bound for GSEMO can be smaller than that of \nsgaIII on \RRRMO, even for a population size of $\mu = \left( 4n/(5m) + 1 \right)^{m-1}$ used in \nsgaIII. Specifically, if $m > 2$ is constant and $p_c = O(m/(n^2( 4n/(5m)+1)^{m-1})) = O(1/n^{m+1})$, while $1/p_c = O(\text{poly}(n))$, then the expected number of fitness evaluations for \nsgaIII becomes $O(\mu / p_c)$, whereas for GSEMO it is $O(n / p_c)$. Thus, GSEMO achieves an upper bound which is by a factor of $\mu / n = \Omega(n^{m-2})$ smaller than that of \nsgaIII. This demonstrates the potential advantage of using a dynamic population size, particularly in optimization processes involving multiple distinct phases. We are not aware of any such result in the literature. 

\section{Difficulty of GSEMO and \nsgaIII on \RRRMO Without Crossover}
Finally, we point out that GSEMO and \nsgaIII without crossover on $m$-\RRRMO for a wide range of the number $m$ of objectives becomes extremely slow. This even holds for finding the first Pareto-optimal search point. The results in this section extend the corresponding results from~\cite{Opris2025} to a superconstant parameter regime of the number of objectives $m$.

\begin{theorem}\label{thm:NSGA-III-pc-zero}
Suppose that $m = o(n/\log(n))$, $m$ is divisible by $2$ and $n$ is divisible by $5m/2$. \nsgaIII (Algorithm~\ref{alg:nsga-iii}) on $m$-\RRRMO with $p_c=0$, any choice of $\refer$, and $\mu = 2^{o(n/m)}$ needs at least $n^{\Omega(n/m)}$ generations in expectation to create any Pareto-optimal search point of $m$-\RRRMO.
\end{theorem}

\begin{proof}
We see that an individual $x$ with $0<\ones{x^j}\leq 6n/(5m)$ for all $j \in [m/2]$ initializes with probability $1-m/2 \cdot 2^{-\Omega(n/m)} = 2^{-\Omega(n/m)}$ since $m=o(n/\log(n))$. Hence, by a union bound, with probability $1-\mu 2^{-\Omega(n/m)} = 1-o(1)$ every individual $x$ initializes with $0<\ones{x^i}\leq 6n/(5m)$ for every $i \in [m/2]$ since $\mu \in 2^{o(n/m)}$. Suppose that this happens. Then all $x \in P_t$ satisfy $f(x) \neq 0$ and hence, the algorithm will always reject search points with fitness zero. Therefore, to create a specific search point $y$ with $y^1 \in A$, at least $2n/(5m)$ specific zero bits in the first block must be flipped simultaneously, which occurs with probability $1/n^{2n/(5m)}$. Since $|A| = 2n/(5m) + 1$, the probability of generating any $y$ with $y^1 \in A$ is at most $(2n/(5m) + 1) \cdot n^{-2n/(5m)} \leq n^{-2n/m + 1} = n^{-\Omega(n/m)}$. However, any Pareto-optimal search point $z$ satisfies $z^1 \in A$ and therefore the event above must occur to create such a $z$. So the expected number of needed generations in total is at least $(1 - o(1))(n^{\Omega(n/m)}/\mu)=n^{\Omega(n/m)}$. 
\end{proof}

We are also able to obtain a similar result for GSEMO.

\begin{theorem}
\label{thm:GSEMO-pc-zero}
Suppose that $m = o(n/\log(n))$, $m$ is divisible by $2$ and $n$ is divisible by $5m/2$. GSEMO (Algorithm~\ref{alg:gsemo}) on $m$-\RRRMO with $p_c=0$ needs at least $n^{\Omega(n/m)}$ fitness evaluations in expectation to create any Pareto-optimal search point of $m$-\RRRMO.
\end{theorem}

\begin{proof}
As in the proof of the previous theorem, we see that an individual $x$ with $f(x)\neq 0$ is initialized with probability $1-m/2 \cdot 2^{-\Omega(n/m)} = 1-o(1)$. Suppose that this happens. Since the algorithm always rejects search points with fitness zero, we see, as in the previous proof, that any search point $y$ with $y^1 \in A$ is created with probability at most $n^{-\Omega(n/m)}$. So the expected number of iterations needed in total is at least $(1 - o(1))(n^{\Omega(n/m)}/\mu)=n^{\Omega(n/m)}$. 
\end{proof}

In case of both algorithms above, if $m \leq d \sqrt{n}$ for a sufficiently small constant $d$, $(4n/(5m)+1)^{m-1} \leq \mu \leq 2^{o(n/m)}$ in case of \nsgaIII and constant crossover probability $p_c$, crossover guarantees a speedup of order $O(n^{\Omega(\sqrt{n})})$ in the runtime in terms of generations which is exponential in $\sqrt{n}$. If the number of objectives is also constant, this speedup is of order $O(n^{\Omega(n)})$ which is exponential in $n$.

\section{A Many-Objective Royal-Road Function For Uniform Crossover}

We now define a many-objective $\RRRfull$ function for uniform crossover, denoted as $m$-\RRRMOuni. In contrast to \RRRMO, the bi-objective \RRRMOuni function from~\citep{Dang2024} attains a maximum fitness value of $O(n)$ in each objective. Consequently, Lemma~\ref{lem:Reference-Points} can be applied with a reasonable number of reference points which is in the same order as for our \RRRMO function. However, for uniform crossover applied to two search points $x$ and $y$ with Hamming distance $H(x,y)$, there are $2^{H(x,y)}$ equally likely offspring, growing exponentially as $H(x,y)$ increases. This fact complicates the construction of $m$-\RRRMOuni, particularly when the number of objectives $m$ increases. As a result, our proposed function is not a straightforward generalization of the corresponding version from~\citep{Dang2024} to the many-objective setting. However, to keep things more similar, we divide each block into two halves, each half further subdivided into four subblocks of equal size.

\begin{definition}
\label{def:uRRMO-sets-and-function}
Let $m \in \mathbb{N}$ be an even number and $n$ be divisible by $8m$. Then we divide the bit string $x \in \{0,1\}^n$ into $m/2$ blocks of equal length such that $x=(x^1, \ldots , x^{m/2})$ and for each $j \in [m/2]$ we also write $x^j = (x_\ell^j,x_r^j) = (x_{1,\ell}^j, \ldots , x_{4,\ell}^j,x_{1,r}^j, \ldots , x_{4,r}^j)$ where the first half $x_\ell^j$ and the second half $x_\ell^r$ of the bit string $x^j$ are divided into four strings $x_{1,\ell}^j, \ldots , x_{4,\ell}^j$ and $x_{1,r}^j, \ldots , x_{4,r}^j$ of equal length $n/(4m)$, respectively. Let $\vec{1}:=(1,1)$. Regarding the left half $x_\ell^j$ of $x^j$ we define 
\begin{align*}
U&:=\{(x_{1,\ell}, \ldots , x_{4,\ell}) \in \{0,1\}^{n/m} \mid \forall i \in \{1, \ldots , 4\}: (x_{i,\ell} \in \{0,1\}^{n/(4m)}) \wedge (n/(12m) \leq \ones{x_{i,\ell}} \leq n/(6m))\}.\\
P&:=\{x_\ell \in \{0,1\}^{n/m} \text{ } \Big| \text{ } \LO(x_\ell)+\TZ(x_\ell)=n/m \}.
\end{align*}
Regarding the right half $x_r^j$ of $x^j$ we define
\begin{align*}
C&:=\{x_r \in \{0,1\}^{n/m} \text{ } \Big| \text{ } (\LO(x_r)+\TZ(x_r)=n/m) \vee (\LZ(x_r)+\TOs(x_r)=n/m)\},\\
T&:=\{(x_{1,r}, \ldots , x_{4,r}) \in \{0,1\}^{n/m} \mid \forall i \in \{1, \ldots , 4\}: x_{i,r} \in \{0,1\}^{n/(4m)} \wedge \ones{x_{i,r}} = \zeros{x_{i,r}}\}.
\end{align*}
Further, for $x \in \{0,1\}^n$ we define the set $\mathcal{U}(x):=\{j \in [m/2] \mid x_\ell^j = 1^{n/m} \wedge x_r^j \in C\}$ and with respect to the whole bit string define the sets
\begin{align*}
L&:=\{x \in \{0,1\}^n \mid \forall j \in [m/2]: x_\ell^j \in U \wedge \exists j \in [m/2]: x_r^j \notin C\},\\
M&:=\{x \in \{0,1\}^n \mid \forall j \in [m/2]: \ones{x_\ell^j} \leq 4n/(6m) \wedge x_r^j \in C, \exists j \in [m/2]: x_\ell^j \neq 0^{n/m}\},\\
N&:=\{x \in \{0,1\}^n \mid \text{there is at most one $i \in [m/2]$}: (x_\ell^i \in P \setminus \{1^{n/m}\} \wedge x_r^i \in T) \vee (x_\ell^i = 1^{n/m} \wedge x_r^i \notin C), \\
\text{ } &\forall j \in [m/2] \setminus (\mathcal{U}(x) \cup \{i\}): x_\ell^j = 0^{n/m} \wedge x_r^j \in C\}.
\end{align*}
For $k \in [m]$ odd we define the tuple of functions
$$(g_k(x),g_{k+1}(x)):= \begin{cases}
	(\LO(x_r^{(k+1)/2}),\frac{n}{m} + \TZ(x_r^{(k+1)/2})) \text{ if $\LO(x_r^{(k+1)/2}) \neq 0$,} \\
	(\frac{n}{m} + \LZ(x_r^{(k+1)/2}), \TOs(x_r^{(k+1)/2})) \text{ otherwise.}
\end{cases}$$

Then the function class $m$-$\text{\RRRMOuni}: \{0,1\}^n \to \mathbb{N}_0^m$ is defined as 
$$\text{$m$-}\text{\RRRMOuni} = (f_1(x), \ldots , f_m(x))$$
where for $k \in [m]$ odd $(f_k(x),f_{k+1}(x)):=$
$$\begin{cases}
	(g_k(x),g_{k+1}(x)) &\text{ if $x \in L$,} \\
	(g_k(x),g_{k+1}(x)) + (\frac{3n}{m}-\ones{x_\ell^{(k+1)/2}}) \cdot \vec{1} &\text{ if $x \in M$,} \\
        (g_k(x),g_{k+1}(x)) + (\frac{7n}{m} \cdot |\mathcal{U}(x)|+\frac{5n}{m}) \cdot \vec{1} &\text{ if $x \in N$ and} \\
        &\text{ $x_{\ell}^{(k+1)/2} = 0^{n/m} \wedge x_r^{(k+1)/2} \in C$,}\\
        (g_k(x),g_{k+1}(x)) + (\frac{7n}{m} \cdot |\mathcal{U}(x)|+\frac{5n}{m} + (\frac{3n}{m}-\zeros{x_\ell^{(k+1)/2}})) \cdot \vec{1} &\text{ if $x \in N$ and }\\
        &\text{ $x_\ell^{(k+1)/2} \in P \setminus \{1^{n/m}\} \wedge x_r^{(k+1)/2} \in T$,}\\
        (g_k(x),g_{k+1}(x)) + (\frac{7n}{m} \cdot |\mathcal{U}(x)| + \frac{10n}{m}) \cdot \vec{1} &\text{ if $x \in N$ and $x_\ell^{(k+1)/2} = 1^{n/m}$,}\\
        (0,0) &\text{ otherwise.}
\end{cases}$$
\end{definition}

In the $m$-\RRRMOuni function, the bit string is partitioned into $m/2$ blocks $x^1, \ldots, x^{m/2}$, each of length $2n/m$ and $x^j$ is further divided into into two sub-blocks $x_\ell^j$ and $x_r^j$ of equal length. For $m = o(\sqrt{n})$, EMOAs typically initialize their search points such that the number of ones in each sub-block $x_{i,\ell}^j$ satisfies $n/(12m) < \ones{x_{i,\ell}^j} < n/(6m)$, particularly, $x_\ell^j \in U$ for all $j \in [m/2]$. Then a fitness signal is provided that encourages solutions $x$ with $x^j \in C$ for all $j \in [m/2]$. Specifically, any increase in the number of leading ones, $\LO(x_r^j)$, or leading zeros, $\LZ(x_r^j)$ (when $\LO(x_r^j) \neq 0$), improves fitness. Once each $x_r^j$ reaches $C$, the fitness signal changes, favoring the elimination of ones in $x_\ell^j$, guiding all left sub-blocks towards the all-zero string $0^{n/m}$, while maintaining $x_r^j \in C$. Further, a strong fitness signal is assigned \emph{equally} to each objective based on $|\mathcal{U}(x)|$, the number of blocks $j$ for which $x_\ell^j = 1^{n/m}$ and $x_r^j \in C$ which can be increased as follows. At first, uniform crossover is applied on two individuals $x$ and $y$ which have maximum Hamming distance in $x_r^j,y_r^j \in C$ for some $j \notin \mathcal{U}(x)$, while satisfying $x_\ell^j = y_\ell^j = 0^{n/m}$ and matching in all other blocks. Such individuals can easily be created through mutation. This crossover can simply produce an offspring $z$ with $z_\ell^j = 0^{n/m}$ and $z_r^j \in T$. Next, the number of leading ones in $z_\ell^j$ is increased until $z_\ell^j = 1^{n/m}$, after which it is mutated until $z_r^j \in C$ and therefore $j \in \mathcal{U}(x)$. Meanwhile, every block $i \neq j$ satisfies $z_r^i \in C$, and either $z_\ell^i = 0^{n/m}$ if $i \notin \mathcal{U}(x)$ or $z_\ell^i = 1^{n/m}$ if $i \in \mathcal{U}(x)$. As a result, $\mathcal{U}(x)$ increases. Once $\mathcal{U}(x) = [m/2]$, the Pareto front is reached and can be efficiently explored and covered through mutation. However, when crossover is omitted, we will see that it becomes very unlikely to find a Pareto-optimal search point for a large parameter regime of $m$. 

Setting $m=2$ does not recover the bi-objective \RRRMOuni function from~\citep{Dang2024}. We need the modification that Pareto-optimal search points $x$ do not satisfy $x_r^j \in T$ since $T$ has cardinality $\binom{n/(4m)}{n/(8m)}^4$ which is exponentially large in $n/m$. As a result, there can be many pairs of individuals with large Hamming distance. If such two individuals are selected as parents, the probability of producing an offspring $x$ with $x_r^j \in T$ for many blocks $j$ may be $o(1/\text{poly}(n))$ since the probability to create $x_r^i \in T$ for a fixed block $i$ is roughly $O(m^2/n^2)$. To avoid this, we introduce additional fitness signals until $x_\ell^j = 1^{n/m}$ and $x_r^j \in C$ for all $j \in [m/2]$ to finally obtain a Pareto-optimum. 

The following two lemmas highlight key structural properties of $m$-\RRRMOuni that are essential for our subsequent analysis. First, we determine dominance relations among different search points and second, we estimate the cardinality of a maximum set of mutually incomparable solutions. 

\begin{lemma}
\label{lem:RRRMOuni-properties}
The following properties are satisfied for $f:=$\RRRMOuni. 
\begin{itemize}
    \item[(1)] Let $x \in L$ and $y \in M$. Then $y$ dominates $x$.  
    \item[(2)] Let $x \in M$ and $y \in N$ with $\mathcal{U}(y) = \emptyset$. Then $y$ dominates $x$.
    \item[(3)] Let $x,y \in N$ with $|\mathcal{U}(x)|<|\mathcal{U}(y)|$. Then $y$ dominates $x$.
    \item[(4)] A Pareto-optimal set for $m$-\RRRMOuni is
    $$\mathcal{P}_u:=\{x \in N \mid \mathcal{U}(x) = [m/2]\} = \{1^{n/m}x_r^1 \ldots 1^{n/m}x_r^{m/2} \mid x_r^j \in C \text{ for all } j \in [m/2]\}$$
    which has cardinality $(2n/m)^{m/2}$.
\end{itemize}
\end{lemma}

\begin{proof}
Throughout this proof, we write $v \leq w$ for two vectors $v, w \in \mathbb{R}^2$ if $v_i \leq w_i$ for all $i \in \{1,2\}$, and $v < w$ if $v \leq w$ and $v_i < w_i$ for one $i \in \{1,2\}$.

(1): We have $(f_k(y),f_{k+1}(y)) \geq 2n/m \cdot \vec{1}$, while $(f_k(x),f_{k+1}(x)) \leq 2n/m \cdot \vec{1}$ for each odd $k \in [m]$. Moreover, there exists at least one objective $j \in \{k,k+1\}$  for which $g_j(y)<2n/m$.

(2): For every odd $k \in [m/2]$ we have
$$(f_k(x),f_{k+1}(x)) = (g_k(x),g_{k+1}(x)) + (3n/m - \ones{x_\ell^{(k+1)/2}}) \cdot \vec{1} < (g_k(y),g_{k+1}(y)) + 5n/m \cdot \vec{1}.$$

(3): For $k \in [m/2]$ odd we have that 
\begin{align*}
(f_k(y),f_{k+1}(y)) &\geq (g_k(y),g_{k+1}(y)) + (7n/m \cdot |\mathcal{U}(y)| + 5n/m) \cdot \vec{1}\\
&\geq (g_k(y),g_{k+1}(y)) + (7n/m \cdot (|\mathcal{U}(x)|+1) + 5n/m) \cdot \vec{1} \\
&= (g_k(y),g_{k+1}(y)) + (7n/m \cdot |\mathcal{U}(x)| + 12n/m) \cdot \vec{1} \\
&> (g_k(x),g_{k+1}(x)) + (7n/m \cdot |\mathcal{U}(x)| + 10n/m) \cdot \vec{1}.
\end{align*}


(4): By property~(3) all $x \in M \cup N$ with $|\mathcal{U}(x)|<m/2$ are dominated by all $y \in N$ with $|\mathcal{U}(y)|=m/2$. Further, $\{x \in \{0,1\}^n \mid |\mathcal{U}(x)| = m/2\} = \mathcal{P}_u$. With properties (1) and (2) it remains to show that two distinct $x,y \in \mathcal{P}_u$ are incomparable. We have that $(f_k(x),f_{k+1}(x)) = (g_k(x),g_{k+1}(x)) + (7n/2+10n/m) \vec{1}$ for all odd $k \in [m]$. For two distinct $x,y \in \mathcal{P}_u$ we find $k \in [m]$ odd with $x_r^{(k+1)/2} \neq y_r^{(k+1)/2}$ and hence, $\LO(x_r^{(k+1)/2}) \neq \LO(y_r^{(k+1)/2})$ or $\LZ(x_r^{(k+1)/2}) \neq \LZ(y_r^{(k+1)/2})$ resulting in $(g_k(x),g_{k+1}(x)) \neq (g_k(y),g_{k+1}(y))$. Since $g_k(x)+g_{k+1}(x) = g_k(y)+g_{k+1}(y) = 2n/m$, one objective $f_i(x)$ must be larger than $f_i(y)$ for $i \in \{k,k+1\}$ and the other one smaller. Hence, $x$ and $y$ must be incomparable. 
\end{proof}

\begin{lemma}
\label{lem:RRRMOuni-dominance}
Let $S$ be a set of mutually incomparable solutions of $m$-\RRRMOuni. Then either 
\begin{itemize}
\item[(i)] $f(x)=0$ for all $x \in S$, or
\item[(ii)] $S \subset L$, or
\item[(iii)] $S \subset M$, or
\item[(iv)] $S \subset N$ where $|\mathcal{U}(x)|=|\mathcal{U}(y)|$ for all $x,y \in S$.
\end{itemize}
In either case we have $|S| \leq (2n/m)^{m-1}$. 
\end{lemma}


\begin{proof}
Lemma~\ref{lem:RRRMOuni-properties} shows that two different search points $x,y$ satisfying two different properties (i)-(iv) are comparable and therefore, exactly one property holds for $S$. It remains to estimate its cardinality. If $f(x)=0$ for all $x \in S$ then $|S|=1 \leq (2n/m)^{m-1}$ since $m \leq n/8$. Suppose that $S \subset L$. Then we have that $(f_k(x),f_{k+1}(x)) = (g_k(x),g_{k+1}(x))$ and hence, the range of each objective contains at most $2n/m$ distinct values, resulting in $|S| \leq (2n/m)^{m-1}$ by Lemma~\ref{lem:dominance-preparation}(i). For the remainder situations we consider separately the cases where $m=2$ and where $m>2$.

\textbf{Case 1:} $m=2$. Let $(x_\ell,x_r) := (x_\ell^1,x_r^1)$. For any $x \in M \cup N$ we have that $|f_2(x)-f_1(x)| \leq |g_2(x)-g_1(x)| = 2n/m - 1$ (no matter if $\mathcal{U}(x) = \emptyset$ or $\mathcal{U}(x) = \{1\}$) and hence, by Lemma~\ref{lem:dominance-preparation}(ii), we have that the number of non-dominated solutions is at most $n$, regardless if $S \subset M$ or $S \subset N$.


\textbf{Case 2:} $m>2$. Suppose that $M \subset S$. For $x \in M$ and $k \in [m]$ odd we have that 
$$(f_k(x),f_{k+1}(x)) = (g_k(x),g_{k+1}(x)) + (3n/m-\ones{x_\ell^{(k+1)/2}}) \cdot \vec{1}.$$
and we see that $(f_k,f_{k+1})$ attains at most $2n/m \cdot (4n/(6m)+1)$ different values (since $g_k$ attains at most $2n/m$ different values, $g_k(x)+g_{k+1}(x) = 2n/m$ for all $x\in M$ and $\ones{x_\ell^{(k+1)/2}}$ attains at most $4n/(6m)+1$ different values, specifically in $\{0, \ldots , \lfloor{4n/(6m)}\rfloor\}$). Hence, $f$ attains at most 
$(2n/m \cdot (4n/(6m)+1))^{m/2}$ different values for $x \in M$. Further, the set 
$$C_x:=\{y \in M \mid y_r^j = x_r^j \text{ for all } j \in [m/2] \text{ and } \ones{y_\ell^j} = \ones{x_\ell^j} \text{ for all } j \in [m/2-1]\}$$
contains only solutions comparable to $x$ since for $y \in C_x$ we have $(f_k(x),f_{k+1}(x)) = (f_k(y),f_{k+1}(y))$ for $k < m-1$ odd and $(f_{m-1}(x),f_m(x)) = (f_{m-1}(y),f_m(y)) + a \cdot \vec{1}$ for an $a \in \mathbb{Z}$. We have $|f(C_x)| \geq \lfloor{4n/(6m)}\rfloor$ (since the number of ones in block $m-1$ can range from $1, \ldots , \lfloor{4n/(6m)}\rfloor$) and for two solutions $x,y \in M$ we either have $C_x = C_y$ or $C_x \cap C_y = \emptyset$. Since $S$ contains only incomparable solutions, we also see that for every $x \in M$ there is at most one $y \in C_x$ with $y \in S$ and hence, $|S|$ is at most
$$\frac{(2n/m \cdot (4n/(6m)+1))^{m/2}}{\lfloor{4n/(6m)}\rfloor} \leq 2^{m/2+1} \cdot \left(\frac{n}{m}\right)^{m/2} \cdot \left(\frac{4n}{6m}+1\right)^{m/2-1} \leq \left(\frac{2n}{m}\right)^{m-1}$$  
since $m \geq 4$ and $m \leq n/8$. 

Now suppose that $S \subset N$. Note that all $x \in S$ have the same $|\mathcal{U}(x)|$-value. Consider
\begin{align*}
N_1&:=\{x \in N \mid \forall j \in [m/2] \setminus \mathcal{U}(x): x_\ell^j = 0^{n/m} \wedge x_r^j \in C\},\\
N_2&:=\{x \in N \mid \exists! j \in [m/2]: x_\ell^j \in P \setminus \{1^{n/m}\} \wedge x_r^j \in T\},\\
N_3&:=\{x \in N \mid \exists! j \in [m/2]: x_\ell^j = 1^{n/m} \wedge x_r^j \notin C\}.
\end{align*}
Then $N=N_1 \cup N_2 \cup N_3$ and we determine the maximum possible contribution of each $N_i$ to $S$ separately. Then $|S|$ cannot be larger than the sum of these contributions. To this end, we often use the inequality 
\begin{align}
\label{eq:useful}
\frac{3m}{4} \left(\frac{2n\sqrt{2e}}{m}\right)^{m/2} \leq \frac{1}{3} \left(\frac{2n}{m}\right)^{m-1}
\end{align}
for $m \geq 6$ which is equivalent to $3m/4 \cdot (2e)^{m/4} \leq (2n/m)^{m/2-1}/3$. The latter holds due to the inequality $3m/4 \cdot (2e)^{m/4} \leq 16^{m/2 - 1}/3$ for $m \geq 6$ which implies $3m/4 \cdot (2e)^{m/2} \leq 16^{m/2 - 1}/3 \leq (2n/m)^{m/2 - 1}/3$ (since $n$ is divisible by $8m$).

$N_1$: Note that 
$$(f_{2j-1}(x),f_{2j}(x)) = (g_{2j-1}(x),g_{2j}(x)) + (7n/m \cdot |\mathcal{U}(x)|+10n/m) \vec{1}$$
if $j \in \mathcal{U}(x)$ (since $x_\ell^j = 1^{n/m}$) and 
$$(f_{2j-1}(x),f_{2j}(x)) = (g_{2j-1}(x),g_{2j}(x)) + (7n/m \cdot |\mathcal{U}(x)|+5n/m) \vec{1}$$
if $j \notin \mathcal{U}(x)$. Hence, for a fixed $\mathcal{U}(x)$, there are at most $(2n/m)^{m/2}$ distinct fitness values which can be attained by search points from $N_1$. This implies that the contribution of $N_1$ to $S$ is at most
$$\binom{m/2}{|\mathcal{U}(x)|} \left(\frac{2n}{m}\right)^{m/2} \leq \binom{m/2}{m/4} \left(\frac{2n}{m}\right)^{m/2} \leq \left(\frac{em/2}{m/4}\right)^{m/4} \left(\frac{2n}{m}\right)^{m/2}  = \left(\frac{2n\sqrt{2e}}{m}\right)^{m/2} \leq \frac{1}{3} \cdot \left(\frac{2n}{m}\right)^{m-1}$$
if $m > 4$ where we used Equation~\ref{eq:useful} and $\binom{n}{k} \leq (en/k)^k$, the latter by Stirling's formula. If $m=4$ we have a contribution of at most 
$$\binom{2}{1} \cdot \left(\frac{2n}{m}\right)^{m/2} = 2 \cdot \left(\frac{2n}{m}\right)^{m/2} = \frac{16}{8} \cdot \left(\frac{2n}{m}\right)^{m/2} \leq \frac{1}{8} \cdot \left(\frac{2n}{m}\right)^{m-1}$$
$N_2$: For $j \in [m/2]$ consider $N_2^j := \{x \in N_2 \mid x_\ell^j \in P \setminus \{1^{n/m}\} \wedge x_r^j \in T\}$. Then $N_2 = \bigcup_{j=1}^{m/2} N_2^j$. For $x \in N_2^j$ we have that 
$$(f_{2j-1}(x),f_{2j}(x)) = (g_{2j-1}(x),g_{2j}(x)) + (7n/m \cdot |\mathcal{U}(x)|+5n/m + (3n/m-\zeros{x_\ell^{(k+1)/2}})) \cdot \vec{1}$$
and 
$$(f_{2i-1}(x),f_{2i}(x)) = (g_{2i-1}(x),g_{2i}(x)) + (7n/m \cdot |\mathcal{U}(x)| + a) \cdot \vec{1}$$
for $i \neq j$ where $a \in \{5n/m,10n/m\}$. We have $a=10n/m$ if $i \in \mathcal{U}(x)$ and $a=5n/m$ otherwise. Define $V_j:= \{(f_1(x), \ldots , f_{2j-2}(x),f_{2j}(x), \ldots , f_m(x)) \mid x \in N_2^j\}$ and let $S_j \subset N_2^j$ a set of mutually incomparable solutions. Since two $x,y \in N_2^j$ with $(f_1(x), \ldots , f_{2j-2}(x),f_{2j}(x), \ldots , f_m(x)) = (f_1(y), \ldots , f_{2j-2}(y),f_{2j}(y), \ldots , f_m(y))$ are comparable, we see that $|S_j| \leq |V_j|$. Further, if $\mathcal{U}(x)$ is fixed, for $x \in N_2^j$ the objectives $f_{2j-1}(x)$ and $f_{2j}(x)$ attain at most $3n/m$ different values while for $i \neq j$ there are only $2n/m$ different values for $(f_{2i-1}(x),f_{2i}(x))$ (since $|\mathcal{U}(x)|$ is the same for all $x \in S$ and $g_{2i-1}(x)+g_{2i}(x)=2n/m$). 
Hence, we obtain a contribution to $S_j$ of at most (since $\mathcal{U}(x) \subset [m/2] \setminus \{j\}$)
$$|V_j| \leq \frac{3n}{m} \cdot \binom{m/2-1}{|\mathcal{U}(x)|} \left(\frac{2n}{m}\right)^{m/2-1} \leq \frac{3n}{m} \cdot \binom{m/2}{m/4} \left(\frac{2n}{m}\right)^{m/2-1} = \frac{3}{2} \cdot \binom{m/2}{m/4} \left(\frac{2n}{m}\right)^{m/2} \leq \frac{3}{2} \cdot \left(\frac{2n\sqrt{2 e}}{m}\right)^{m/2}$$ 
where the last inequality can be derived in a similar way as above in the case for $N_1$. Finally, the total contribution to $S$ cannot be larger than the sum of the contributions to the single $S_j$ which is at most
$$\frac{m}{2} \cdot \frac{3}{2} \cdot \left(\frac{2n\sqrt{2 e}}{m}\right)^{m/2} = \frac{3m}{4} \cdot \left(\frac{2n\sqrt{2 e}}{m}\right)^{m/2} \leq \frac{1}{3} \cdot \left(\frac{2n}{m}\right)^{m-1}$$
with Equation~\ref{eq:useful} if $m > 4$. When $m=4$ we obtain a contribution to $S$ of at most 
$$\frac{m}{2} \cdot \frac{3n}{m} \cdot \binom{2}{1} \left(\frac{2n}{m}\right)^{m/2-1} = \frac{m}{2} \cdot \frac{6n}{m} \cdot \left(\frac{2n}{m}\right)^{m/2-1} = \frac{3m}{2} \cdot \left(\frac{2n}{m}\right)^{m/2} \leq \frac{6}{16} \cdot \frac{2n}{m} \cdot \left(\frac{2n}{m}\right)^{m/2} = \frac{3}{8} \cdot \left(\frac{2n}{m}\right)^{m-1}$$
$N_3$: The estimation is similar as in the previous case since for all $j \in [m/2]$ there is $a \in \{5n/m,10n/m\}$ with
$$(f_{2j-1}(x),f_{2j}(x)) = (g_{2j-1}(x),g_{2j}(x)) + (7n/m \cdot |\mathcal{U}(x)|+a) \cdot \vec{1}.$$
For $j \in [m/2]$ consider $N_3^j := \{x \in N_3 \mid x_\ell^j = 1^{n/m} \wedge x_r^j \notin C\}$. Then $N_3 = \bigcup_{j=1}^{m/2} N_3^j$. Let $S_j \subset N_3^j$ a set of mutually incomparable solutions. Let $V_j:= \{(f_1(x), \ldots , f_{2j-2}(x),f_{2j}(x), \ldots , f_m(x)) \mid x \in N_3^j\}$. Similar to the above, we observe that $|S| \leq |V|$ and that $f_{2j}(x)$ attains $2n/m$ distinct values, and $(f_{2i-1}(x), f_{2i}(x))$ also attains $2n/m$ distinct values for $i \neq j$ if $\mathcal{U}(x)$ is fixed (since $g_{2i-1}(x)+g_{2i}(x)=2n/m$). Then the contribution of $N_3$ to $S$ of at most 
$$\frac{m}{2} \cdot \binom{m/2-1}{|\mathcal{U}(x)|} \cdot \frac{2n}{m} \cdot \left(\frac{2n}{m}\right)^{m/2-1} \leq \frac{1}{3} \cdot \left(\frac{2n}{m}\right)^{m-1}$$
if $m>4$ and $3/8 \cdot (2n/m)^{m-1}$ if $m=4$.
In all the cases, by summing up the single contributions of $N_1$, $N_2$ and $N_3$ to $S$, we obtain that $S \leq (2n/m)^{m-1}$, concluding the proof.
\end{proof}

\section{Runtime Analyses of GSEMO and NSGA-III with Uniform Crossover on $m$-\RRRMOuni}

As for the $m$-\RRRMO case, we will show that NSGA-III and GSEMO exhibit a significant performance gap depending on whether uniform crossover is used, across a wide range of the number $m$ of objectives which becomes exponential if $m$ is constant.

\subsection{Analysis of NSGA-III}

\begin{theorem}
\label{thm:Runtime-Analysis-NSGA-III-mRRMO-uniform}
Let $m \in \mathbb{N}$ be divisible by $2$ and $n$ be divisible by $8m$. Then the algorithm \nsgaIII (Algorithm~\ref{alg:nsga-iii}) with $p_c \in (0,1)$, $\varepsilon_{\text{nad}} \geq 7n/2+12n/m$, a set $\refer$ of reference points as defined above for $p \in \mathbb{N}$ with $p \geq 2m^{3/2}(7n/2+12n/m)$, and a population size $\mu \geq (2n/m)^{m-1}$, $\mu = 2^{O(n^2)}$, finds a Pareto set of $f:=m\text{-\RRRMOuni}$ in expected $O(m n^2/(1-p_c)+ n^2/(p_c \cdot m))$ generations and $O(\mu m n^2/(1-p_c)+ \mu n^2/(p_c \cdot m))$ fitness evaluations.  
\end{theorem}

\begin{proof}
We have $f_{\max}=7n/2+12n/m$ since $|\mathcal{U}(x)|$ has a maximum value of $m/2$ and for odd $k$ the functions $g_k(x)$ and $g_{k+1}(x)$ have a maximum value of $2n/m$. So we may always apply Lemma~\ref{lem:Reference-Points} to protect non-dominated solutions which means that for every non-dominated $x \in P_t$ there is $x' \in P_{t+1}$ weakly dominating $x$. 

\textbf{Phase 1:} Create $x \in P_t$ with $f(x) \neq 0$.

An individual $x$ has fitness distinct from $0$ if $x_\ell^j \in U$ for all $j \in [m/2]$ since in this case, $x \in L \cup M$. By a classical Chernoff bound the probability that $x_\ell^j \in U$ for a fixed $j \in [m/2]$ is $p^{\text{init}}:=1-e^{-\Omega(n/m)}$ after initializing $x$ uniformly at random. Note that there is a constant $0<c<1$ such that $p^{\text{init}} \geq c$ regardless of $m$. Further, the probability that $x_\ell^j \in U$ for all $j \in [m/2]$ is $(p^{\text{init}})^{m/2}$. After initializing $\mu$ individuals uniformly at random, the probability that there is one individual $x$ satisfying $x_\ell^j \in U$ for all $j \in [m/2]$ is $1-(1-(p^{\text{init}})^{m/2})^{\mu}$. If $m=2$ then $1-(1-(p^{\text{init}})^{m/2})^{\mu} = 1-(1-p^{\text{init}})^{\mu} =  1-e^{-\Omega(\mu n)} = 1-e^{-\Omega(n^2)}$ since $\mu = \Omega(n)$. If $m>2$ then 
\begin{align*}
1-(1-(p^{\text{init}})^{m/2})^{\mu} &\geq 1-e^{-\mu \cdot (p^{\text{init}})^{m/2}} \geq 1-e^{-(2n/m)^{m-1} \cdot (p^{\text{init}})^{m/2}}\\
&\geq 1-e^{-c^{m/2} (2n/m)^{m-1}} = 1-e^{-\sqrt{c}(2\sqrt{c}n/m)^{m-1}} = 1-e^{-\Omega(n^3)}.
\end{align*}
For $m = \Omega(n)$ we see that $(2\sqrt{c}n/m)^{m-1} = e^{-\Omega(n)}$ due to $m \leq n/8$ and hence, the last equality holds. It is also satisfied for $m = o(n)$ by Lemma~\ref{lem:function-monotone}(iii) (applied on $k=2\sqrt{c}n$) since $m \geq 4$. 
Hence, the probability that Phase~1 is not finished after initializing $\mu$ individuals is at most $e^{-\Omega(n^2)}$ regardless of the value of $m$. Then, any individual can be created through mutation with probability $n^{-n}$, regardless of whether crossover is applied. Consequently, the expected number of generations required to complete this phase is at most $(1 - e^{-\Omega(n^2)}) + n^n \cdot e^{-\Omega(n^2)} = 1+o(1)$. 

\textbf{Phase 2:} Create $x \in M$, or in other words an $x$ with $x_r^j \in C$ and $\ones{x_\ell^j} \leq 4n/(6m)$ for all $j \in [m/2]$.

Let $Z_t:=\max\{\sum_{j=1}^m g_j(x) \mid x \in P_t, x_\ell^j \in U \text{ for all } j \in [m/2]\}$. Then $m/2 + n/2 \leq Z_t \leq n$ since for odd $k \in [m]$ we either have $\LO(x^{(k+1)/2})>0$ (implying $g_k(x) \geq 1$ and $g_{k+1}(x) \geq n/m$) or $\LZ(x^{(k+1)/2})>0$ (implying $g_k(x) \geq n/m$ and $g_{k+1}(x) \geq 1$) yielding $g_k(x)+g_{k+1}(x) \geq n/m+1$. We also have $g_k(x)+g_{k+1}(x) \leq 2n/m$ which implies the upper bound. Further, $Z_t$ cannot decrease by Lemma~\ref{lem:Reference-Points} since a solution $x$ with $\sum_{k=1}^m g_k(x)=Z_t$ is non-dominated due to its maximum sum of objective values. 
To increase $Z_t$ in one trial it suffices to choose a parent $z \in P_t$ with $\sum_{k=1}^m g_k(z)=Z_t$ (prob. at least $1-(1-1/\mu)^2 \geq 1/\mu$), omit crossover (prob. $(1-p_c)$) and flip  a specific bit (to increase $\LO(x^j)+\TZ(x^j)$ or $\LZ(x^j)+\TOs(x^j)$ in a block $j$ where this sum is not maximum) (prob. at least $1/n \cdot (1-1/n)^{n-1} \geq 1/(en)$). Hence, in one generation, the probability to increase $Z_t$ is at least 
$$1-\left(1-\frac{1-p_c}{en \mu}\right)^{\mu/2} \geq \frac{(1-p_c)/(2en)}{1+(1-p_c)/(2en)} \geq \frac{(1-p_c)}{4en}$$
where the last inequality holds by Lemma~\ref{lem:Badkobeh}. Hence, the expected number of generations to complete this phase is at most 
$$\left(\frac{n}{2}-\frac{m}{2}\right) \cdot \frac{4en}{(1-p_c)} = O\left(\frac{n^2}{1-p_c}\right).$$

\textbf{Phase 3:} Create $x \in N$.

We estimate the expected time until we created $x$ with $x_r^j \in C$ and $x_\ell^j = 0^{n/m}$ for all $j \in [m/2]$ since such an $x$ is contained in $N$. Let $O_t:=\max\{\sum_{j=1}^{m/2} \zeros{x_\ell^j} \mid x \in P_t \cap M\}$. Then $O_t \in \{\lceil{n/6}\rceil, \ldots , n/2\}$ (since the number of zeros of $x \in M$ is at least $n/m - 4n/(6m) = 2n/(6m)$ in each left block) and this phase is finished if $O_t=n/2$. Further, $O_t$ cannot decrease since a solution $x$ with $x \in M$ and $\sum_{j=1}^{m/2} \zeros{x_\ell^j} = O_t$ is non-dominated (since the sum of all objectives is maximum). To increase $O_t$ in one trial one has to choose an individual $y$ with $\sum_{j=1}^{m/2} \zeros{y_\ell^j} = O_t$ as a parent, omit crossover and flip one of $n/2-O_t$ ones in one of the right blocks while not flipping another bit. This happens with probability at least $(1-p_c)(n/2-O_t)/(en\mu)$ in one trial and hence, with probability at least 
$$1-\left(1-\frac{(1-p_c)(n/2-O_t)}{en \mu}\right)^{\mu/2} \geq \frac{(1-p_c)(n/2-O_t)/(2en)}{1+(1-p_c)(n/2-O_t)/(2en)} \geq \frac{(1-p_c)(n/2-O_t)}{4en}$$
in one generation. Hence, the expected number of generations until $O_t=n/2$ is at most 
$$\sum_{j=\lceil{n/6}\rceil}^{n/2-1} \frac{4en}{(1-p_c)(n/2-j)} = O\left(\frac{n \log(n)}{1-p_c}\right).$$

For defining the next phases, let $\gamma_t:=\max\{\vert{\mathcal{U}(x)}\vert \mid x \in N\}$. Note that $\gamma_t \in \{0, \ldots , m/2\}$ and we have $\gamma_t = m/2$ if a Pareto-optimal individual is created (by Lemma~\ref{lem:RRRMOuni-properties}(4)). Suppose that there is an individual $z \in P_t \cap N$ with $\vert{\mathcal{U}(z)}\vert = \gamma_t$ where $\gamma_t < m/2$, but no corresponding individual $w \in P_t \cap N$ with $\vert{\mathcal{U}(w)}\vert > \gamma_t$. 

\textbf{Phase $\gamma_t$ + 4:} Create $x \in N$ with $|\mathcal{U}(x)|>\gamma_t$.

Phase~4 starts when $\gamma_t=0$. If Phase~$m/2 + 3$ is finished, then there is an individual $x \in P_t \cup N$ with $\mathcal{U}(x) = [m/2]$ and hence, a Pareto-optimal search point has been created. 

\textbf{Subphase~A:} $P_t \subset N$ and $|\mathcal{U}(z)| = \gamma_t$ for every $z \in P_t$. 

Let $S_t:=\{z \in P_t \mid z \in N \text{ and } |\mathcal{U}(z)| = \gamma_t\}$. By Lemma~\ref{lem:RRRMOuni-properties}(1)-(3) we see that an individual $z \in S_t$ dominates every other $w \in P_t \setminus S_t$ since either $w \in L \cup M$ or $w \in N$ with $|\mathcal{U}(w)|< |\mathcal{U}(z)|$. Hence, $|S_t|$ cannot decrease. As in Subphase~A in the proof of Theorem~\ref{thm:Runtime-Analysis-NSGA-III-mRRMO-onepoint}, this phase can be finished by producing $\mu-1$ clones of individuals from $S_t$ happening in at most $O(n^2/(1-p_c))$ generations in expectation since $\mu = 2^{O(n^2)}$.

Denote by $t^*$ the first generation when Subphase~A is finished. We can assume that $x^j_r \in C$ and $x^j_\ell = 0^{n/m}$ for each $x \in P_t$ and every $j \notin \mathcal{U}(x)$ since otherwise we reach the goal of a subsequent phase. Let
\begin{align*}
\mathcal{K}_t:=\{\mathcal{U} \subset [m/2] \mid |\mathcal{U}| = \gamma_t \wedge \exists x \in P_t: \mathcal{U}(x) = \mathcal{U} \wedge x_r^j \in C \wedge x_\ell^j = 0^{n/m} \text{ for all $j \in [m/2] \setminus \mathcal{U}(x)$})\}.
\end{align*}
Denote the increase of $\mathcal{K}_t$ in an iteration $t \geq t^*$ before increasing $\gamma_t$ as a \emph{failure}. The following claim estimates the probability that a failure occurs.   

\begin{claim}
\label{claim:failure-probability}
The probability $p_{\text{fail}}$ that a failure occurs is at most $5/6$.
\end{claim}

\begin{proofofclaim}
Denote the event that $\gamma_t$ increases before a failure occurs as a \emph{success} and denote its probability by $p_{\text{succ}}$. We show that $5 p_{\text{succ}} \geq p_{\text{fail}}$ and the claim follows by noting that $p_{\text{fail}} \leq 1- p_{\text{succ}} \leq 1-p_{\text{fail}}/5$. For $z \in \{0,1\}^n$ denote by $p_z$ the probability that the outcome $z$ is produced in one trial. 

At first suppose that only mutation is executed (with probability $(1-p_c)$) and fix the parent $y \in P_t$. Let $O_t$ be the set of all possible outcomes $z \in N$ with $|\mathcal{U}(z)|=\gamma_t$, $\mathcal{U}(z) \neq \mathcal{U}(y)$ and $z_r^j \in C$ for all $j \in [m/2]$ (i.e. either $x_\ell^j=1^{n/m}$ or $x_\ell^j = 0^{n/m}$ for all $j \in [m/2]$). Note that a failure can only occur if $z \in O_t$ and hence, the failure probability $p_{\text{failure}}$ is at most $\sum_{z \in O_t} p_z$. Observe that for $z \in O_t$ there are $i,j \in [m/2]$ with $i \neq j$ such that $y_\ell^i = 0^{n/m}, y_\ell^j = 1^{n/m}$, but $z_\ell^i = 1^{n/m}, z_\ell^j = 0^{n/m}$. For $z \in O_t$ denote by $z^*$ the corresponding outcome with $(z^*)_\ell^j = 1^{n/m}$ for the smallest $j \in [m/2]$ with $y_\ell^j = 1^{n/m}$ and $z_\ell^j = 0^{n/m}$, and coinciding with $z$ in all other blocks (including $z_r^j$) and denote by $O_t^*:=\{z^* \mid z \in O_t\}$ the set of all such possible corresponding $z^*$. Then for $z^* \in O^*_t$ we see that $z^* \in N$ and $|\mathcal{U}(z^*)| = \gamma_t+1 >\gamma_t$. Hence, if such a $z^*$ is produced a success occurs and therefore $p_{\text{succ}} \geq \sum_{z^* \in O_t^*} p_{z^*}$.
For $z^* \in O^*_t$ denote by $O_t^{z^*}:=\{w \in O_t \mid w^* = z^*\}$ the set of all outcomes $w$ from $O_t$ whose corresponding $w^*$ coincides with $z^*$. We obtain  $p_{z^*} ´\geq \sum_{w \in O_t^{z^*}}p_w \cdot n^{n/m} \cdot (1-1/n)^{n/m}$ since the probability to flip no bits in $y_\ell^j$ is $(1-1/n)^{n/m}$ while flipping all bits is $(1/n)^{n/m}$. Since each $z \in O_t$ belongs to a $O_t^{w^*}$ for a $w^* \in O_t^*$ we see that $\sum_{z^* \in O^*_t} p_{z^*} \geq \sum_{z \in O_t} p_z \cdot n^{n/m} \cdot (1-1/n)^{n/m}$ and hence, $p_{\text{succ}} \geq p_{\text{failure}} \cdot n^{n/m} \cdot (1-1/n)^{n/m} \geq p_{\text{failure}}/5$.

Suppose that crossover is executed followed by mutation (with probability $p_c$) and let $y_1,y_2 \in P_t$ be the parents. If $\mathcal{U}(y_1) = \mathcal{U}(y_2)$ we can argue in a similar way as in the mutation only case since crossover has no effect on the left blocks. So suppose that $\mathcal{U}(y_1) \neq \mathcal{U}(y_2)$. Denote by $Q_t$ the set of all possible outcomes $z$ with $\gamma_t$ all-one strings and $n/2-\gamma_t$ all-zero strings in left blocks. 
Let $a := |\mathcal{U}(y_1) \setminus \mathcal{U}(y_2)|$. Then, the bit strings $(y_1)_\ell^j$ and $(y_2)_\ell^j$ differ in exactly $2a$ blocks $j \in [2m]$, which are the only left blocks where crossover has an effect. After applying crossover and mutation, we have $|\mathcal{U}(z)| = \gamma_t$ if exactly half of these $2a$ differing blocks result in $1^{n/m}$ and the other half in $0^{n/m}$, while the remaining left blocks are not changed. If, instead, there are $j \in \{0, \dots, 2a\}$ blocks among the $2a$ left blocks that become $1^{n/m}$ (or $0^{n/m}$), then mutation must flip at least $|j-a|$ additional left blocks entirely to ensure that $|\mathcal{U}(z)| = \gamma_t$. Such specific $|j-a|$ blocks are mutated entirely with probability at most $n^{-|a-j|n/m}$. Furthermore, for a crossover followed by mutation on $(0^{n/m},1^{n/m})$, the probability of producing a specific bit string is $2^{-n/m}$. Hence, the probability that $z \in Q_t$ is at most 
\begin{align*}
\sum_{j=0}^{2a} \binom{2a}{j} \binom{m/2}{|a-j|} 2^{-2a n/m} n^{-|a-j|n/m} &= \binom{2a}{a} 2^{-2a n/m} + 2 \sum_{j=0}^{a-1} \binom{2a}{j} \binom{m/2}{a-j} 2^{-2a n/m} n^{-(a-j)n/m} \\
&\leq \binom{2a}{a} 2^{-2a n/m} + 2 \sum_{j=0}^{a-1}\binom{2a}{a-1} \binom{m/2}{1} 2^{-2a n/m} n^{-n/m}\\
&\leq \binom{2a}{a} 2^{-2a n/m} + \binom{2a}{a} a m 2^{-2a n/m} n^{-n/m} \\
&= \binom{2a}{a} 2^{-2a n/m}(1+ am n^{-n/m}).
\end{align*}
For a given $w = (w_r^1, \ldots , w_r^{m/2}) \in C^{m/2}$ denote by $p_w$ the probability that $z_r^j = w_r^j$ for all $j \in [m/2]$. Then a necessary condition that a failure occurs is that an outcome $z$ with $z \in Q_t$ and $z_r^j \in C$ for all $j \in [m/2]$ is produced. Therefore,
$$p_{\text{fail}} \leq \binom{2a}{a} 2^{-2a n/m}(1+ am n^{-n/m}) \sum_{w \in C^{m/2}} p_w.$$
Furthermore, a success occurs if $z_r^j \in C$ for all $j \in [m/2]$ after crossover and mutation, while $j > a$ blocks among the $2a$ left blocks become $1^{n/m}$ during crossover, the remaining blocks become $0^{n/m}$, and mutation flips no bit in any left block. This happens with probability at least
$$\left(\sum_{w \in C^{m/2}}p_w\right) \cdot \sum_{j=0}^{a-1}\binom{2a}{j} 2^{-2a n/m} \left(1-\frac{1}{n}\right)^{n/2} \geq \binom{2a}{a} \frac{2^{-2a n/m}}{4} \cdot \sum_{w \in C^{m/2}}p_w \geq \frac{p_{\text{fail}}}{4(1+amn^{-n/m})} \geq \frac{p_{\text{fail}}}{5}$$
(since $\sum_{j=0}^{a-1}\binom{2a}{j} \geq \binom{2a}{a-1} \geq \binom{2a-1}{a-1} = \binom{2a}{a}/2$ and $(1-1/n)^{n/2} \geq 1/2$) that $\gamma_t$ increases. Hence, $5 p_{\text{succ}} \geq p_{\text{fail}}$.

By the law of total probability on the mutation and crossover cases above, and on all possible parents in each case, the statement follows.
\end{proofofclaim}

Now we consider the next subphases as follows where we define
$$\mathcal{A}_t:=\{x \in N \mid \mathcal{U}(x) \in \mathcal{K}_t, x_\ell^j = 0 \text{ and } x_r^j \in C \text{ for all } j \in [m/2] \setminus \mathcal{U}(x)\}.$$

\textbf{Subphase B:} 
Find all individuals in $\mathcal{A}_t$ or a failure occurs. 

We may assume that no failure occurs, as the occurrence of a failure only accelerates the completion of the phase. We give an upper bound for the duration of this phase for all possible $\mathcal{A}:=\mathcal{A}_t$. Fix $y \in \mathcal{A}$ which is not already found. Let $d_t:=\min\{H(x,y) \mid x \in P_t \cap \mathcal{A} \wedge \mathcal{U}(x)=\mathcal{U}(y)\}$. Since $P_t \cap \mathcal{A} \neq \emptyset$ we have $1 \leq d_t \leq n/m (m/2 - \gamma_t) \leq n/2$. Further, $d_t$ cannot increase since all solutions in $\mathcal{A}$ are non-dominated. One can decrease $d_t$ by choosing an individual $w \in \mathcal{A}$ with $H(w,y) = d_t$ as a parent where $w_\ell^j \neq y_\ell^j$ for $j \in [m/2]$, omitting crossover and then flipping 
\begin{itemize}
\item the first zero in $w_\ell^j$ from the left to increase $\LO(w_\ell^j)$ if $0<\LO(w_\ell^j)<\LO(y_\ell^j)$ or 
\item the first one in $w_\ell^j$ from the right to decrease $\LO(w_\ell^j)$ if $0<\LO(y_\ell^j)<\LO(w_\ell^j)$ or
\item the first one in $w_\ell^j$ from the left to increase $\LZ(w_\ell^j)$ if $0<\LZ(w_\ell^j)<\LZ(y_\ell^j)$ or 
\item the first zero in $w_\ell^j$ from the right to decrease $\LZ(w_\ell^j)$ if $0<\LZ(y_\ell^j)<\LZ(w_\ell^j)$ 
\end{itemize}
while not changing the remaining bits. Such a decrease happens with probability at least $(1-p_c)/(\mu e n)$ in one trial and hence, with probability at least 
$$1-\left(1-\frac{1-p_c}{\mu e n}\right)^{\mu/2} \geq \frac{(1-p_c)/(2en)}{1+(1-p_c)/(2en)} \geq \frac{1-p_c}{4en}$$
in one generation. Hence, we see that the time $T_y$ to find $y$ is stochastically dominated by the independent sum $Z:=\sum_{j=1}^{n/2} Z_j$ of geometrically distributed random variables with success probability $q:=(1-p_c)/(4en)$. We have $\expect{Z}=2en^2/(1-p_c)$ and we can use Theorem~\ref{thm:Doerr-dominance}(1) to obtain for $\lambda:= 16en^2/(1-p_c)$, $s:= 8e^2n^3/(1-p_c)^2$ and $p_{\min} = q$
$$\Pr\left(T_y \geq \frac{18en^2}{1-p_c}\right) \leq \Pr\left(Z \geq \frac{18en^2}{1-p_c}\right) \leq \exp\left(-\frac{1}{4} \min\left\{\frac{\lambda^2}{s}, \lambda p_{\min}\right\}\right) = e^{-n}.$$
We also have that
$$|\mathcal{A}| \leq \binom{m/2}{m/4} \cdot \left(\frac{2n}{m}\right)^{m/2} \leq 2^{m/2} \cdot (16)^{n/16} \leq e^{n\ln(2)/16 + n \ln(16)/16} \leq e^{n/2}$$
where the second inequality holds due to Lemma~\ref{lem:function-monotone}(iv): the function $]0,2n/e] \to \mathbb{R}, x \mapsto (2n/x)^{x/2},$ is strictly monotone increasing and hence, we obtain $(2n/m)^{m/2}\leq (16)^{n/16}$ by pluggin in $n/8$ for $x$ (which is the maximum possible value for $m$). Therefore, by a union bound on all possible $y$, we obtain that with probability at most $e^{-n/2}$ that all $y \in \mathcal{A}$ are found within $18en^2/(1-p_c)$ generations. If this does not happen we repeat the above arguments where the expected number of periods is $1+o(1)$. Hence, $\expect{T}=O(n^2/(1-p_c))$.

\textbf{Subphase C:} Find an individual $x \in N$ with $|\mathcal{U}(x)|=\gamma_t$, and $x_\ell^j \in P \setminus \{1^{n/m}\} \wedge x_r^j \in T$ for exactly one $j \in [m/2]$ or a failure occurs.

We give an upper bound for the duration of this phase for all possible $\mathcal{A}$. To finish this phase, it suffices to choose two $y,z \in P_t \subset \mathcal{A}$ where for a block $k \in [m/2]$ we have $(z_r^k)_i = 1 - (y_r^k)_i$ for all $i \in [n/m]$ whereas $z_\ell^k = y_\ell^k = 0^{n/m}$, and both $y$ and $z$ coincide at the remaining blocks distinct from $k$ (prob. at least $1/\mu$ since the number of such pairs is at least $\mu$ and we successfully passed through Subphase~B), perform uniform crossover to create $w$ with $\ones{w_{i,\ell}^k} = \zeros{w_{i,\ell}^k}$ for all $i \in [4]$ which happens with probability
$$p^{\text{cross}} := p_c \cdot \left(\binom{n/(4m)}{n/(8m)} \cdot \left(\frac{1}{2}\right)^{n/(4m)}\right)^4 = \Omega(p_c m^2/n^2)$$
by Stirling's formula and omit mutation (prob. $(1-1/n)^n \geq 1/4$). In one trial this happens with probability at least $p^{\text{cross}}/(4\mu)$ and hence, with probability at least $p^{\text{cross}}/16= \Omega(p_cm^2/n^2)$ in one generation. Hence, the expected number of generations to complete this phase is not more than $O(n^2/(p_c m^2))$.

If a failure in Subphase~B or~C occurs then we can repeat the above arguments in both subphases since the time for passing through both hold for all possible $\mathcal{A}$. Since the probability that a failure happens during Subphase~B or~C is at most $5/6$ (by Claim~\ref{claim:failure-probability}), the expected number of traversals through Subphase~B and~C is at most $1/(1-5/6) = 6=O(1)$ showing that the expected number of generations to finish Subphases~B and~C including repetitions is at most $O(n^3/(1-p_c)+n^2/(p_cm^2))$.

\textbf{Subphase D:} Find an individual $x \in N$ with $|\mathcal{U}(x)|=\gamma_t$, and $x_\ell^j = 1^{n/m}$ for exactly one $j \in [m/2] \setminus \mathcal{U}(x)$.

This phase is similar as Phase~3 from above with the small difference that we optimize the number of ones $\ones{x_\ell^j}$ in a certain block $j$ with $x_r^j \in T$. Let $Z_t:=\min\{\zeros{x_\ell^j} \mid x \in N \text{ with } x_r^j \in T \text{ for a } j \in [m/2]\}$. Then $Z_t \in \{1, \ldots , n/m\}$ and this phase is finished if $Z_t=0$. Further, $Z_t$ cannot increase since a solution $x$ with $x \in N$ and $\zeros{x_\ell^j} = Z_t$ where $x_r^j \in T$ is non-dominated (since the sum of all objectives is maximum). To decrease $Z_t$ one may choose an individual $y$ with $\zeros{y_\ell^j} = Z_t$ where $y_r^j \in T$ as a parent, omit crossover and flip a specific zero (to increase the $\LO(y)$ value) while not flipping another bit. This happens with probability at least $(1-p_c)/(en\mu)$ in one trial and hence, with probability at least $(1-p_c)/(4en)$ in one generation. Hence, the expected number of generations until $Z_t=0$ is at most $O(n^2/(m(1-p_c)))$.

\textbf{Subphase E:} Find an individual $x \in N$ with $|\mathcal{U}(x)|=\gamma_t+1$.

This phase is similar as Phase~2 from above with the small difference that we have to find $y$ with $y_r^j \in C$ where $j \in [m/2]$ with $y_\ell^j = 1^{n/m}$, but $j \notin \mathcal{U}(y)$. Let $Z_t:=\max\{g_{2j-1}(x) + g_{2j}(x) \mid x \in N \text{ with } x_\ell^j = 1^{n/m} \text{ and } x_r^j \notin C \text{ for } j \in [m/2]\}$. Then $n/m+1 \leq Z_t < 2n/m$ and this phase is finished if $Z_t$ reaches $2n/m$. Further, $Z_t$ cannot decrease since a solution $y$ with a corresponding value of $Z_t$ is non-dominated (since it has a maximum sum of objective values). As in Phase~2 above, $Z_t$ can be increased with probability at least $(1-p_c)/(4en)$ in one generation which leads to an expected number of $O(n^2/(m(1-p_c)))$ generations to finish this phase.

Hence, we see that Phase~$\gamma_t$+4 is finished in expected 
$$O\left(\frac{n^2}{1-p_c} + \frac{n^2}{1-p_c} + \frac{n^2}{p_c \cdot m^2} + \frac{n^2}{(1-p_c) \cdot m} + \frac{n^2}{(1-p_c) \cdot m}\right) = O\left(\frac{n^2}{1-p_c} + \frac{n^2}{p_c m^2}\right)$$
generations.

\textbf{Phase m/2+4:} Cover the whole Pareto front.

The treatment of this phase is similar to Subphase~B with the minor difference that we want to find all individuals in $\mathcal{P}_u:=\{y \mid \mathcal{U}(y) = [m/2]\} = \{1^{n/m}x_r^1 \ldots 1^{n/m}x_r^{m/2} \mid x_r^j \in C \text{ for all } j \in [m/2]\}$. Note that $P_t \cap \mathcal{P}_u \neq \emptyset$. Fix $y \in \mathcal{P}_u \setminus P_t$. Define $d_t = \min\{H(x,y) \mid x \in P_t \cap \mathcal{P}_u\}$. Then $1 \leq d_t \leq n/2$ and similar as in Subphase~B we see that we can decrease $d_t$ with probability at least $(1-p_c)/(4en)$ in one generation and hence, we also obtain $P(T \geq 18en^2/(1-p_c)) = e^{-n}$ for the time $T$ to find $y$. As in Subphase~B, by a union bound on all possible $y$ (since $|\mathcal{P}_u| \leq (2n/m)^{m/2} \leq e^{-n/2}$) we see that with probability at most $e^{-n/2}$ that all $y \in \mathcal{P}_u$ are found within $18en^2/(1-p_c)$ generations. If this does not happen we repeat the above arguments where the expected number of periods is $1+o(1)$. Hence, $\expect{T}=O(n^2/(1-p_c))$.

Now we upper bound the total runtime. Since Subphases~A,B and~C are passed at most $m/2$ times, the complete Pareto front is covered in expected 
$$O\left(\frac{n^2}{1-p_c} + \frac{n \log(n)}{1-p_c} + \frac{m n^2}{(1-p_c)} + \frac{mn^2}{p_c m^2} + \frac{n^2}{1-p_c}\right) = O\left(\frac{m n^2}{1-p_c}+ \frac{n^2}{p_c \cdot m}\right)$$
generations and $O(\mu m n^2/(1-p_c)+ \mu n^2/(p_c \cdot m))$ fitness evaluations since Subphases~A,B and~C are passed at most $m/2$ times. This finishes the proof.
\end{proof}

If $p_c \in (0,1)$ is constant the expected number of generations to optimize $m$-\RRRMOuni is $O(n^2 \cdot m)$ and hence, does not asymptotically depend on the population size $\mu$ which is also an improvement to  the corresponding result in~\citep{Dang2024}. They showed a runtime bound of $O(n^2/(1-p_c)+\mu n / p_c)$ generations in expectation for \nsga on their bi-objective version of \RRRMOuni which is by a factor of $\mu/n$ worse for population sizes $\mu \in \Omega(n)$ and constant $p_c \in (0,1)$. 

\subsection{Analysis of GSEMO}
As for the one-point crossover case above, we give a corresponding result also for the classical GSEMO in terms of fitness evaluations. Our result would also hold for an arbitrary $m$ if GSEMO initializes a search point with fitness distinct from zero.
\begin{theorem}
\label{thm:Runtime-Analysis-GSEMO-mRRMO-uniform}
Let $m \in \mathbb{N}$ be divisible by $2$ and $n$ be divisible by $8m$. Suppose that $m \leq d \sqrt{n/\log(n)}$ for a sufficiently small constant $d$. Then the GSEMO algorithm (Algorithm~\ref{alg:gsemo}) with $p_c \in (0,1)$ finds a Pareto set of $f:=m\text{-\RRRMOuni}$ in expected 
$$O\left(\frac{n^2(2n/m)^{m-1}}{1-p_c}+ \frac{n^2 (2n/m)^{m/2}}{p_c m}\right)$$
fitness evaluations.
\end{theorem}

\begin{proof}

Note that the runtime in terms of fitness evaluations corresponds to the runtime in terms of iterations. As in the previous proofs, we use the method of typical runs and divide the optimization procedure into several phases.

\textbf{Phase 1:} Find $x$ with $f(x) \neq 0$.

By a classical Chernoff bound the probability that the initial search point $x$ fulfills $\ones{x_\ell^j} \in U$ for all $j \in [m/2]$ is $1-e^{-\Omega(n/m)}$. Suppose that this does not happen. For a block $j \in [m/2]$ and $i \in [4]$, denote by $X_t^{i,j}$ the number of ones in $x_{i,\ell}^j$ in iteration $t$. Such a block has length $n/(4m)$. Let $c:=1/(4m) \geq \sqrt{\log(n)}/(4d\sqrt{n})$. Then with Lemma~\ref{lem:additive-drift-main} applied on $\lambda = 1/6$ we obtain that $X_t^{i,j} \in [5n/(48m),7n/(48m)]$ in expected $O(n)$ iterations (since $cn(1+\lambda)/2 = n(1+1/6)/(8m) = 7n/(48m)$ and $cn(1-\lambda)/2 = n(1-1/6)/(8m) = 5n/(48m)$). Suppose that $X_t^{i,j} \in [5n/(48m),7n/(48m)]$. If $\lambda=3$ we see that $cn(1+\lambda)/2 = n(1+1/3)/(8m) = n/(6m)$ and $cn(1-\lambda)/2 = n(1-1/3)/(8m) = n/(12m)$. Hence, with  Lemma~\ref{lem:negative-drift-main} applied on $\lambda = 1/3$ we see that, as long as the fitness is zero, the probability is $O(1/n^3)$ to obtain $X_t^{i,j} \notin [n/(12m),n/(6m)]$ within the next $n^3$ iterations. Now the expected time until $X_t^{i,j} \in [n/(12m),n/(6m)]$ for all $i \in [4]$ and $j \in [m/2]$ is $O(n^2)$ which can be seen in a similar way as in Phase 1 of the proof of Theorem~\ref{thm:Runtime-Analysis-GSEMO-mRRMO} by the same argument about marked blocks where $x_{i,\ell}^j$ is called \emph{marked} in an iteration $t$ if there is $t' < t$ such that $X_{t'+1}^{i,j} \in [n/(12m),n/(6m)]$. Hence, in expected $(1-e^{-\Omega(n/m)})+e^{-\Omega(n/m)} \cdot O(n^2) = 1+o(1)$ iterations, Phase~1 is finished.

\textbf{Phase 2:} Create $x \in M$.

As in the proof of the previous theorem, define $Z_t:=\max\{\sum_{j=1}^m g_j(x) \mid x \in P_t, x_\ell^j \in U \text{ for all } j \in [m/2]\}$. Then $Z_t$ cannot decrease since GSEMO preserves non-dominated solutions. The number of non-dominated solutions is at most $S_2:=(2n/m)^{m-1}$ in this phase. 
To increase $Z_t$ in one iteration it suffices to choose a parent $z \in P_t$ with $\sum_{k=1}^m g_k(z)=Z_t$ (prob. at least $1/S_2$), omit crossover and flip  a specific bit. Since $m/2 + n/2 \leq Z_t \leq n$ we need at most $n/2-m/2$ such steps and hence, the expected number of iterations to complete this phase is at most 
$$\left(\frac{n}{2}-\frac{m}{2}\right) \cdot \frac{en S_2}{(1-p_c)} = O\left(\frac{n^2S_2}{1-p_c}\right).$$

\textbf{Phase 3:} Create $x \in N$.

As in the previous proof, define $O_t:=\max\{\sum_{j=1}^{m/2} \zeros{x_\ell^j} \mid x \in P_t \cap M\} \in \{m/2, \ldots , n/2\}$. Then this phase is finished if $O_t=n/2$ and $O_t$ cannot decrease. The number of non-dominated solutions in this phase is also at most $S_2$. To increase $O_t$ one has to choose an individual $y$ with $\sum_{j=1}^{m/2} \zeros{y_\ell^j} = O_t$ as a parent, omit crossover and flip one of $n/2-O_t$ ones in one of the right blocks while not flipping another bit happening with probability $(1-p_c)(n/2-O_t)/(enS_2)$. Hence, the expected number of iterations to finish this phase is at most
$$\sum_{j=m/2}^{n/2-1} \frac{enS_2}{(1-p_c)(n/2-j)} = O\left(\frac{n \log(n) S_2}{1-p_c}\right).$$

For defining the next phases, let $\gamma_t:=\max\{\vert{U(x)}\vert \mid x \in N\}$. Suppose that there is an individual $z \in P_t \cap N$ with $\vert{U(z)}\vert = \gamma_t$ where $\gamma_t < m/2$, but no corresponding individual $w \in P_t \cap N$ with $\vert{U(w)}\vert > \gamma_t$. 

\textbf{Phase $\gamma_t$ + 4:} Create $x \in N$ with $|U(x)|>\gamma_t$.

Phase~4 starts when $\gamma_t=0$. If Phase~$m/2 -1 + 4 = m/2+3$ is finished, a Pareto-optimal search point has been created. Fix $y \in P_t$ with $|\mathcal{U}(y)| = \gamma_t$. Define 
\begin{align*}
\mathcal{K}_t:=\{\mathcal{U} \subset [m/2] \mid |\mathcal{U}| = \gamma_t \wedge \exists x \in P_t: (\mathcal{U}(x) = \mathcal{U} \wedge x_r^j \in C \wedge x_\ell^j = 0^{n/m} \text{ for all $j \in [m/2] \setminus \mathcal{U}(x)$})\}
\end{align*}
as in the proof of the previous theorem. Note that $\mathcal{K}_t$ cannot decrease since $P_t \subset N$ and two solutions $x,y \in P_t$ with $\mathcal{U}(x) \neq \mathcal{U}(y)$ are non-dominated. As before, denote the increase of $\mathcal{K}_t$ before increasing $\gamma_t$ as a \emph{failure}. Note that the failure probability is at most $5/6$ (by Claim~\ref{claim:failure-probability} in the previous proof). As in the previous proof, let $\mathcal{A}_t:=\{x \in N \mid \mathcal{U}(x) \in \mathcal{K}_t, x_\ell^j = 0 \text{ and } x_r^j \in C \text{ for all } j \in [m/2] \setminus \mathcal{U}(x)\}$.

\textbf{Subphase~A:} Find all individuals in $\mathcal{A}_t:=\mathcal{A}$ or a failure occurs. 

We may assume that no failure occurs, as the occurrence of a failure only accelerates the completion of the phase.
Fix $y \in \mathcal{A}$ which is not already found let $d_t:=\min\{H(x,y) \mid x \in P_t \cap \mathcal{A} \mid \mathcal{U}(x)=\mathcal{U}(y)\}$. Then $1 \leq d_t \leq n/m (m/2 - \gamma_t) \leq n/2$. The number of non-dominated solutions is at most $|\mathcal{K}_t| \cdot (2n/m)^{m/2}$. One can decrease $d_t$ by choosing an individual $w \in \mathcal{A}$ with $H(w,y) = d_t$ as a parent where $w_\ell^j \neq y_\ell^j$ for $j \in [m/2]$, omitting crossover and then either flipping a specific bit happening with probability at least $(1-p_c)/( e n |\mathcal{K}_t| (2n/m)^{m/2})$ in one iteration. Hence, we see that the time $T_y$ to find $y$ is stochastically dominated by the independent sum $Z:=\sum_{j=1}^{n/2} Z_j$ of geometrically distributed random variables with success probability $q:=(1-p_c)/(e n |\mathcal{K}_t|(2n/m)^{m/2})$. We have $\expect{Z}=e n^2 |\mathcal{K}_t|(2n/m)^{m/2}/(2(1-p_c))$. With Theorem~\ref{thm:Doerr-dominance}(1) we get for $\lambda:= 4en^2|\mathcal{K}_t|(2n/m)^{m/2}/(1-p_c)$, $s:= e^2n^3|\mathcal{K}_t|^2(2n/m)^m/(2(1-p_c)^2)$ and $p_{\min} = q$
$$\Pr\left(T_y \geq \frac{9en^2|\mathcal{K}_t|(2n/m)^{m/2}}{2(1-p_c)}\right) \leq \Pr\left(Z \geq \frac{9en^2|\mathcal{K}_t|(2n/m)^{m/2}}{2(1-p_c)}\right) \leq \exp\left(-\frac{1}{4} \min\left\{\frac{\lambda^2}{s}, \lambda p_{\min}\right\}\right) = e^{-n}.$$
As $|\mathcal{A}| \leq e^{n/2}$, we obtain by a union bound on all possible $y$ that with probability at least $1-e^{-n/2}$ all $y \in \mathcal{A}$ are found within $9en^2|\mathcal{K}_t|(2n/m)^{m/2}/(2(1-p_c))$ iterations. If this does not happen we repeat the above arguments where the expected number of periods is $1+o(1)$. Hence, $\expect{T} = O(n^2|\mathcal{K}_t|(2n/m)^{m/2}/(1-p_c))$.

\textbf{Subphase~B:} Find an individual $x \in N$ with $|\mathcal{U}(x)|=\gamma_t$, and $x_\ell^j \in P \setminus \{1^{n/m}\} \wedge x_r^j \in T$ for exactly one $j \in [m/2]$ or a failure occurs.

As in the previous subphase we can assume that no failure occurs. Note that $P_t = \mathcal{A}$ or we reached the goal of a subseqent phase. To finish this phase, it suffices to choose two $y,z \in P_t$ where for a block $j \in [m/2]$ we have $(z_r^j)_i = 1 - (y_r^j)_i$ for all $i \in [n/m]$ whereas $z_\ell^j = y_\ell^j = 0^{n/m}$, and both $y$ and $z$ coincide in all remaining blocks (prob. at least $1/|\mathcal{A}| = 1/(|\mathcal{K}_t| \cdot (2n/m)^{m/2})$), perform uniform crossover to create $w$ with $\ones{w_{i,\ell}^j} = \zeros{w_{i,\ell}^j}$ for all $i \in [4]$ (prob. $p^{\text{cross}} := p_c \cdot \left(\binom{n/(4m)}{n/(8m)} \cdot (1/2)^{n/(4m)}\right)^4 = \Omega(p_c m^2/n^2)$ by Stirling's formula) and omit mutation (prob. $(1-1/n)^n \geq 1/4$). This happens with probability at least $p^{\text{cross}}/(4|\mathcal{K}_t| (2n/m)^{m/2})= \Omega(p_cm^2/(n^2|\mathcal{K}_t| (2n/m)^{m/2}))$ in one iteration. Hence, this phase is finished in expected $4|\mathcal{K}_t|(2n/m)^{m/2}/p^{\text{cross}} = O(n^2|\mathcal{K}_t|(2n/m)^{m/2}/(p_c m^2))$ iterations.

Let $\mathcal{K}:=\mathcal{K}_t$. If a failure occurs either in Subphase~A or Subphase~B then $|\mathcal{K}|$ increases by one and we can pessimistically assume to start again at the beginning of Subphase~A. Denote by $T_{\mathcal{K}}$ the expected time until we passed through Subphase~A and~B without a failure. To ease notation we write $T_{|\mathcal{K}|}$ instead of $T_{\mathcal{K}}$. Then have to determine $T_1$ since initially $|\mathcal{K}| = 1$ and GSEMO produces one outcome per iteration. Denote by $L(|\mathcal{K}|)$ the determined upper bound for the time of passing through Subphase~A and~B one time. Then it is easy to see that $L(|\mathcal{K}|) \leq |\mathcal{K}|L(1)$.  
By Claim~\ref{claim:failure-probability} from the previous proof we see $T_{|\mathcal{K}|} \leq L(|\mathcal{K}|) + 5/6 \cdot T_{|\mathcal{K}|+1} = |\mathcal{K}|L(1) + 5/6 \cdot T_{|\mathcal{K}|+1}$ since $\mathcal{K}$ increases by one if a failure happens. The maximum possible cardinality of $\mathcal{K}$ is $\alpha:=\binom{m/2}{\gamma_t}$. 
Hence, by writing the above formula in closed form we obtain
$$T_1 \leq \sum_{i=0}^\alpha \frac{5^i(i+1)}{6^i}L(1) \leq L(1) \sum_{i=0}^\infty \frac{5^i(i+1)}{6^i} = O(L(1))$$
implying that the desired $x$ is created in expected $O(L(1))$ iterations. We see
$$O(L(1)) = O\left(\frac{n^2(2n/m)^{m/2}}{1-p_c} + \frac{n^2 (2n/m)^{m/2}}{p_c m^2}\right).$$

\textbf{Subphase~C:} Find an individual $x \in N$ with $|\mathcal{U}(x)|=\gamma_t$, and $x_\ell^j = 1^{n/m}$ for exactly one $j \in [m/2] \setminus \mathcal{U}(y)$.

Let $Z_t:=\min\{\zeros{x_\ell^j} \mid x \in N \text{ with } x_r^j \in T \text{ for a } j \in [m/2]\} \in \{1, \ldots , n/m\}$. This phase is finished if $Z_t=0$. Further, $Z_t$ cannot increase. To decrease $Z_t$ one may choose an individual $y$ with $\zeros{y_\ell^j} = Z_t$ where $y_r^j \in T$ as a parent, omit crossover and flip a specific zero while not flipping another bit. This happens with probability at least $(1-p_c)/(enS_2)$ in one iteration since the maximum number of mutually incomparable solutions is at most $S_2$. Hence, the expected number of iterations until $Z_t=0$ is at most $O(n^2S_2/(m(1-p_c)))$.

\textbf{Subphase D:} Find an individual $x \in N$ with $|\mathcal{U}(x)|=\gamma_t+1$.

Let $Z_t:=\max\{g_{2j-1}(x) + g_{2j}(x) \mid x \in N \text{ with } x_\ell^j = 1^{n/m} \text{ and } x_r^{j} \notin C \text{ for a } j \in [m/2]\}$. Then $n/m+1 \leq Z_t < 2n/m$ and this phase is finished if $Z_t$ reaches $2n/m$. Further, $Z_t$ cannot decrease. As in Phase~2 above, $Z_t$ can be increased with probability at least $(1-p_c)/(enS_2)$ in one iteration which leads to an expected number of $O(n^2S_2/(m(1-p_c)))$ iterations to finish this phase.

Hence, the expected number of fitness evaluations to pass Phase~$\gamma_t$+4 is at most 
$$O\left(\frac{n^2(2n/m)^{m/2}}{1-p_c} + \frac{n^2 (2n/m)^{m/2}}{p_c m^2} + \frac{n^2S_2}{m(1-p_c)}\right)$$

\textbf{Phase m/2+4:} Cover the whole Pareto front.

The treatment of this phase is similar to Subphase~A with the minor difference that we want to find all individuals in $\mathcal{P}_u:=\{y \mid \mathcal{U}(y) = [m/2]\} = \{1^{n/m}x_r^1 \ldots 1^{n/m}x_r^{m/2} \mid x_r^j \in C \text{ for all } j \in [m/2]\}$. We have $P_t \cap \mathcal{P}_u \neq \emptyset$ since we passed through Phase~m/2+3. Fix $y \in \mathcal{P}_u \setminus P_t$. Define $d_t = \min\{H(x,y) \mid x \in P_t \cap \mathcal{P}_u\}$. Then $1 \leq d_t \leq n/2$ and similar as in Subphase~A we see that we can decrease $d_t$ with probability at least $(1-p_c)/(en(2n/m)^{m/2})$ in one iteration and hence, we also obtain $P(T \geq 9en^2(2n/m)^{m/2}/(2(1-p_c))) = e^{-n}$ for the time $T$ to find $y$. By a union bound on all possible $y$ (since $|\mathcal{P}_u| \leq (2n/m)^{m/2}$) we see that with probability at most $e^{-n/2}$ that all $y \in \mathcal{P}_u$ are found within $9en^2(2n/m)^{m/2}/(2(1-p_c))$ iterations. If this does not happen we repeat the above arguments where the expected number of periods is $1+o(1)$. Hence, $\expect{T}=O(n^2(2n/m)^{m/2}/(1-p_c))$.

Now we upper bound the total runtime. Since Subphases~A,B and~C are passed at most $m/2$ times, the complete Pareto front is covered in expected
\begin{align*}
O&\left(\frac{n^2S_2}{1-p_c}+\frac{n\log(n)S_2}{1-p_c}+\frac{n^2m(2n/m)^{m/2}}{1-p_c} + \frac{n^2 (2n/m)^{m/2}}{p_c m}+\frac{n^2S_2}{1-p_c} + \frac{n^2 (2n/m)^{m/2}}{1-p_c}\right) \\
&= O\left(\frac{n^2S_2}{1-p_c}+ \frac{n^2 (2n/m)^{m/2}}{p_c m}\right) = O\left(\frac{n^2(2n/m)^{m-1}}{1-p_c}+ \frac{n^2 (2n/m)^{m/2}}{p_c m}\right)
\end{align*}
iterations since $|S_2| = (2n/m)^{m-1} \geq m (2n/m)^{m/2}$ for $m \geq 4$ (due to $m = O(\sqrt{n}/\log(n))$). This finishes the proof of the whole theorem.
\end{proof}

We note that for a constant number of objectives $m$ the upper bound for the runtime of GSEMO on $m$-\uRRRMO is $O(n^{m+1}/(1-p_c) + n^{m/2+2}/p_c)$ and for \nsgaIII it is $O(n^{m+1}/(1-p_c) + n^{m+1}/p_c)$ for $\mu=(2n/m)^{m-1}$ in terms of fitness evaluations. If $p_c = O(1/n^{m/2-1})$ we see that the former becomes $O(n^{m/2+2}/p_c)$ and the latter $O(n^{m+1}/p_c)$ implying that the former is by a factor of $n^{m/2-1}$ smaller. As for $m$-\RRRMO, the reason for this difference is again the dynamic population size of GSEMO.

\section{Difficulty of GSEMO and \nsgaIII on \RRRMOuni Without Crossover}

Finally, we observe that both GSEMO and \nsgaIII (without crossover) become extremely slow on the $m$-\RRRMOuni benchmark. This slowdown occurs even when attempting to find the first Pareto-optimal solution. We begin by citing a result from~\citep{Dang2024} which states that the sets $C$ and $T$ are separated by a large Hamming distance, and the same holds for the sets $U$ and $P$.

\begin{lemma}[Lemma~16 in~\citep{Dang2024}]
\label{lem:large-Hamming}
The following properties hold for the subsets of Definition~\ref{def:uRRMO-sets-and-function}.
\begin{itemize}
    \item[(i)] $\forall x_\ell \in U, \forall y_\ell \in P: H(x_\ell , y_\ell) \geq n/(4m)$,
    \item[(ii)] $\forall x_r \in C, \forall y_r \in T: H(x_\ell , y_\ell) \geq 3n/(8m)$.
\end{itemize}
\end{lemma}

\begin{theorem}\label{thm:NSGA-III-pc-zero-uniform}
Suppose that $m = o(n/\log(n))$, $m$ is divisible by $2$ and $n$ is divisible by $8m$. \nsgaIII (Algorithm~\ref{alg:nsga-iii}) on $m$-\RRRMO with $p_c=0$, any choice of $\refer$, and $\mu = 2^{o(n/m)}$ needs at least $n^{\Omega(n/m)}$ generations in expectation to create any Pareto-optimal search point of $m$-\RRRMO.
\end{theorem}

\begin{proof}
An individual $x$ with $0<\ones{x_{i,\ell}^j}\leq 6n/(5m)$ for all $i \in [4]$ and $j \in [m/2]$ initializes with probability $1-m/2 \cdot 2^{-\Omega(n/m)} = 2^{-\Omega(n/m)}$ since $m=o(n/\log(n))$. Hence, by a union bound, with probability $1-\mu 2^{-\Omega(n/m)} = 1-o(1)$ every individual $x$ initializes with $n/(12m) \leq \ones{x_{i,\ell}^j} \leq n/(6m)$ for every $i \in [4]$ and $j \in [m/2]$ since $\mu \in 2^{o(n/m)}$. Suppose that this happens. Then all $x \in P_t$ satisfy $x \in L$ and have nonzero fitness; hence, the algorithm will always reject search points with zero fitness. Let $x \in L \cup M \cup N$ where $\mathcal{U}(x) = \emptyset$ if $x \in N$. We now show that any $y \in N$ with $y_\ell^j = 1^{n/m}$ or $y_r^j \in T$ and $y_\ell^j \in P$ for some $j \in [m/2]$ can be created from this $x$ via mutation with probability $n^{-\Omega(n/m)}$. Note that this includes also these $y \in N$ with $|\mathcal{U}(y)|>0$.

If $x \in L$ then $x_\ell^j \in U$ for all $j \in [m/2]$ and therefore, by Lemma~\ref{lem:large-Hamming}(i), one has to flip at least $n/(4m) = \Omega(n/m)$ many bits in a left block to create $x_\ell^j \in P$. By a union bound on all blocks and all possible strings from $P$, creating $x_\ell^j \in P$ for one $j \in [m/2]$ happens with probability at most $m/2 \cdot (n/m+1) \cdot n^{-\Omega(n/m)} = n^{-\Omega(n/m)}$ since $m \leq n/8$.

If $x \in N \cup M$ then $x_\ell^j \in C$ for all $j \in [m/2]$. To create $y$ with $y_r^j \in T$ and $y_\ell^j \in P$ is is required to flip at least $3n/(8m)$ bits in a left block by Lemma~\ref{lem:large-Hamming}(ii). Hence, by a union bound on $T$ and on all blocks we see that with probability at most
$$m|T|n^{-3n/(8m)}/2 = m \binom{n/(4m)}{n/(8m)}^4 n^{-3n/(8m)}/2 \leq m 2^{n/m} 2^{-\log(n) \cdot 3n/(8m)} = m 2^{-\Omega(\log(n) 3n/(8m))} = n^{-\Omega(n/m)}.$$
To create $y$ with $y_\ell^j = 1^{n/m}$ it is required to flip at least $4n/(6m)$ specific one bits any left block happening with probability at most $n^{-\Omega(n/m)}$. By a union bound over all blocks, this happens with probability at most $n^{-\Omega(n/m)}$. Hence, the expected number of needed generations in total is at least $(1 - o(1))(n^{\Omega(n/m)}/\mu)=n^{\Omega(n/m)}$.
\end{proof}

We are also able to obtain a similar result for GSEMO.

\begin{theorem}
	Suppose that $m = o(n/\log(n))$, $m$ is divisible by $2$ and $n$ is divisible by $8m$. GSEMO (Algorithm~\ref{alg:nsga-iii}) on $m$-\RRRMO with $p_c=0$ needs at least $n^{\Omega(n/m)}$ fitness evaluations in expectation to create any Pareto-optimal search point of $m$-\RRRMO.
	\label{thm:GSEMO-pc-zero-uniform}
\end{theorem}

\begin{proof}
As in the proof of the previous theorem, we see that an individual $x$ with $x \in L$ initializes with probability $1-o(1)$. Suppose that this happens. Since the algorithm always rejects search points with fitness zero, we see as in the previous proof that an $x \in N$ with $|\mathcal{U}(x)|>0$ is created with probability at most $n^{-\Omega(n/m)}$. So the expected number of required iterations in total is also at least $(1 - o(1))n^{\Omega(n/m)}$. 
\end{proof}

Hence, in the case of both algorithms above, if $p_c$ is constant, $m \leq d \sqrt{n}$ for a sufficiently small constant $d$ and $(2n/m)^{m-1} \leq \mu \leq 2^{o(n/m)}$ in case of \nsgaIII, uniform crossover guarantees also a speedup of order $n^{\Omega(\sqrt{n})}$. This speedup is even of order $n^{\Omega(n)}$ if $m$ is constant.

\section{Conclusions}
We defined $m$-\RRRMO and $m$-\RRRMOuni, variants of the bi-objective \RRRMO-functions proposed by~\citep{Dang2024}, for the many objective setting on which the EMO algorithm \nsgaIII using crossover for a constant $0<p_c<1$ and a constant number of objectives $m$ can find the whole Pareto set in epected polynomial time. As for other many-objective function classes like $\LOTZ$, $\OMM$ and $\COCZ$, the upper bound on the expected number of generations behaves asymptotically independently of $\mu$ and $m$. On the other hand, if crossover is disabled, in both cases \nsgaIII requires exponential time to even find a single Pareto-optimal point. Similar results also hold for the GSEMO algorithm. To the best of our knowledge, these are the first proofs for an exponential performance disparity for the use of
crossover in the many-objective setting, particularly for \nsgaIII, where the number of objectives is not necessarily constant. However, we are confident that for the $m$-\textsc{OneJumpZeroJump}$_k$ benchmark proposed by~\citet{Zheng_Doerr_2024}, the many-objective version of the bi-objective \textsc{OneJumpZeroJump}~\citep{DoerrQu22}, crossover provably guarantees a subexponential speedup of order $O(n)$. However, the population dynamics of \nsgaIII is different to that of \nsga since solutions become spread quite evenly on the whole Pareto front (see \citet{opris2025multimodal}) and hence, different arguments than in~\citep{DoerrQ23b} are needed to prove such results. We hope that our work may serve as a stepping stone towards a better understanding of the advantages of crossover on more complex problem classes, as it has been done in single-objective optimization.\\ 

\textbf{Declaration of generative AI and AI-assisted technologies in the writing process}

During the preparation of this work the author used Chatgpt 4.0 in order to improve language and readability. After using this tool, the author reviewed and edited the content as needed and takes his full responsibility for the content of the publication.

\bibliographystyle{abbrvnat}
\bibliography{Journal}
\end{document}